\documentclass[10pt]{article}

\usepackage[
  a4paper,
  top=20mm,
  bottom=22mm,
  left=22mm,
  right=22mm,
  headsep=6mm,
  footskip=10mm
]{geometry}

\usepackage[T1]{fontenc}
\usepackage[utf8]{inputenc}
\usepackage{newtxtext}
\usepackage{microtype}

\usepackage{amsmath}
\usepackage{amssymb}
\usepackage{amsthm}
\usepackage{mathtools}
\usepackage{bm}

\usepackage{newtxmath}

\usepackage{apptools}
\usepackage{thmtools}
\usepackage{thm-restate}

\newtheorem{theorem}{Theorem}
\newtheorem{lemma}[theorem]{Lemma}

\newtheorem{definition}[theorem]{Definition}

\AtAppendix{\counterwithin{theorem}{section}}
\AtAppendix{\counterwithin{proposition}{section}}

\newtheorem{repeatedtheorem}{Theorem}

\theoremstyle{definition}
\newtheorem{rmk}{Remark}
\newtheorem{example}{Example} 
\newtheorem{repeatedexample}{Example}

\usepackage{graphicx}
\usepackage{booktabs}
\usepackage{array}
\usepackage{tabularx}
\usepackage{longtable}
\usepackage{multirow}
\usepackage{float}

\usepackage[
  font=small,
  labelfont=bf,
  labelsep=period,
  skip=5pt
]{caption}

\usepackage{subcaption}

\usepackage{enumitem}
\setlist{
  topsep=3pt,
  itemsep=1pt,
  parsep=0pt,
  partopsep=0pt
}

\usepackage{titlesec}

\titleformat{\section}
  {\normalfont\large\bfseries}
  {\thesection}{0.6em}{}

\titleformat{\subsection}
  {\normalfont\normalsize\bfseries}
  {\thesubsection}{0.6em}{}

\titleformat{\subsubsection}
  {\normalfont\normalsize\itshape}
  {\thesubsubsection}{0.6em}{}

\titlespacing*{\section}
  {0pt}{11pt plus 2pt minus 2pt}{5pt}

\titlespacing*{\subsection}
  {0pt}{8pt plus 2pt minus 2pt}{3pt}

\titlespacing*{\subsubsection}
  {0pt}{6pt plus 1pt minus 1pt}{2pt}

\usepackage[authoryear,longnamesfirst, round]{natbib}
\usepackage{xcolor}
\usepackage[
  colorlinks=true,
  linkcolor=blue!45!black,
  citecolor=blue!45!black,
  urlcolor=blue!45!black,
  pdfborder={0 0 0}
]{hyperref}

\usepackage[nameinlink,noabbrev]{cleveref}

\usepackage{xspace}
\usepackage{siunitx}
\usepackage{algorithm}
\usepackage{algpseudocode}

\usepackage{authblk}

\title{\vspace{-1.5em}%
  Training Deep Morphological Neural Networks as Universal Approximators
  \vspace{-0.3em}
}

\author[1,4,*]{Konstantinos Fotopoulos}
\author[1,2,3]{Petros Maragos}

\affil[1]{Institute of Robotics, Athena Research Center, Maroussi, Greece}
\affil[2]{School of ECE, National Technical University of Athens, Athens, Greece}
\affil[3]{HERON - Hellenic Robotics Center of Excellence, Athens, Greece}
\affil[4]{Department of CS, ETH Z\"urich, Z\"urich, Switzerland}

\affil[ ]{\texttt{kostfoto2001@gmail.com}, \quad \texttt{petros.maragos@athenarc.gr}}

\date{}

\begin{document}

\maketitle

{
\renewcommand{\thefootnote}{\fnsymbol{footnote}}
\footnotetext[1]{Corresponding author.}
}

\begin{abstract}
We investigate deep morphological neural networks (DMNNs), studying how changes in algebraic structure affect the expressivity and trainability of deep architectures. We show that despite the inherent non-linearity of morphological operations, existing deep morphological architectures fail to be universal approximators and exhibit optimization limitations related to sparse and uninformative gradients. To address these issues, we introduce architectures incorporating constrained "linear" activations between morphological layers and averaging max-plus and min-plus neurons. Only \(O(N)\) parameters (or learnable parameters) per layer of size \(N\) belong to the activations, with the remaining parameters constrained to morphological operations. We prove universal approximation results for the proposed architectures without requiring substantially larger parameter counts than comparable linear networks. Residual connections and weight dropout further improve generalization. Our experiments show that our networks are trainable and compact, despite the imposed architectural restrictions. \end{abstract}


% Main text
\section{Introduction}

Modern deep learning architectures are predominantly built on linear operations, such as matrix multiplications and convolutions, combined with nonlinear activations \citep{goodfellow2016deep}. This raises the following broader question: how do the expressiveness, optimization, and representation properties of deep networks change when the underlying algebraic structure is modified? Recent interest in binary \citep{yuan2023comprehensive} and other non-standard neural architectures reflects growing interest in understanding neural computation beyond the standard linear setting. 

Before deep learning, mathematical morphology \citep{haralick1987image, heihmans1994morphological, maragos1987morphological, maragos1999morphological, maragos2005morphological, najman2010mathematical, serra1982image, serra1988image, soille2004morphological}
% \citep{HSZ87,Heij94,Mara99,Mara05b,NaTa10,Serr82,Serr88,Soil04} 
played a central role in image and signal processing. Its operations -- dilations and erosions -- enabled effective, task-specific feature extraction using max-plus and min-plus algebra. Its historical success and theoretical foundations make mathematical morphology a natural framework for studying these questions. Motivated by this, researchers developed models based on morphological operations. One of the earliest was the \textit{Morphological Perceptron (MP)}, which replaces addition and multiplication with max and plus, enabling nonlinear decision boundaries and forming the basis of morphological neural networks 
% \citep{davidson1993morphology,PeMa00,RitSus96,RiUr03,SuEs11,YaMa95}
\citep{davidson1993morphology, pessoa2000neural, ritter1996introduction, ritter2003lattice, sussner2011morphological, yang1995minmax}. 
These models offer appealing properties, including compressibility 
% \citep{DiMa21, GrDoGe2023, Zha+19} 
\citep{dimitriadis2021advances, groenendijk2023geometric, zhang2019maxplus}
and fast training. The Dilation-Erosion Perceptron (DEP) 
% \citep{DEAARAUJO2011513}
\citep{araújo2011class}, trained using CCP 
% \citep{ChMa17}
\citep{charisopoulos2017morphological}, blends dilation and erosion via a convex combination.

Despite their advantages, both the MP and DEP have a limitation, which restricts their applicability: Their decision boundaries are axis-aligned 
% \citep{RitSus96,YaMa95}
\citep{ritter1996introduction, yang1995minmax}. To address this, researchers have proposed approaches to transform the input space into one that is MP-separable or DEP-separable using kernel transformations. For example, \citet{valle2020reduced} proposed a simple linear transformation as a kernel to map features into a more favorable space.

This idea of mapping inputs to a separable space is not new -- it is the foundation of deep learning, motivating the study of whether deep morphological networks (DMNNs), built on the MP and its variants, can similarly learn effective representations using morphological operations. More motivating factors include the fact that morphological networks are i) multiplication-free, which is the most expensive floating point operations in typical linear networks, ii) a natural continuous extension of boolean networks, and iii) a constrained category of maxout networks \citep{goodfellow2013maxout}. 

Recent works have explored this question with varying levels of success 
% \citep{HU2022108893, GrDoGe2023, ShShZhCh2022}
\citep{hu2022learning, groenendijk2023geometric, shen2022deep}. Notable recent efforts, related to this paper, include the works of 
% \citet{FFY20}
\citet{franchi2020deep}, who demonstrated the potential of deep hybrid morphological-linear networks, 
% \citet{DiMa21}
\citet{dimitriadis2021advances}, who demonstrated the amenability to pruning of morphological networks, and 
% \citet{Angu24,VeAn22} 
\citet{angulo2024nonlinear, velasco-forero2022discrete}, who proposed representing the combined convolutions, nonlinear activation and max-pooling in ConvNets  as a max of erosions or min of dilations by leveraging the morphological representation theory 
% \citep{Mara89a}. 
\citep{maragos1989representation}. 

However, despite these advancements, most successful approaches rely heavily on standard linear components, making it difficult to isolate the role of the underlying algebraic structure. In addition, as noticed by \citet{dimitriadis2021advances} and \citet{dimitrova2026learning}, the sparsity of the gradients of the network make training difficult. In this paper, we study, both theoretically and empirically, what minimal algebraic "augmentations" are required for deep morphological networks to achieve powerful expressivity and trainable representations. Our contributions can be summarized as follows:
\paragraph{Contributions.} In this paper, we address the challenges of training deep morphological networks.

\begin{itemize}[left=0pt, nosep]
\item We show that several existing deep morphological architectures cannot be universal approximators (i.e., cannot approximate continuous functions arbitrarily well), despite their inherent non-linearity, and identify optimization limitations arising from their operations. Notably, the difficulty in training these networks is not isolated just to the sparsity of their gradients, but also to the fact that they are not 1\textsuperscript{st}-order universal approximators.
\item Motivated by these obstructions, we introduce deep morphological architectures incorporating constrained linear activations between morphological layers and averaging max-plus and min-plus neurons. Only \(O(N)\) parameters (or learnable parameters) per layer of size \(N\) belong to the activations, with the remaining parameters constrained to morphological operations.
\item We prove universal approximation/representation results for the deterministic architectures and show empirically that they can be successfully trained on standard image classification tasks despite the imposed architectural constraints. To the best of our knowledge, we are the first to propose morphological networks that are \(1\)-st order universal approximators. 
\item We improve the generalization ability of the proposed networks using residual connections and weight dropout.
\item We empirically show that the proposed architectures do not require substantially larger parameter counts than comparable linear networks to achieve expressivity.
\item Finally, we study a hybrid linear/morphological architecture and observe that it benefits from large batches with respect to its training convergence under gradient descent. 
\end{itemize}

\section{Prerequisites}
\label{section:2}

\paragraph{Tropical Algebra.}  
Tropical algebra is a branch of mathematics focused on the study of the tropical semiring, which encompasses both the \textit{min-plus semiring} and the \textit{max-plus semiring} 
% \citep{butkovivc2010max, Cuni79,maclagan2021introduction,MCT21}
\citep{butkovivc2010max, cuninghame1979minimax, maclagan2021introduction, maragos2021tropical}. The \textit{max-plus semiring}, denoted as \((\mathbb{R}_{\mathrm{max}}, \vee, +)\), consists of the set \(\mathbb{R}_{\mathrm{max}} = \mathbb{R} \cup \{-\infty\}\), equipped with two binary operations: \(\vee\) (the maximum operator) and \(+\) (ordinary addition). Similarly, the \textit{min-plus semiring}, denoted as \((\mathbb{R}_{\mathrm{min}}, \wedge, +)\), is defined over the set \(\mathbb{R}_{\mathrm{min}} = \mathbb{R} \cup \{+\infty\}\), with \(\wedge\) (the minimum operator) and \(+\) as its binary operations. 
These semirings naturally extend to operations on vectors and matrices. For instance, the \textit{max-plus matrix multiplication} \(\boxplus\) and the \textit{min-plus matrix multiplication} \(\boxplus'\) for matrices \(\mathbf{A} = [a_{ij}] \) and \(\mathbf{B} = [b_{ij}]\) is defined as:  
\[
(\mathbf{A} \boxplus \mathbf{B})_{ij} = \bigvee_k (a_{ik} + b_{kj}), \quad 
(\mathbf{A} \boxplus' \mathbf{B})_{ij} = \bigwedge_k (a_{ik} + b_{kj}).
\]    

\paragraph{Mathematical morphology.} 
Modern mathematical morphology 
% \citep{Heij94,Serr88} 
\citep{heihmans1994morphological, serra1988image} is  defined on complete lattices. A partially ordered set \((\mathbb{L}, \preceq)\) is a complete lattice  \citep{birkhoff1967lattice} if and only if (iff) every subset  \(X\subseteq \mathbb{L}\) has a supremum and an infimum, denoted by \(\bigvee X\) and \(\bigwedge X\) respectively. Consider two complete lattices \(\mathbb{L}\) and \(\mathbb{M}\). A lattice operator \( \delta:\mathbb{L}\rightarrow \mathbb{M} \) is called a \textit{dilation} iff it distributes over the supremum of any input collection (possibly infinite). Dually, an \textit{erosion} is defined as any lattice operator \(\varepsilon:\mathbb{L}\rightarrow \mathbb{M}\) that distributes over the infimum. Namely, the following properties hold:
\[
\delta \left(\bigvee X\right) = \bigvee \delta (X), 
\quad 
\varepsilon \left(\bigwedge X\right) = \bigwedge \varepsilon (X). 
\; \; \forall X\subseteq \mathbb{L}
\]
Let \(\overline{\mathbb{R}}=\mathbb{R}\cup \{+\infty, -\infty\}\) denote the extended set of real numbers, which is a complete lattice when equipped with its ordinary order. The set \(\overline{\mathbb{R}}^n\) of vectors equipped with the partial order \(\mathbf{x}\preceq \mathbf{y} \Leftrightarrow x_i \leq y_i, \forall i\in [n]\) forms a complete lattice. There exists a natural way to define a dilation and an erosion from \(\overline{\mathbb{R}}^n\) to \(\overline{\mathbb{R}}\): Given an input vector \(\mathbf{x}\in \overline{\mathbb{R}}^n\) and weights \(\mathbf{w}, \mathbf{m} \in \mathbb{R}^n\), a vector dilation \(\delta_{\mathbf{w}}\) and erosion \(\varepsilon_{\mathbf{m}}\) are defined as follows: 
\begin{equation*}
\delta_{\mathbf{w}}(\mathbf{x})=\bigvee_{i\in [n]}(x_i + w_i) = \mathbf{w}^\top \boxplus \mathbf{x}, \quad 
\varepsilon_{\mathbf{m}}(\mathbf{x})=\bigwedge_{i\in [n]}(x_i + m_i) = \mathbf{m}^\top \boxplus' \mathbf{x}.
\end{equation*}
The Morphological Perceptron (MP) is simply a dilation (or erosion) as defined above, optionally biased with an additional bias term (i.e. \(\mathrm{MP}(\mathbf{x})=w_0 \vee \delta_{\mathbf{w}}(\mathbf{x})\) for a dilation-based MP). Note that if \(\mathbf{x} < +\infty\), then \(\mathbf{w}\) is allowed to take the value \(-\infty\), and dually if \(\mathbf{x} > -\infty\), then \(\mathbf{m}\) can take the value \(+\infty\). The DEP is a convex combination of a dilation and an erosion (i.e. \(\mathrm{DEP}_{\mathbf{w}, \mathbf{m}}(\mathbf{x}) = \lambda \delta_{\mathbf{w}}(\mathbf{x})+(1-\lambda) \varepsilon_{\mathbf{m}}(\mathbf{x})\), for \(\lambda\in [0,1]\)).
% , and it is trained using the Convex-Concave Procedure \citep{ChMa17}. 

Mathematical morphology also employs dilations and erosions defined on sets of functions, which naturally model images. The set \(\mathbb{F}=\{f:\mathbb{Z}^n \to \overline{\mathbb{R}}\}\) of extended real-valued functions over the \(n\)-dimensional grid of integers, equipped with the partial ordering \(f\preceq g \Leftrightarrow f(\mathbf{x}) \leq g(\mathbf{x}) \;  \forall \mathbf{x}\),  becomes a complete lattice. On this set, we can define the \textit{max-plus convolution} \(\oplus\) and \textit{min-plus convolution} \(\oplus'\) of a function $f\in \mathbb{F}$ with a structuring element (function)  \( g: \mathbb{Z}^n \to \mathbb{R} \cup \{-\infty\} \), whose domain is \( \mathrm{\mathbf{dom}}(g)=\{\mathbf{x} : g(\mathbf{x}) > -\infty\} \), as follows:  
\begin{equation*}
(f \oplus g)(\mathbf{x}) = \bigvee_{\mathbf{y} \in \mathrm{\mathbf{dom}}(g)} f(\mathbf{x} - \mathbf{y}) + g(\mathbf{y}), \quad
(f \oplus' g)(\mathbf{x}) = \bigwedge_{\mathbf{y} \in \mathrm{\mathbf{dom}}(g)} f(\mathbf{x} - \mathbf{y}) - g(\mathbf{y}).
\end{equation*}
% The max-plus (min-plus) convolution is a dilation \(\delta_{g}(f)\) (resp. erosion \( \varepsilon_{g}(f)\)) on \(\mathbb{F}\). 
The max-plus convolution and \(\delta_{g}(f)=f\oplus g\) and min-plus correlation \( \varepsilon_{g}(f)(\mathbf{x})=f(\mathbf{x})\oplus' g(-\mathbf{x}) \) are a dilation and erosion on \(\mathbb{F}\). 
By combining these dual morphological operators, we can construct more complex operators such as \textit{openings} $\alpha_g(f) = \delta_{g}(\varepsilon_{g}(f))$ and \textit{closings} $\beta_g(f) = \varepsilon_{g}(\delta_{g}(f))$. The opening can smooth the input by removing small bright structures, while the closing can fill gaps and remove small dark structures 
% \citep{Serr82,Mara05b}
\citep{serra1982image, maragos2005morphological}. 

The \textit{Representation Theorem} of 
% \citet{MaSc87a}
\citet{maragos1987morphological} takes this one step further: it proves that we can represent any increasing, translation-invariant linear shift-invariant filter as a supremum of weighted erosions. A byproduct of this is that we can write any linear perceptron with positive weights that sum up to \(1\) as a supremum of weighted erosions: If \( \alpha_i > 0 \) for all \( i \) and \( \sum_i \alpha_i = 1 \), then the following identity holds:
% \[
%     \sum_{i = 0}^n \alpha_i x_i = \sup_{r_0, \ldots, r_{n-1} \in \mathbb{R}} \Bigg[ \min \Bigg\{ x_0 - r_0, \ldots, x_{n-1} - r_{n-1},\ x_n + \frac{\sum_{i = 0}^{n-1} \alpha_i r_i}{\alpha_n} \Bigg\} \Bigg].
% \]
\begin{equation*}
    \sum_{i = 0}^n \alpha_i x_i = \sup_{r_0, \ldots, r_{n-1} \in \mathbb{R}} \Bigg[ \min \Bigg\{ x_0 - r_0, \ldots,
    x_{n-1} - r_{n-1},\ x_n + \frac{\sum_{i = 0}^{n-1} \alpha_i r_i}{\alpha_n} \Bigg\} \Bigg].
\end{equation*}
If we allow biased erosions we can relax the condition that the weights must sum up to 1, as long as they sum up to less than 1. For example, we have \(\frac{x}{2}=\sup_{r\in \mathbb{R}}\min(x-r,r)\). Finally, if one additionally allows \(\min\)-terms of the form \(-x_i+r_{n+i}\) (that is, taking an erosion and an anti-dilation), then every \(1\)-Lipschitz function with respect to the \(\|\cdot\|_{\infty}\)-norm can be represented \citep{banon1993decomposition, luo2021min}. For details regarding the Representation Theorem, see Appendix \ref{appendix:f}. 

A central theme of this work will be universal approximation; the property of a function class to be able to approximate any continuous (or some other type of) function. Formally, we have the following definition:
\begin{definition}[Universal approximator]
\label{definition:b0}
We say that a class of functions \(\mathcal{F}\) is a universal approximator on a set \(\mathcal{D} \subseteq \mathbb{R}^n\) if and only if it is a dense subclass of the class of continuous functions on \(\mathcal{D}\) under the \(\|\cdot\|_{\infty}\) function norm. Equivalently, for any continuous \(g\) defined on \(\mathcal{D}\), there exists a sequence \(\{f_i\}_{i\in \mathbb{N}} \subseteq \mathcal{F}\) such that \(f_i \to g\) uniformly. 
\end{definition}

We say that a network architecture is a universal approximator on \(\mathcal{D}\) if the class of functions it defines is a universal approximator on \(\mathcal{D}\). Most of our claims hold for compact  \(\mathcal{D}\), which for \(\mathbb{R}^n\) correspond to bounded and closed domains by the Heine–Borel theorem. With a slight abuse of terminology, whenever we say that a network is (or is not) a universal approximator, we mean that the defined architecture is (or is not) a universal approximator. 

We will also concern ourselves with higher order approximability, captured by the following definition:
\begin{definition}[\(k\)-th order Universal approximator]
\label{definition:higher}
We say that a class of functions \(\mathcal{F}\) is a \(k\)-th order universal approximator on a set \(\mathcal{D} \subseteq \mathbb{R}^n\) if and only if for any \(k\times\) continuously differentiable \(g\in C^k(\mathcal{D})\) defined on \(\mathcal{D}\), there exists a sequence \(\{f_i\}_{i\in \mathbb{N}} \subseteq \mathcal{F}\) such that \(D^jf_i \to D^jg\) uniformly for every \(j \in \{0,\ldots, k\}\). 
\end{definition}
Here, \(D^j\) is the derivative operator. For example, \(D^0=\operatorname{id}, D^1=(\nabla \cdot)^\top\) and \(D^2=\nabla^2\) the Hessian. Note that a stronger definition may also ask that \emph{every} uniformly convergent sequence also converges in derivatives up to \(k\)-th order.

% Finally, the definition of Lipschitz continuity will be useful:
% \begin{definition}[Lipschitz]\label{def:lipschitz}
%     We say a function \(f:(X, \|\cdot\|_{X}) \to (Y, \|\cdot\|_Y)\) from and onto normed spaces is called \(L\)-Lipschitz if 
%     \[
%     \|f(y)-f(x)\|_Y \leq L \|y - x\|_X, \quad \forall y, x \in X. 
%     \]
%     We will care about i) uni-variate scalar functions, where both norms are the absolute value \(|\cdot|\), in which case we simply refer to the function as being Lipschitz, and ii) scalar functions over subsets of \(\mathbb{R}^d\) with \(\|\cdot\|_X = \|\cdot\|_{\infty}\), referred to as \(L_\infty\)-Lipschitz. 
% \end{definition}

\section{Deep morphological neural networks}

Defining deep morphological neural networks (DMNNs) involves combining the fundamental operations of mathematical morphology. In this section, we present the most common DMNNs, and showcase their fundamental problems. Then, we present our proposed architectures under differing set of constraints. Throughout, a superscript \(({\cdot})^{(n)}\) is used to indicate when something corresponds to the \(n\)-th layer, and \(N^{(n)}\) refers to the number of units in the \(n\)-th layer. We use \([k]:=\{1,\ldots,k\}\). Proofs of our theorems can be found in Appendices \ref{appendix:b} and \ref{appendix:c}.

Existing morphological architectures often compose max-plus or related morphological operators either directly or with standard nonlinear activations 
\citep{dimitriadis2021advances, araujo2017morphological, franchi2020deep}.
% , groenendijk2023geometric}. 
However, as we show below, such constructions suffer from fundamental expressivity limitations that can already be understood from basic principles of tropical algebra.

\paragraph{Max-plus MP-based DMNNs.}
The most obvious way to building a DMNN is by appending morphological layers consisting of several MPs. Suppose we have \(L\) morphological layers, with the \(n\)-th layer having \(N^{(n)}\) biased MPs. We can then write the network in recursive form:
\[
x^{(n)}_i=w^{(n)}_{i0} \vee \max_{j \in [N^{(n-1)}]}(x^{(n-1)}_j + w^{(n)}_{ij}),
\]
with \(\mathbf{x}^{(0)}\) being the input and \(\mathbf{x}^{(L)}\) the output. 

As mentioned, \( (\mathbb{R}_{\mathrm{max}}, \vee, +) \) is a semiring. This means that we have distributivity of \(+\) over \(\vee\), and associativity of \(+\). These properties are enough to prove the associativity of matrix multiplication, i.e. for every matrices \(\mathbf{A}, \mathbf{B}, \mathbf{C} \) over \(\mathbb{R}_{\mathrm{max}}\), we have \((\mathbf{A}\boxplus \mathbf{B})\boxplus \mathbf{C} = \mathbf{A} \boxplus (\mathbf{B} \boxplus \mathbf{C}) = \mathbf{A}\boxplus \mathbf{B} \boxplus \mathbf{C}\). 

Let \( L \geq 2 \). The recursive form of the network can be expressed as:  
\[
\mathbf{x}^{(n)} = \mathbf{w}^{(n)}_0 \vee \mathbf{W}^{(n)} \boxplus \mathbf{x}^{(n-1)}, \quad n \in [L],
\]
where \( \mathbf{w}^{(n)}_0 = [w^{(n)}_{10}, \ldots, w^{(n)}_{N^{(n)}0}]^\top \in \mathbb{R}^{N^{(n)}}\) is the bias vector, and \( \mathbf{W}^{(n)} = (w^{(n)}_{ij})_{ij} \in \mathbb{R}^{N^{(n)} \times N^{(n-1)}}\) is the weight matrix.  

We can solve the recursive form for \( \mathbf{x}^{(L)} \) by unfolding the recursion, obtaining:
% \[
% \mathbf{x}^{(L)} = \mathbf{w}^{(L)}_0 \vee \mathbf{W}^{(L)} \boxplus \mathbf{x}^{(L-1)}
% \]
% \[
% = \mathbf{w}^{(L)}_0 \vee \mathbf{W}^{(L)} \boxplus \mathbf{w}^{(L-1)}_0 \vee \mathbf{W}^{(L)} \boxplus \mathbf{W}^{(L-1)} \boxplus \mathbf{x}^{(L-2)}.
% \]
% By continuing this process, we obtain:
\[
\mathbf{x}^{(L)} = \mathbf{w}_{eq, 0} \vee \mathbf{W}_{eq} \boxplus \mathbf{x}^{(0)},
\]
where 
\[
\mathbf{w}_{eq, 0} = \mathbf{w}^{(L)}_0 \vee \bigvee_{k=1}^{L-1} \Big( \mathbf{W}^{(L)} \boxplus \mathbf{W}^{(L-1)} \boxplus \cdots \boxplus \mathbf{W}^{(L-(k-1))} \boxplus \mathbf{w}^{(L-k)}_0 \Big)
\]
% \begin{multline*}
%     \mathbf{w}_{eq, 0} = \mathbf{w}^{(L)}_0 \vee \bigvee_{k=1}^{L-1} \Big( \mathbf{W}^{(L)} \boxplus \mathbf{W}^{(L-1)} \boxplus \cdots \\
%     \boxplus \mathbf{W}^{(L-(k-1))} \boxplus \mathbf{w}^{(L-k)}_0 \Big)
% \end{multline*}
\[
\mathbf{W}_{eq} = \mathbf{W}^{(L)} \boxplus \mathbf{W}^{(L-1)} \boxplus \cdots \boxplus \mathbf{W}^{(1)}.
\]

This expanded form shows that the network with \(L \geq 2\) is equivalent to a network with a single morphological layer, which is not a universal approximator. Hence, these types of networks require some sort of activation. For example in \citep{dimitriadis2021advances}, their "\(\delta\) network" has two max-plus MP layers, which effectively reduce to one. 

\paragraph{Min-max-plus networks and layers combining max-plus and min-plus MPs.}
One way to add complexity to the network is by including min-plus MP units. This can be done in two main ways. The principled way of doing it is by a discretization of the representation theorem, yielding alternating min-plus and max-plus layers. resulting in what is sometimes called an Alternating Sequential Filter (ASF) 
% \citep{DeMendonçaBragaNeto1996, Serr88}, 
\citep{serra1982image, serra1988image}. A second method is by incorporating both max-plus and min-plus MPs within the same layer, as proposed by 
% \citet{mondal2019dense,DiMa21}. 
\citet{mondal2019dense, dimitriadis2021advances}. While these architectures are more complex, and in the case of the min-max-plus network also come with theoretical guarantees, they both face fundamental limitations limiting their expressive power and trainability. We demonstrate that even with the inclusion of min-plus MPs, these networks remain limited in their representational power without the use of activations, and face severe optimization limitations related both to sparse, but also to uninformative gradients. 

By tracing the results of the max and min operations backwards, we can find a dependence of the output on the input and the parameters of the network. The situation is illustrated in Figure~\ref{fig:network_diff_nondiff} of Appendix~\ref{appendix:b}. We prove Theorems~\ref{theorem:1} and \ref{theorem:2}. 

\begin{figure*}
    \centering
    \begin{subfigure}{0.4\textwidth}
        \centering
        \includegraphics[width=\linewidth]{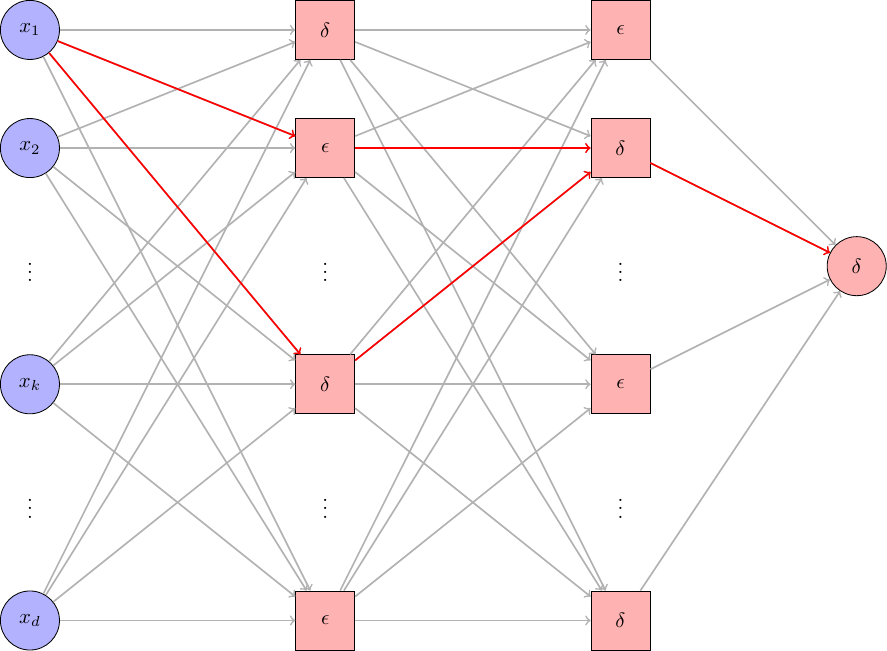}
        \caption{Differentiable with respect to input.}
        \label{fig:network_input_diff}
    \end{subfigure}
    \hspace{5pt}
    \begin{subfigure}{0.4\textwidth}
        \centering
        \includegraphics[width=\linewidth]{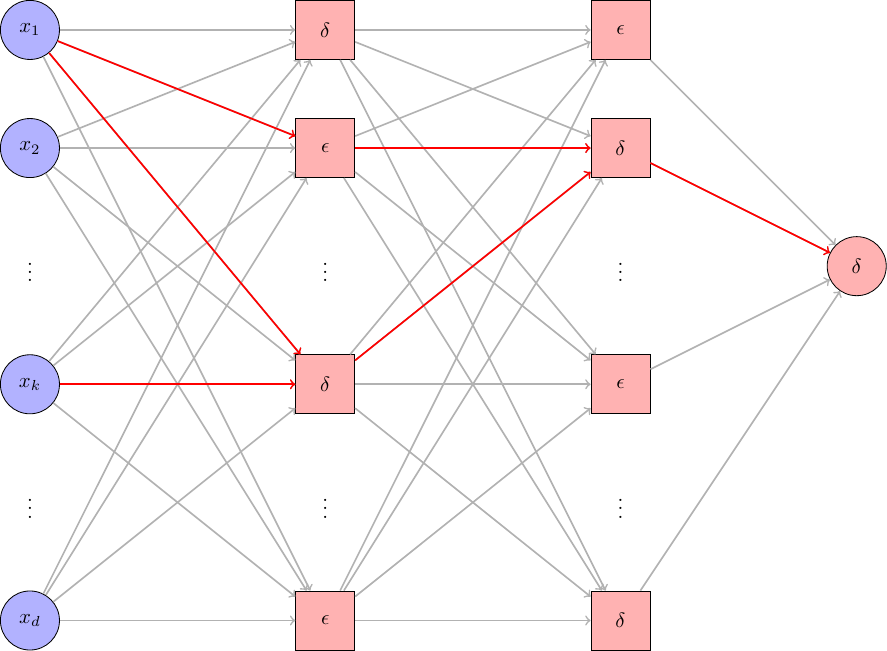}
        \caption{Non-differentiable with respect to input.}
        \label{fig:network_input_nondiff}
    \end{subfigure}

    \begin{subfigure}{0.4\textwidth}
        \centering
        \includegraphics[width=\linewidth]{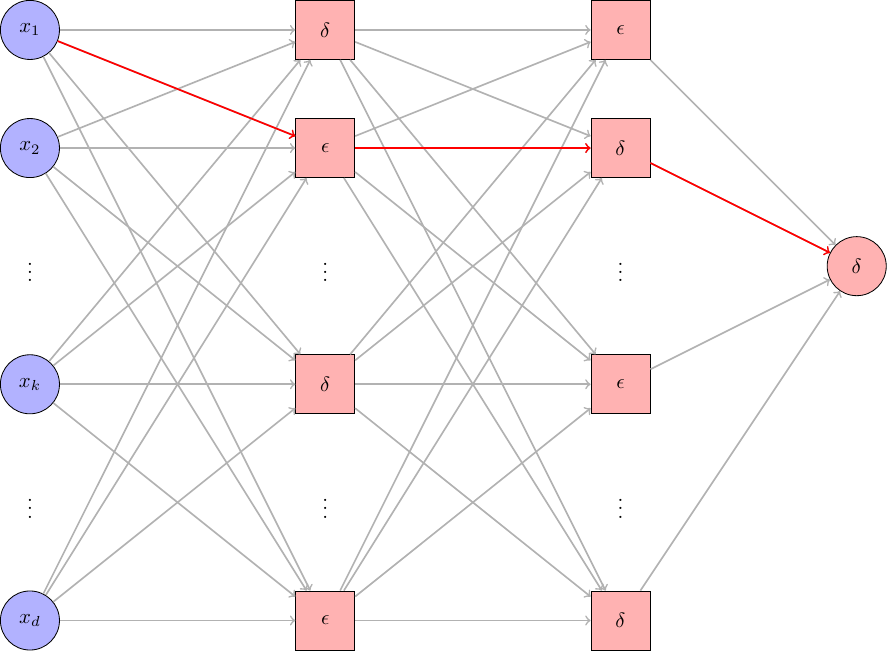}
        \caption{Differentiable with respect to weights.}
        \label{fig:network_weights_diff}
    \end{subfigure}
    \hspace{5pt}
    \begin{subfigure}{0.4\textwidth}
        \centering
        \includegraphics[width=\linewidth]{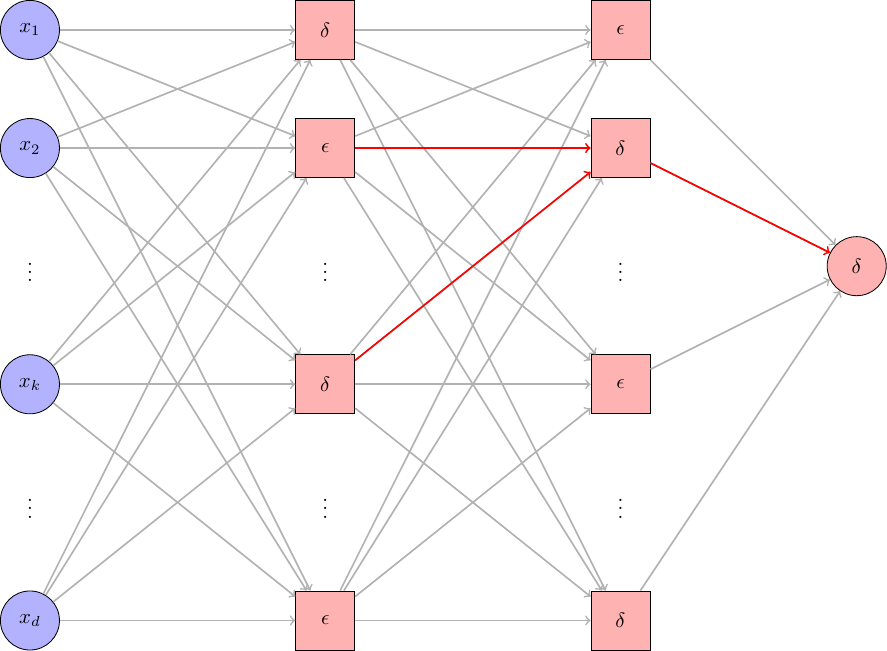}
        \caption{Non-differentiable with respect to weights.}
        \label{fig:network_weights_nondiff}
    \end{subfigure}
    \caption{Cases of differentiable and non-differentiable networks with respect to input and weights.}
    \label{fig:network_diff_nondiff}
\end{figure*}

\begin{restatable}{theorem}{first}
\label{theorem:1}
    For any network that \textbf{only} uses max-plus and min-plus MPs with input \(\mathbf{x}\in \mathbb{R}^d\) and a single output \(y(\mathbf{x})\), we have that \(y(\mathbf{x})\) is Lipschitz continuous on \(\mathbb{R}^d\) and a.e. it holds that either \(\nabla y(\mathbf{x})=0\) or \( \nabla y(\mathbf{x}) = \mathbf{e}_i = [0,\ldots,1,\ldots, 0]^\top \) for some \(i=i(\mathbf{x})\). The \(\mathbf{0}\) case exists due to biases. If one additionally allows negations for anti-dilation and anti-erosion terms, then we can have the additional possibility of \( \nabla y(\mathbf{x}) = -\mathbf{e}_i = [0,\ldots,-1,\ldots, 0]^\top\).
\end{restatable}

This result implies that such networks cannot be universal approximators. Consider, for instance, the function \(f(\mathbf{x})=2x_1\), whose gradient is \(\nabla f=2\mathbf{e}_1= [2,0,\ldots,0]^{\top}\). Since the networks described above can only produce gradients of the form \(\mathbf{e}_i\), they are incapable of representing even this simple linear function. For details, refer to Theorem~\ref{theorem:b5}.

It is worth investigating how Theorem~\ref{theorem:1} relates to the representation theorem. For this, let us see how the discretized representation theorem of min-max-plus networks achieves its approximations. To approximate a function \(f:\mathbb{R}^d \to \mathbb{R}\), the discrete representation theorem samples finite points \(S \subseteq \mathbb{R}^d\), and for each \(\mathbf{r} \in S\) it defines a "pyramid" \(f_i(\mathbf{x})=(\min_i(x_i-r_i)\wedge \min_i(-x_i+r_i'))+f(\mathbf{r})\). Then, it builds its final approximation as
\begin{equation*}
y(\mathbf{x}) = \max_if_i(\mathbf{x}) = \max_i(\min(x_i-r_i) \wedge \min(-x_i+r_i)) + f(\mathbf{r}).
\end{equation*}
If \(f\) does not change too quickly --- that is, if it is \(1\)-Lipschitz with respect to \(\|\cdot\|_{\infty}\) --- then these pyramids do not interfere on their tips, and taking their maximum yields an interpolation of the points \(\{(\mathbf{r}, f(\mathbf{r})) : \mathbf{r} \in S\} \subseteq \operatorname{graph}(f)\). The construction is given in Figure~\ref{fig:min-max-plus} for the function \(f(x)=x/2\).

If the function \(f\) is not \(1\)-Lipschitz, then the pyramids interfere with each other and the approximation breaks. This is covered by the more general Theorem~\ref{theorem:b5}. 
% (for the pure erosion case; a similar proof goes through when considering anti-dilation terms \(-x_i+r_i\) as well).
In Appendix~\ref{appendix:f}, we extend this result to the continuous case, and show that one of the byproducts of Theorem \ref{theorem:1} is that we can represent a linear shift-invariant filter as a supremum of weighted erosions exactly when the filter is increasing and translation-invariant. 

Interference of the pyramids was observed by \citet[Theorem 4.2]{luo2021min}, who added multipliers \\
\((1/L_1, \ldots, 1/L_d)\) on the inputs of the network to control the slopes of the pyramids and cover \((L_1, \ldots, L_d)\)-coordinate-wise-Lipschitz functions. This makes the multiplier augmented min-max-plus network a universal approximator of Lipschitz functions. Unfortunately, it is a very "weak" universal approximator: if a function's local Lipschitz continuity showcases a lot of variance, then the number of required pyramids explodes. This can already be seen from the function \(f^{(\epsilon)}(x)=\max(-|x|, -|x|/\epsilon+1)\) on \([-1,1]\) as \(\epsilon\to 0^+\): a PWL function with 4 pieces that requires an unbounded number of pyramids to cover the flat regions; we cover this later in Example~\ref{example:1}. 

\begin{figure}[t]
    \centering
    \begin{subfigure}{0.35\linewidth}
        \includegraphics[width=\linewidth]{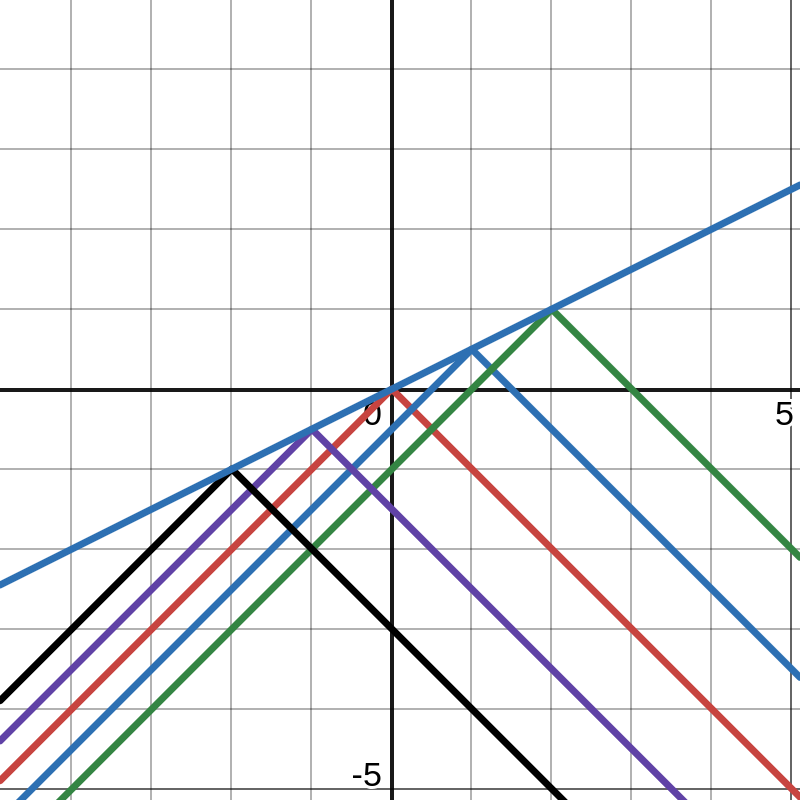}
        \caption{Pyramids for approximating \(x/2\). \(f_i(x) = \min(-x,x)+i/2\) for \(i \in \{-2,-1,0,1,2\}\).}
        \label{fig:pyramids}
    \end{subfigure}
    \hspace{5em}
    \begin{subfigure}{0.35\linewidth}
        \includegraphics[width=\linewidth]{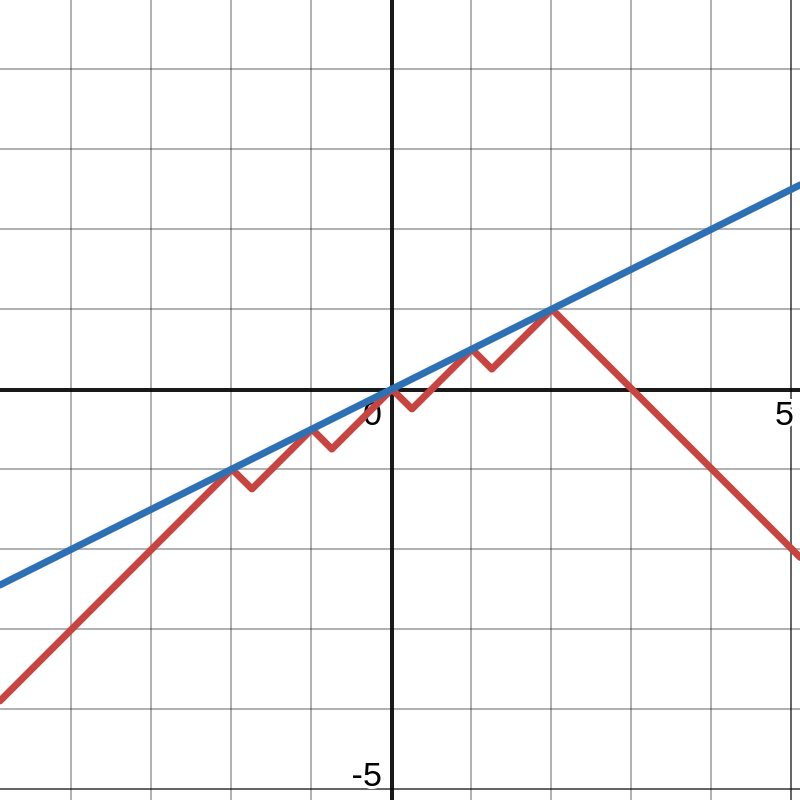}
        \caption{(Blue): \(x/2\), (Red): approximating max of minimums \(f(x) = \max_if_i(x)\).}
        \label{fig:max-pyramids}
    \end{subfigure}
    \caption{Min-max-plus construction.}
    \label{fig:min-max-plus}
\end{figure}

\begin{restatable}{theorem}{second}
\label{theorem:2}
Consider a network that \textbf{only} uses max-plus and min-plus MPs with output \(\mathbf{y}\in \mathbb{R}^m\). For any given input \(\mathbf{x}\), if \(\mathbf{y}\) is differentiable with respect to the network parameters, then in each layer \(n\), there exists at most \(m\) parameters \(w^{(n)}_{ij}\) for which the derivative of \(\mathbf{y}\) is nonzero. 
\end{restatable}

This result highlights a fundamental limitation in training such networks: the sparsity of their gradient signals during backpropagation. This was first noticed by \citet{dimitriadis2021advances}, and very recently revisited by \citet{dimitrova2026learning}. Consequently, gradient-based optimization methods become highly inefficient. 

However, sparsity is not the whole picture. A more important problem is the fact that the gradient is uninformative. Let us focus on the min-max-plus networks: No matter the cardinality of \(|S|\), the gradient \(\nabla f\) is never approximated well. Indeed, by Theorem~\ref{theorem:1}, every min-max-plus form has gradients almost everywhere (being Lipschitz continuous), and they are equal to either \(1\) or \(-1\) in exactly one argument, and \(0\) in the others. This holds for every function of the sequence \(f_i\) that approximates \(f\), meaning that the limit of the gradients does not converge to the gradient of the approximated function. In fact, we do not even have a good estimator of the gradient, since it is possible that \(\nabla_j f > 0\) while \(\nabla_j f_i < 0\) for some input \(j\). This is shown in Figure \ref{fig:bad-gradients}. The function behaves similarly to the Cantor function. This creates problems with back propagation, where the gradient with respect to the input of a layer is required for propagation of the gradient signal. Hence, the problem is not restricted merely to the sparsity of the gradients, but the fact that the min-max-plus network is only a \(0\)-th order universal approximator. 

\begin{figure}
    \centering
    \begin{subfigure}{0.35\linewidth}
        \includegraphics[width=\linewidth]{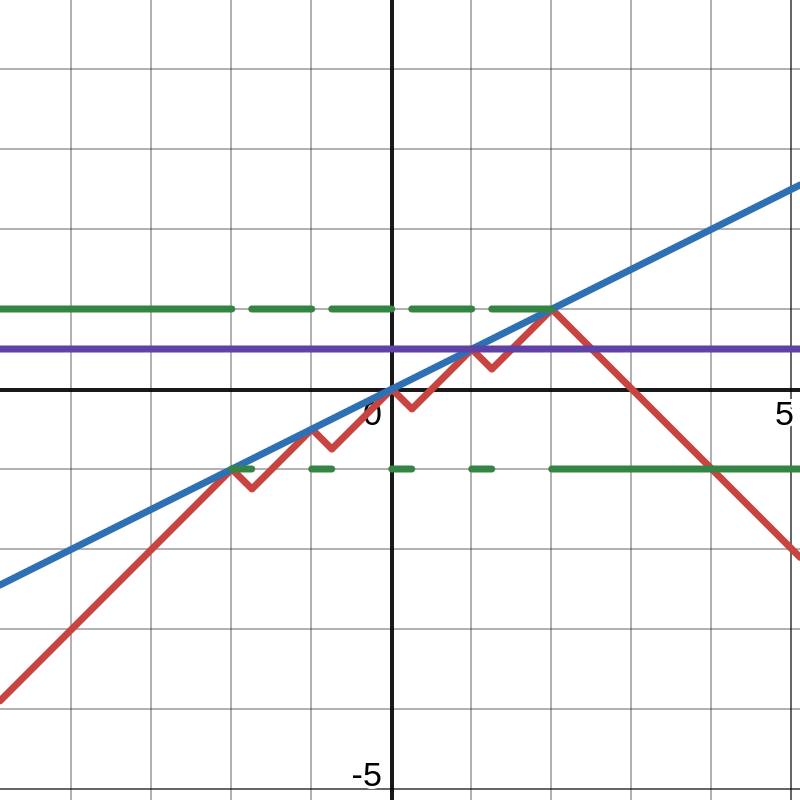}
        \caption{(Green): gradient of min-max-plus form. Compared with (purple): gradient of underlying function. The gradient is a bad estimator.}
        \label{fig:bad-gradients}
    \end{subfigure}
    \hspace{5em}
    \begin{subfigure}{0.35\linewidth}
        \includegraphics[width=\linewidth]{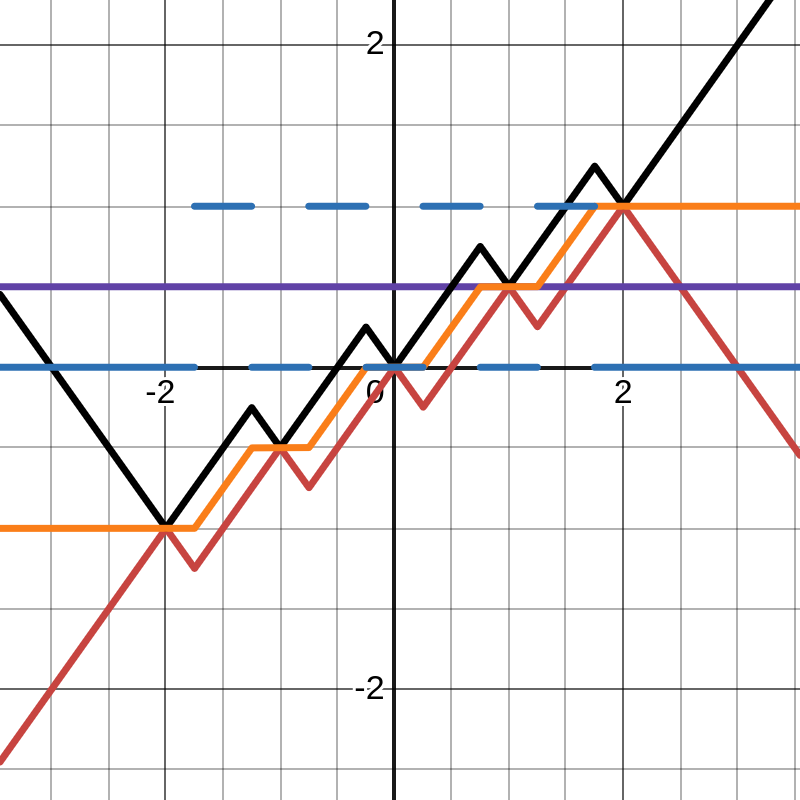}
        \caption{(Red): under-envelop, (black): over-envelop, (orange): average, (blue): gradient of average, (purple): true gradient.}
        \label{fig:MPM-gradients}
    \end{subfigure}
    \caption{Gradient approximation}
    \label{fig:gradient}
\end{figure}

Interestingly, linear networks with a ReLU activation correspond to PWL functions and are \(1\)-st order universal approximators. This allows efficient first-order optimization. They are not \(2\)-nd order universal approximators (with the second derivative being \(0\) almost everywhere), which might explain why SGD is so efficient at training linear networks, while second order methods are not. That is, linear networks suffer from the same exact training problem, one level higher.

\paragraph{DEP-based DMNNs.}
Another way of building a DMNN is by using DEP units in each layer 
% \citep{ARAUJO201712}
\citep{araujo2017morphological}. For a network with \(L\) layers we can write it in recursive form as follows:
% \[
% x^{(n)}_i = f^{(n)}(\lambda^{(n)}_i \max_{j\in [N^{(n-1)}]} (x^{(n-1)}_j+w^{(n)}_{ij}) + (1-\lambda^{(n)}_i) \min_{j\in [N^{(n-1)}]} (x^{(n-1)}_j + m^{(n)}_{ij})), 
% \]
\begin{equation*}
    x^{(n)}_i = f^{(n)}(\lambda^{(n)}_i \max_{j\in [N^{(n-1)}]} (x^{(n-1)}_j+w^{(n)}_{ij}) + (1-\lambda^{(n)}_i) \min_{j\in [N^{(n-1)}]} (x^{(n-1)}_j + m^{(n)}_{ij})), 
\end{equation*}
where \( \lambda^{(n)}_i \in [0,1], \forall i\in [N^{(n)}], n\in [L] \), and \(f^{(n)}\) a common activation function or identity. 

On first sight, this network, a generalization of the previous networks, looks like it solves the bounded slope and sparsity problems. By taking the sum, working backwards the gradient disperses in two paths and we have diffusion of the gradient across more parameters and more inputs. In addition, by multiplying the two terms by \(\lambda^{(n)}_i\) and \(1-\lambda^{(n)}_i\) we have more diversity in the values of the gradient. However, while it indeed solves the problem of sparse gradients, this network is too not a universal approximator, and in fact very limited in the range of functions it can represent, as the following theorem suggests. 

\begin{restatable}{theorem}{third}
\label{theorem:3}
For existing DEP-based networks with input \(\mathbf{x}\in \mathbb{R}^d\) and a single output \(y(\mathbf{x})\), we have that \(y(\mathbf{x})\) is Lipschitz continuous on \(\mathbb{R}^d\) and a.e. it holds that \( \nabla y(\mathbf{x}) \succeq 0, \|\nabla y(\mathbf{x})\|_{1} \leq 1 \). 
\end{restatable}

This result implies that existing DEP-based networks cannot be universal approximators. To see this, we can consider again the example of the function \(f(x)=2x_1\). For details, refer to Theorem \ref{theorem:b7}.

In addition to the above problem, DEP-based networks also showcase another problem. The existence of \(\lambda^{(n)}_i\) terms makes training slow. To understand why, suppose that at some point, either after initialization or during training, we have \(\lambda^{(n)}_i\neq 1/2\). Then, the weights \(\mathbf{w}^{(n)}_i, \mathbf{m}^{(n)}_i\) will have to follow distributions with different means in order for the output to be zero-mean. In addition, every time \(\lambda^{(n)}_i\) changes, in order for the output to remain zero-mean, the mean of the distributions of \(\mathbf{w}^{(n)}_i, \mathbf{m}^{(n)}_i\) has to change, which is very slow in morphological networks due to sparse gradients. In general, having \(\lambda\neq 1/2\), fixed or not, hinders trainability. For a detailed discussion we refer the reader to Appendix \ref{appendix:d}.

We seek to solve all of these problems. i) We need a network that is a universal approximator for any continuous function, ii) we need the network to be a powerful universal approximator in the sense that it is parameter efficient, iii) we need it to mitigate the sparsity of the gradients, iv) we need it to have more informative gradients, and v) we need it to be less sensitive to initiaization.

\paragraph{Our proposal.} 
We propose solutions to the above problems that can be summarized as follows: 

\begin{enumerate}[left=0pt, nosep]
    \item We introduce learnable "linear" activations between the morphological layers of the network. Depending on our constraints, the complexity of the "linear" activations varies. 
    \item We change the DEP-based architecture in the following ways:
    \begin{enumerate}[left=0pt, nosep]
        \item We use the same weights for the maximum and the minimum.
        \item We introduce biases.
        \item We remove the learnable parameters \(\lambda^{(n)}_i\) and simply take the average of the maximum and the minimum. 
    \end{enumerate}
\end{enumerate}

We develop networks under three different constraint \textbf{settings} based on the allocation of parameters:  
\begin{enumerate}[left=0pt, nosep]
    \item Each layer \( n \) of size \( N^{(n)} \) has at most \( O(N^{(n)}) \) parameters allocated for activations, with the remainder constrained to morphological operations.  
    \item Each layer \( n \) of size \( N^{(n)} \) has at most \( O(N^{(n)}) \) \textit{learnable} parameters allocated for activations, with the remainder constrained to morphological operations.
    \item A hybrid setting where no constraints are imposed on the number of parameters used in the "linear" activations. 
\end{enumerate}  

\paragraph{Setting 1.}  

In the most restrictive setting, we define a fully connected network, which we refer to as the \textit{Max-Plus-Min} (MPM) network. The network is recursively defined as follows:  
% \[
% x^{(n)}_i = \alpha^{(n)}_i \bigg( \bigg( w^{(n)}_{i0} \vee \max_{j\in [N^{(n-1)}]} (x^{(n-1)}_j + w^{(n)}_{ij}) \bigg) + \bigg( m^{(n)}_{i0} \wedge \min_{j\in [N^{(n-1)}]} (x^{(n-1)}_j + w^{(n)}_{ij}) \bigg) \bigg).
% \]  
\begin{equation*}
x^{(n)}_i = \alpha^{(n)}_i \bigg( \bigg( w^{(n)}_{i0} \vee \max_{j\in [N^{(n-1)}]} (x^{(n-1)}_j + w^{(n)}_{ij}) \bigg) + \bigg( m^{(n)}_{i0} \wedge \min_{j\in [N^{(n-1)}]} (x^{(n-1)}_j + w^{(n)}_{ij}) \bigg) \bigg).
\end{equation*}
Or, equivalently, as follows: 
% \[
% \mathbf{x}^{(n)} = \mathrm{diag}([\alpha^{(n)}_i]_{i \in N^{(n)}}) \Big( \Big( \mathbf{w}_0^{(n)} \vee \mathbf{W}^{(n)} \boxplus \mathbf{x}^{(n-1)} \Big) + \Big( \mathbf{m}_0^{(n)} \wedge \mathbf{W}^{(n)} \boxplus' \mathbf{x}^{(n-1)} \Big) \Big).
% \]
\begin{equation*}
\mathbf{x}^{(n)} = \mathrm{diag}([\alpha^{(n)}_i]_{i \in N^{(n)}}) \Big( \Big( \mathbf{w}^{(n)}_i \vee \mathbf{W}^{(n)} \boxplus \mathbf{x}^{(n-1)} \Big) + \Big( \mathbf{m}^{(n)}_i \wedge \mathbf{W}^{(n)} \boxplus' \mathbf{x}^{(n-1)} \Big) \Big).
\end{equation*}

Here, our proposed morphological layer is the addition of a max-plus MP layer and a min-plus MP layer sharing the same learnable weights \(\mathbf{W}^{(n)} \in \mathbb{R}^{N^{(n)} \times N^{(n-1)}}\) but with different learnable biases \(\mathbf{w}_0^{(n)}, \mathbf{m}_0^{(n)} \in \mathbb{R}^{N^{(n)}}\). The activation function is a simple scaling operation: after computing the sum of the maximum and minimum, each output \(x^{(n)}_i\) is multiplied by a learnable parameter \(\alpha^{(n)}_i \in \mathbb{R}\). The final layer is not activated, i.e. \(\alpha^{(L)}_i\) is fixed to \(1\).

For a layer \(n\) of size \(N^{(n)}\), the activation function introduces only \(N^{(n)} \in O(N^{(n)})\) additional parameters, with all remaining parameters dedicated to morphological operations. A central property of the MPM unit is that for bounded domains it can be decomposed into max-plus and min-plus MPs as follows: For sufficiently large \(C\) we have
\begin{align*}
\max_i(x_i+w_i)\vee w_0 &= (\max_i(x_i+w_i+C)\vee (w_0+C)) + (\min_i(x_i+w_i)\wedge -C), 
\\
\min_i(x_i+w_i)\wedge m_0 &= (\max_i(x_i+w_i)\vee C) + (\min_i(x_i+w_i-C)\wedge (m_0-C)).
\end{align*}
The presence of learnable parameters \(\alpha^{(n)}_i\) and the inclusion of bias terms allow us to establish the following result:  

\begin{restatable}{theorem}{fourth}  
\label{theorem:4}
    If the domain of the input is compact (i.e., bounded and closed), the Max-Plus-Min (MPM) network is a universal approximator. In fact, it can represent exactly any continuous PWL function with an additional \(O(\log d)\) factor of active parameters compared to the function's representation as a difference of maxout units (used in the proof of universality of maxout networks \citep{goodfellow2013maxout}), where \(d\) is the input dimension. This makes it a \(1\)-st order universal approximator. 
\end{restatable}  

With universal approximation being a property in the limit, a strength of Theorem \ref{theorem:4} is that the MPM does not require a prohibitively large number of parameters to achieve its stated expressivity guarantee. More broadly, it is interesting to reason about what the inductive bias of the MPM is. In this regard, we note the following:

\begin{itemize}[left=0pt, nosep]
\item The proof of Theorem \ref{theorem:4} reduces universal approximation on bounded domains to the approximation of piecewise affine functions. In particular, for input dimension \(d\), the construction implies that the MPM can represent any maxout function considered in the universality proof of \citet{goodfellow2013maxout} using at most \(O(\log d)\) layers of width \(O(dK)\), where \(K\) denotes the rank of the corresponding maxout construction, which uses \(\Theta(dK)\) parameters. However, the construction of Theorem \ref{theorem:4} only requires each MPM unit to have \(O(1)\) active inputs. Hence, the MPM represents the same function using at most \(O(d \log d \cdot K)\) active weights, incurring only an additional \(O(\log d)\) factor in active parameters relative to the corresponding maxout construction.

\item In practice, although the input dimension \(d\) may be large, the complexity of the target function often dominates the approximation problem, and the \(K\) needed for a good approximation dominates \(d\). The hierarchical construction underlying Theorem \ref{theorem:4} suggests an inductive bias in which the slopes, constructed progressively, share intermediate results, potentially allowing reduction of the width \(K\) required in practice and compensating for the additional \(O(\log d)\) factor relative to the corresponding maxout construction.  

\item We empirically investigate this behavior through aggressive pre-training pruning experiments. The results suggest that the MPM preserves its performance without requiring a blow-up of active parameters.  

\end{itemize}

Compared to architectures that use only max-plus and min-plus MPs, and the DEP architecture, the MPM is a universal approximator on bounded domains. In addition, it is a significantly more powerful approximator than the min-max-plus network with input multipliers. Having linear activations on the outputs of all layers instead of input multipliers improves the representation power over the min-max-plus network. Having additionally the summation of the maximum and the minimum allows \emph{deeper} MPM networks to represent continuous PWL functions exactly. This makes deep MPMs \(1\)-st order universal approximators, allowing their efficient training. While the MPM benefits from increasing depth, the min-max-plus network faces the optimization limitations outlined above.

There are advantages for shallower networks as well. Regarding expressivity, MPM networks can handle functions with varying slopes better; one demonstration of this is Example~\ref{example:1}. 

\begin{example}\label{example:1}
    Consider the uni-variate function 
    \[
    f^{(\epsilon)}(x)=\max(-|x|+\epsilon, -|x|/\epsilon+1)
    \] on \([-1, 1]\) for \(\epsilon > 0\). The optimal min-max-plus network approximating this function uses a multiplier \(1/\epsilon\) on the input, and \(2\operatorname{round}(\frac{1-\epsilon}{2\epsilon d})+1\) pyramids to achieve an approximation error of \(\delta(1-\epsilon^2)>0\). See Figure~\ref{fig:variation} for an illustration. For \(\epsilon \to 0^+\) and \(\delta(1-\epsilon^2) = \mathrm{const} \Rightarrow \delta \approx \mathrm{const}\), this requires an unbounded number of min-plus MPs. The MPM network, on the other hand, can represent the function exactly with 3 MPM units. For details, refer to Appendix~\ref{appendix:c}.
\end{example}

\begin{figure}
    \centering
    \includegraphics[width=0.35\linewidth]{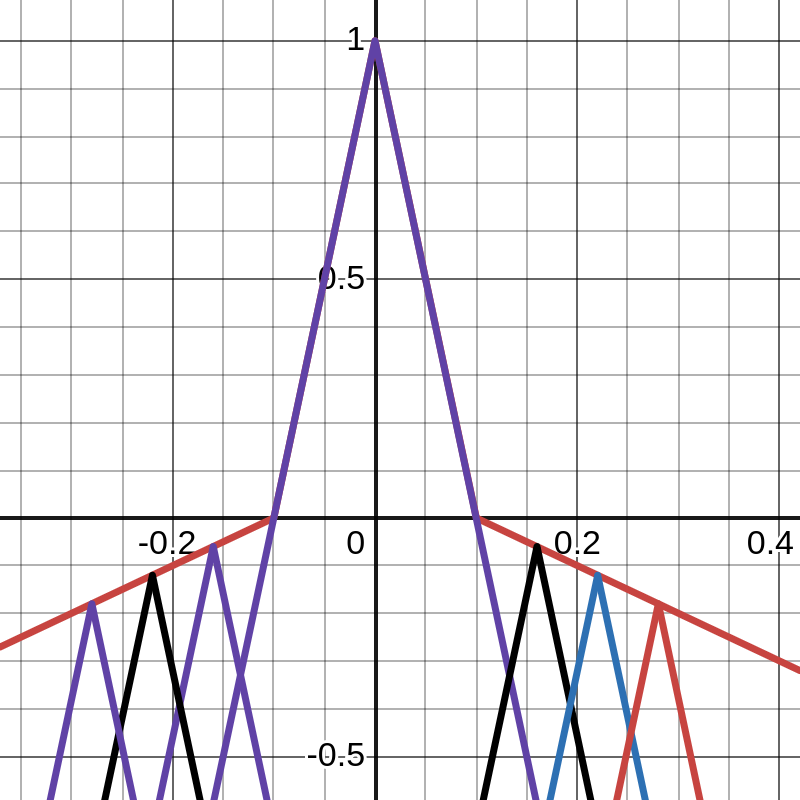}
    \caption{Min-max-plus approximation for a function with slope variation.}
    \label{fig:variation}
\end{figure}

The reason why the min-max-plus network underperformed in this example is because the function had variation in its slopes. By having multipliers as linear activations on all layers of the network, the pyramids are allowed to have varying slopes, and this problem is tackled. By additionally taking in mind the fact that the MPM unit can be decomposed into either a max-plus or min-plus MP, the shallow MPM network can function as a smoother interpolator: we can prove a 1D approximation result that is more closely related to the complexity of the underlying function (cf. the maxout construction used as the baseline for Theorem~\ref{theorem:4}).

\begin{restatable}{proposition}{eighth}\label{theorem:8}
    Suppose we have a compact interval \(\mathcal{I}\subseteq \mathbb{R}\) and an \(1\)-Lipschitz function \(f:\mathcal{I}\to \mathbb{R}\) with a finite number of inflection points \(N_i\). Then, a shallow \(4\)-layer Max-Plus-Min (MPM) network is an approximator asymptotically almost as efficient as the linear interpolator, requiring at most \(N_s + 2N_i + 5\) MPM units, with a total of \(O(N_s)\) active parameters and \(O(N_s+N_i)\) active connections to achieve the same error as an optimal linear interpolator using \(N_s\) samples. The min-max-plus network is, on the other hand, not efficient. 
\end{restatable}

Intuitively, a first minimum layer forms a set of pyramids, the linear activation changes their slopes, a second minimum layer can selectively choose which pyramids to use to form over-envelops of smoother concave shapes, and the final maximum layer uses the concave shapes to interpolate the function.  

Regarding the uninformative gradient problem on \emph{shallow} networks, notice in Figure~\ref{fig:MPM-gradients} that this can be mitigated by augmenting the under-envelop with an over-envelop, and taking their average. This way, the gradient becomes a better approximation of the true gradient. At the very least, it becomes a better-behaved estimator, because the estimated gradient and the underlying gradient have the same sign. After decomposing it to max-plus and min-plus MPs, notice that this is precisely what the MPM unit does: an averaging with shared weights, and so the MPM encompasses this class of functions as well. 

Note that in higher dimensions than one, there exist edge case examples where conservation of the sign does not hold. We leave to future work a theoretical closure of the gap between the shallow networks' and deep networks' analysis and a formal result capturing the above intuition of the gradient of the shallow MPM.

If we had no biases and the max-plus and min-plus weights were different, then the MPM network would be equivalent to a DEP architecture with fixed parameters \(\lambda^{(n)}_i = 0.5\) for all \(n, i\), followed by a learnable scaling for each output. In the Experiments, we show that: i) by introducing biases and sharing the same weights we get slightly better training accuracy and generalization, and ii) both introducing the learnable scaling and taking \(\lambda^{(n)}_i=0.5\) are essential for the network to be trainable. The importance of learnable scaling is highlighted in the 1D regression toy examples of Figure~\ref{fig:regression} in the Appendix. 

\paragraph{Improving generalization.}

The Max-Plus-Min (MPM) network struggles with generalization. To mitigate this, we introduce the Residual-Max-Plus-Min (RMPM) network, which incorporates residual connections for layers where input and output dimensions match. Specifically, we add the input to the activated output before propagating it to the next layer. This residual mechanism accelerates training and slightly improves generalization. 

\begin{restatable}{corollary}{sixth}
\label{theorem:6}
    If the domain of the input is compact (i.e., bounded and closed), and assuming no residual skip connection for the output layer, the Residual-Max-Plus-Min (RMPM) network is a universal approximator with the same guarantees as the MPM. 
\end{restatable}

To further improve generalization, we encourage the RMPM to learn robust representations by introducing \textit{weight dropout} during training -- randomly dropping a percentage of weights during training forward passes. 

\paragraph{Setting 2.}  

In Setting 1, the activation function is effectively a linear transformation with a learnable diagonal matrix. However, a richer class of linear transformations is desirable to enhance gradient diffusion. In Setting 2, our goal is to achieve this expressivity while maintaining a count of \textit{learnable} activation parameters of at most \(O(N^{(n)})\) per layer. Additionally, we prefer that the transformation be full-rank to maximize information propagation.  

This motivates our second architecture, the \textit{Max-Plus-Min-SVD} (MPM-SVD) network, defined as follows:  

% \[
% y^{(n)}_i = \left( w^{(n)}_{i0} \vee \max_{j\in [N^{(n-1)}]} (x^{(n-1)}_j + w^{(n)}_{ij}) \right) + \left( m^{(n)}_{i0} \wedge \min_{j\in [N^{(n-1)}]} (x^{(n-1)}_j + w^{(n)}_{ij}) \right), \quad i\in N^{(n)}, n\in [L],
% \]  
% \[
% \mathbf{x}^{(n)} = U^{(n)} \mathrm{diag}(\sigma^{(n)}_1, \ldots, \sigma^{(n)}_{N^{(n)}}) (V^{(n)})^\top \mathbf{y}^{(n)}, \quad n\in [L],
% \]  

% or equivalently as follows: 
\[
\mathbf{y}^{(n)} = \left( \mathbf{w}^{(n)}_0 \vee \mathbf{W}^{(n)} \boxplus \mathbf{x}^{(n-1)} \right) + \left( \mathbf{m}^{(n)}_0 \wedge \mathbf{W}^{(n)} \boxplus' \mathbf{x}^{(n-1)} \right),
\]
% \begin{multline*}
% \mathbf{y}^{(n)} = \left( \mathbf{w}^{(n)}_i \vee \mathbf{W}^{(n)} \boxplus \mathbf{x}^{(n-1)} \right) \\
% + \left( \mathbf{m}^{(n)}_i \wedge \mathbf{W}^{(n)} \boxplus' \mathbf{x}^{(n-1)} \right),
% \end{multline*}
\[
\mathbf{x}^{(n)}=U^{(n)} \mathrm{diag}(\sigma^{(n)}_1, \ldots, \sigma^{(n)}_{N^{(n)}}) (V^{(n)})^\top \mathbf{y}^{(n)}.
\]

Again, \(\mathbf{W}^{(n)} \in \mathbb{R}^{N^{(n)} \times N^{(n-1)}}\) and \(\mathbf{w}_0^{(n)}, \mathbf{m}_0^{(n)} \in \mathbb{R}^{N^{(n)}}\) are learnable morphological parameters. \(\{\sigma^{(n)}_i\}_i \in \mathbb{R}^{N^{(n)}}\) are learnable scaling parameters, while \(U^{(n)}, V^{(n)} \in \mathbb{R}^{N^{(n)} \times N^{(n)}}\) are \textbf{fixed}, random orthonormal matrices. These matrices are initialized as follows: we sample a matrix using Kaiming initialization \citep{he2015delving} and compute its singular value decomposition (SVD) to obtain \(U, V\), ensuring a well-conditioned transformation. While the number of learnable scaling parameters per layer remains \(N^{(n)} \in O(N^{(n)})\), the total parameter count, including the fixed matrices \(U\) and \(V\), scales as \(\Theta((N^{(n)})^2)\). The final layer is not activated, i.e. \(\sigma^{(L)}_i\) are fixed to \(1\) and \(U^{(L)}=V^{(L)}=I_{N^{(L)}}\).

At this point we should emphasize that networks such as those of \citet{mondal2019dense,dimitriadis2021advances, valle2020reduced} incorporate fully connected linear layers and do not fall in Settings 1 or 2, but in Setting 3. To the best of our knowledge, we are the first to provide trainable networks in settings as restrictive as 1 and 2.

\paragraph{Setting 3.}
In this setting, we impose no constraints on the number of parameters in the linear transformations. We study networks following a hybrid architecture, alternating between linear and our proposed morphological layers. We refer to this model as the Hybrid-MLP, defined as follows:
\[
\mathbf{y}^{(n)}=\mathbf{A}^{(n)}\mathbf{x}^{(n-1)} + \mathbf{b}^{(n)},
\]
\[
\mathbf{x}^{(n)}=\left( \mathbf{w}^{(n)}_0 \vee \mathbf{W}^{(n)} \boxplus \mathbf{y}^{(n)} \right) + \left( \mathbf{m}^{(n)}_0 \wedge \mathbf{W}^{(n)} \boxplus' \mathbf{y}^{(n)} \right), 
\]
% \[
% \mathbf{x}^{(L)}=\mathbf{y}^{(L)}.
% \]

Here, \(\mathbf{A}^{(n)} \in \mathbb{R}^{N^{(n)} \times N^{(n-1)}}, \mathbf{b}^{(n)} \in \mathbb{R}^{N^{(n)}}, \mathbf{W}^{(n)} \in\mathbb{R}^{N^{(n)} \times N^{(n)}}, \mathbf{w}^{(n)}_0 \in \mathbb{R}^{N^{(n)}}, \mathbf{m}^{(n)}_0 \in \mathbb{R}^{N^{(n)}}\) are all trainable parameters. The Hybrid-MLP is effectively an MLP whose ReLU activations have been replaced with our proposed morphological layers. We should note that the biases \(\mathbf{b}^{(n)}\) are theoretically not necessary. 

\begin{restatable}{theorem}{fifth}
\label{theorem:5}
If the domain of the input is compact (i.e., bounded and closed), the Hybrid-MLP is a universal approximator. In fact, any fully connected ReLU or maxout network is a special case of the Hybrid-MLP. 
\end{restatable}

\paragraph{Convolutional networks.}

So far we have focused on building fully connected DMNNs. We can extend our insights from these networks and build convolutional networks. Convolutional networks will be based on the morphological convolution. The morphological layer we propose takes the sum of a dilation and an erosion with shared weights and different biases. For Setting 1, for a layer with \(N^{(n)}\) output channels, we activate the layer by linearly convoluting each channel with a learnable \(3\times 3\) matrix, obtaining \(9N^{(n)}\in O(N^{(n)})\) parameters in total. For Setting 2, we activate the layer by a linear convolutional layer, which has weight matrix \(A^{(n)}\in \mathbb{R}^{N^{(n)}\times N^{(n)}\times 3\times 3}\), initialized according to Glorot. For each \(i, j\in [3]\) we write \(A^{(n)}_{:,:,i,j}=U^{(n)}_{i,j}\Sigma^{(n)}_{i,j}(V^{(n)}_{i,j})^\top\), fix \(U^{(n)}_{i,j}, V^{(n)}_{i,j}\), and take \(\Sigma^{(n)}_{i,j}\) to be a learnable diagonal matrix, obtaining \(9N^{(n)}\in O(N^{(n)})\) learnable parameters. For Setting 3, we alternate between linear convolutional layers and our proposed morphological convolutional layers. 

\section{Experiments}

To test the efficacy of our networks, we conduct experiments on the MNIST, Fashion-MNIST, \citep{deng2012mnist, xiao2017fashionmnist}, and CIFAR-10 \citep{krizhevsky2009learning} datasets. Unless otherwise stated, we use a batch size of \(64\), a random \(80/20\) train-validation split, and the Adam optimizer \citep{kingma2015adam} with a learning rate of \(0.001\) for 50 epochs. Our loss function is the Cross Entropy Loss. Throughout, we report mean and std of accuracy for different runs. The networks were initialized according to the remarks of Appendix \ref{appendix:d}. For 
% experiments on the Hybrid-MLP, 
additional experimental details, additional experiments, and number of parameters of each model, refer to Appendix~\ref{appendix:e}. 

\subsection{Fully Connected Networks}

In this section, we evaluate various fully connected networks on the MNIST and Fashion-MNIST datasets. Our primary objectives are:  
i) to demonstrate that our proposed networks are trainable, and  
ii) to highlight the necessity of our proposed modifications through systematic ablation. We train the following networks:  

\begin{itemize}[left=0pt, nosep]
    \item \textbf{MLP}: A standard ReLU-activated multilayer perceptron.  
    \item \textbf{MP}: A max-plus MP-based DMNN.  
    \item \textbf{DEP}: A non-activated DEP-based DMNN.
    \item \textbf{DEP (\(\lambda=1/2\))}: A non-activated DEP-based DMNN with fixed \(\lambda=1/2\).
    \item \textbf{Act-MP}: A max-plus MP-based DMNN with activation applied according to our proposed method in Setting 1.  
    \item \textbf{Act-DEP}: A DEP-based DMNN with learnable \(\lambda^{(n)}_i\), activated according to our proposed method in Setting 1.  
    \item \textbf{Act-DEP (\(\lambda=3/4\))}: A DEP-based DMNN with fixed \(\lambda^{(n)}_i = 3/4\), activated according to our proposed method in Setting 1.
    \item \textbf{Act-DEP (\(\lambda=1/2\))}: A DEP-based DMNN with fixed \(\lambda^{(n)}_i = 1/2\), activated according to our proposed method in Setting 1.  
    \item \textbf{MinMaxPlus}: A min-max-plus network \citep{luo2021min} with input multipliers. Each "layer" is a min-max combination layer.
    \item \textbf{MPM}: Our proposed Max-Plus-Min network for Setting 1.  
    \item \textbf{RMPM}: Our proposed Residual-Max-Plus-Min network for Setting 1.  
    \item \textbf{RMPM-Drop}: An RMPM with 0.3 weight dropout, trained for 200 epochs.  
    \item \textbf{MPM-SVD}: Our proposed Max-Plus-Min-SVD network for Setting 2.  
\end{itemize}

\begin{table}
\caption{Accuracies of fully connected networks}
\begin{subtable}{0.45\linewidth}
  \caption{Train (peak) and test accuracy of fully connected networks on MNIST.}
  \label{table:fully-connected-mnist}
  \centering
  \begin{tabular}{lcc}
    \toprule
    \textbf{Network} & \textbf{Train (\%)} & \textbf{Test (\%)} \\
    \midrule
    MLP & 99.88 $\pm$ 0.04 & 98.01 $\pm$ 0.08 \\
    \midrule
    MP & 31.24 $\pm$ 1.51 & 31.59 $\pm$ 1.28 \\
    DEP & 76.63 $\pm$ 3.46 & 76.51 $\pm$ 3.36 \\
    DEP (\(\lambda=1/2\)) & 76.96 $\pm$ 1.13 & 77.64 $\pm$ 1.03 \\
    Act-MP & 66.17 $\pm$ 14.17 & 65.27 $\pm$ 13.97 \\
    Act-DEP & 84.51 $\pm$ 1.16 & 84.15 $\pm$ 1.03 \\
    Act-DEP (\(\lambda=3/4\)) & 94.04 $\pm$ 0.37 & 92.29 $\pm$ 0.64 \\
    Act-DEP (\(\lambda=1/2\)) & 99.15 $\pm$ 0.23 & 94.43 $\pm$ 0.26 \\
    MinMaxPlus & 76.92 $\pm$ 1.09 & 77.29 $\pm$ 1.07 \\
    MPM & 99.82 $\pm$ 0.04 & 94.66 $\pm$ 0.13 \\
    RMPM & 99.99 $\pm$ 0.00 & 95.52 $\pm$ 0.22 \\
    RMPM-Drop & 99.85 $\pm$ 0.03 & 97.49 $\pm$ 0.12 \\
    \midrule
    MPM-SVD & 99.99 $\pm$ 0.00 & 96.14 $\pm$ 0.04 \\
    \bottomrule
  \end{tabular}
\end{subtable}
\hspace{2em}
\begin{subtable}{0.45\linewidth}
  \caption{Train (peak) and test accuracy of fully connected networks on Fashion-MNIST.}
  \label{table:fully-connected-fashion-mnist}
  \centering
  \begin{tabular}{lcc}
    \toprule
    \textbf{Network} & \textbf{Train (\%)} & \textbf{Test (\%)} \\
    \midrule
    MLP & 98.17 $\pm$ 0.12 & 88.82 $\pm$ 0.23 \\
    \midrule
    MP & 22.34 $\pm$ 3.04 & 22.05 $\pm$ 2.93 \\
    DEP & 66.41 $\pm$ 2.08 & 65.30 $\pm$ 2.38 \\
    DEP (\(\lambda=1/2\)) & 70.99 $\pm$ 2.12 & 70.30 $\pm$ 2.29 \\
    Act-MP & 49.37 $\pm$ 9.84 & 48.50 $\pm$ 9.54 \\
    Act-DEP & 68.52 $\pm$ 2.26 & 66.79 $\pm$ 2.25 \\
    Act-DEP (\(\lambda=3/4\)) & 82.82 $\pm$ 0.89 & 79.26 $\pm$ 0.64 \\
    Act-DEP (\(\lambda=1/2\)) & 95.13 $\pm$ 0.59 & 82.58 $\pm$ 0.39 \\
    MinMaxPlus & 74.29 $\pm$ 1.00 & 72.84 $\pm$ 0.79 \\
    MPM & 98.42 $\pm$ 0.01 & 82.86 $\pm$ 0.17 \\
    RMPM & 99.66 $\pm$ 0.03 & 84.24 $\pm$ 0.26 \\
    RMPM-Drop & 96.07 $\pm$ 0.19 & 86.88 $\pm$ 0.19 \\
    \midrule
    MPM-SVD & 99.63 $\pm$ 0.02 & 84.72 $\pm$ 0.12 \\
    \bottomrule
  \end{tabular}
\end{subtable}
\end{table}

All networks consist of \(5\) hidden layers of size \(256\). Dropout makes convergence slower, hence RMPM-Drop was trained for 200 epochs. The results are reported in Tables \ref{table:fully-connected-mnist}, \ref{table:fully-connected-fashion-mnist}. As expected, learnable "linear" activation is crucial for training these networks, with activated networks consistently outperforming their non-activated counterparts. Moreover, both MPM and Act-DEP (\(\lambda=1/2\)) outperform the models with learnable \(\lambda\) and fixed \(\lambda=3/4\), highlighting the importance of summing the maximum and the minimum for trainability. 

MPM further benefits from biases and using the same weights for both the dilation and the erosion, resulting in a slight edge in training accuracy and generalization. The networks Act-DEP (\(\lambda=1/2\)), MPM, RMPM, RMPM-Drop, and MPM-SVD successfully train (i.e. reach satisfactory train accuracy peaks and convergence; for convergence results refer to Appendix \ref{appendix:e}) and achieve training accuracy comparable to that of a standard MLP. Among morphological networks without dropout, RMPM demonstrates slightly better generalization than MPM, while MPM-SVD -- operating under a different setting -- achieves the best generalization. With weight dropout, RMPM-Drop improves generalization significantly, achieving on MNIST/Fashion-MNIST test accuracies 97.49\%/86.88\%, 0.52\%/1.94\% lower compared to the linear MLP. Interesting, weight dropout in Setting 1 was a more successful strategy at incorporating all weights into training and improving generalization than what was relaxing the setting to 2 with the MPM-SVD. 

As expected, the MPM networks surpassed the MinMaxPlus network, which showcased both limited expressivity and a slow convergence, highlighting the need for linear activations instead of just input multipliers. Note that for the MinMaxPlus network, we considered the min-max combination as the "atomic layer", meaning the network had about double the number of parameters.

\subsection{Convolutional networks}

In this section, we evaluate various convolutional networks on the MNIST, Fashion-MNIST, and CIFAR-10 datasets. We train the following networks:  

\begin{itemize}[left=0pt, nosep]
    \item \textbf{LeNet-5}: A variant of the standard linear LeNet-5 with max-pooling. 
    \item \textbf{MPM-LeNet-5}: A morphological LeNet-5 according to Setting 1. Both convolutional and fully connected layers are morphological.
    \item \textbf{MPM-SVD-LeNet-5}: A morphological LeNet-5 according to Setting 2. Both convolutional and fully connected layers are morphological.
    \item \textbf{ResNet-20}: A variant of the standard, linear ResNet-20. 
    \item \textbf{MPM-ResNet-20}: A morphological ResNet-20 according to Setting 1, trained for 100 epochs. All convolutional layers are morphological. The classification layer is linear. 
\end{itemize}

For MPM-LeNet-5 and MPM-SVD-LeNet-5, morphological convolutions were implemented with a LogSumExp scheme for compatibility with PyTorch. For MPM-ResNet-20, we developed a CUDA module\footnote{see Appendix \ref{appendix:e} for details thereof} for PyTorch implementing max-plus convolution under its strict definition. For all LeNet-5 networks, a slight variation was adopted: incorporating max-pooling instead of average pooling and using \(5\times 5\) kernels with a padding of \(1\) for all convolutional layers. For ResNet-20 networks, we used max-pooling instead of stride for down-sampling. Although the validation accuracy converged within 50 epochs, for MPM-ResNet-20 the training accuracy took longer to converge, and thus we trained it for 100 epochs. We did not use any form of dropout. The results are reported in Tables \ref{table:convolutional-mnist}, \ref{table:convolutional-fashion-mnist}, \ref{table:convolutional-cifar-10}. On MNIST, it is clear that the networks we propose are trainable. 
% (i.e. reach satisfactory train accuracy peaks and convergence; for convergence results refer to Appendix \ref{appendix:e}). 
It is also clear that they benefit from incorporating convolutions, showcasing a clear jump in generalization compared to fully connected networks. However, they lag behind the linear LeNet-5. On Fashion-MNIST, the MPM-LeNet-5 and MPM-SVD-LeNet-5 networks struggled to train as effectively as their fully connected counterparts. The deeper MPM-ResNet-20 network, on the other hand, trained successfully and reached 89.38\% test accuracy, comparable to the linear LetNet-5 but slightly lower than the linear ResNet-20. On CIFAR-10, MPM-ResNet-20 reached satisfactory training accuracy, and 62.14\% test accuracy. Although the accuracy is lower than that of a linear ResNet-20, the results show that our proposed morphological networks are capable of learning meaningful representations. To the best of our knowledge, this is the first demonstration of training such networks on CIFAR-10. For comparison, a linear LeNet-5 typically achieves around 60–65\% test accuracy on this dataset.

\begin{table}
\caption{Accuracies of convolutional networks}
\begin{subtable}{0.45\linewidth}
  \caption{Train (peak) and test accuracy of convolutional networks on MNIST.}
  \label{table:convolutional-mnist}
  \centering
  \begin{tabular}{lcc}
    \toprule
    \textbf{Network} & \textbf{Train (\%)} & \textbf{Test (\%)} \\
    \midrule
    LeNet-5 & 99.99 $\pm$ 0.01 & 99.03 $\pm$ 0.07 \\
    MPM-LeNet-5 & 98.46 $\pm$ 0.76 & 97.08 $\pm$ 0.30 \\
    MPM-SVD-LeNet-5 & 99.21 $\pm$ 0.19 & 97.25 $\pm$ 0.28 \\
    \bottomrule
  \end{tabular}
\end{subtable}
\hspace{2em}
\begin{subtable}{0.45\linewidth}
  \caption{Train (peak) and test accuracy of convolutional networks on Fashion-MNIST.}
  \label{table:convolutional-fashion-mnist}
  \centering
  \begin{tabular}{lcc}
    \toprule
    \textbf{Network} & \textbf{Train (\%)} & \textbf{Test (\%)} \\
    \midrule
    LeNet-5 & 99.20 $\pm$ 0.07 & 90.14 $\pm$ 0.10 \\
    MPM-LeNet-5 & 85.22 $\pm$ 1.56 & 82.12 $\pm$ 1.47 \\
    MPM-SVD-LeNet-5 & 88.57 $\pm$ 0.72 & 84.52 $\pm$ 0.28 \\
    ResNet-20 & 99.69 $\pm$ 0.05 & 92.42 $\pm$ 0.04 \\
    MPM-ResNet-20 & 96.30 $\pm$ 0.22 & 89.38 $\pm$ 0.26 \\
    \bottomrule
  \end{tabular}
\end{subtable}
\centering
\begin{subtable}{0.45\linewidth}
  \caption{Train (peak) and test accuracy of convolutional networks on CIFAR-10.}
  \label{table:convolutional-cifar-10}
  \centering
  \begin{tabular}{lcc}
    \toprule
    \textbf{Network} & \textbf{Train (\%)} & \textbf{Test (\%)} \\
    \midrule
    ResNet-20 & 99.10 $\pm$ 0.11 & 81.74 $\pm$ 0.27 \\
    MPM-ResNet-20 & 95.34 $\pm$ 0.80 & 62.14 $\pm$ 0.22 \\
    \bottomrule
  \end{tabular}
\end{subtable}
\end{table}

\subsection{Pruning}
\textit{Revision note: This version corrects an implementation issue affecting only the SNIP pruning experiments reported in an earlier preprint version. All SNIP pruning results have been recomputed and the corresponding discussion has been revised. All other experiments, theoretical results, and conclusions are unchanged.
}

We study pruning as a means of investigating whether the proposed architectures require prohibitively many active parameter counts to achieve expressive representations in practice. For lack of a pruning method compatible with morphological architectures, we consider: i) the pre-training pruning method SNIP \citep{lee2018snip}, and ii) unstructured \(\ell_1\) masking,

SNIP is applied as standard: redundant weights are identified and are completely removed from the network. We apply a slightly modified version of the standard procedure. Since SNIP relies on gradient-based importance scores and morphological networks exhibit sparse gradients (Theorem \ref{theorem:2}), directly applying SNIP can produce unstable pruning scores. To mitigate this issue, we apply SNIP after a short warm-up phase and aggregate importance scores across multiple batches. The same procedure is used for both linear and morphological networks. In addition, for the convolutional MPM architectures, pruning was restricted to the morphological backbone parameters, while the constrained linear activations and final classification layer were kept fixed. Empirically, the linear activations were systematically assigned low importance by SNIP. For aggressively pruned MPMs, we additionally maintain validity masks to avoid degenerate max/min units with all inputs removed. The results are reported in Tables \ref{table:pruning-snip} and \ref{table:pruning-snip-res}. 

We can see that, even after aggressive pre-training pruning, the fully connected RMPM and convolutional MPM-ResNet-20 maintain stable performance, indicating that an excessive number of active parameters are not required for the networks to achieve satisfactory expressivity. 

For \(\ell_1\) masking, in each layer a given fraction \(r\) of low-importance weights are set to zero (i.e., they are not removed), where importance is measured by the absolute value of the weight. Note that in the case of morphological networks, this means signal still potentially passes through the connection. The practical advantage is that this produces sparse parameter matrices that can be stored more efficiently. More speculatively, if pruning is done in blocks, then maximums and minimums of blocks of inputs can be calculated once for multiple outputs, resulting in a potential speedup under suitable hardware support. The results are reported in Table \ref{table:pruning} of Appendix \ref{appendix:e}. 

Finally, we note that SNIP and related pruning methods were originally developed for linear networks, where removing a connection is modeled by setting its weight to zero. In morphological networks, removing a connection does not correspond to setting its weight to zero, and the associated importance scores are not aligned with the underlying pruning objective. To the best of our knowledge, there is currently no broadly adopted pruning method specifically designed for deep morphological networks; our experiments only serve the purpose of proving that our proposed networks do not require prohibitively many parameters to achieve expressivity, not fairly comparing the pruning performance of morphological networks to linear ones. 

\begin{table}
\caption{Performance of SNIP-pruned networks.}
\begin{subtable}{\linewidth}
\centering
\caption{Performance of SNIP pruned MLP and RMPM for various pruning ratios on MNIST and Fashion-MNIST.}
\label{table:pruning-snip}
\begin{tabular}{cccccccc}
\toprule
\multicolumn{4}{c}{MLP (initial params: 466698)} &
\multicolumn{4}{c}{RMPM (initial params: 469268)} \\
\cmidrule(lr){1-4} \cmidrule(lr){5-8}
\multirow{2}{*}{\shortstack[l]{Pruning\\ ratio}} & \multirow{2}{*}{\shortstack[l]{Params\\ (kept)}} & MNIST & F-MNIST &
\multirow{2}{*}{\shortstack[l]{Pruning\\ ratio}} & \multirow{2}{*}{\shortstack[l]{Params\\ (kept)}} & MNIST & F-MNIST\\
& & & & & & & \\
\midrule
0.9875 & 5839 & 95.41 $\pm$ 0.24 & 85.05 $\pm$ 0.06 & 
0.9875 & 5895 & 94.61 $\pm$ 0.16 & 84.32 $\pm$ 0.23 \\
0.9900 & 4672 & 94.98 $\pm$ 0.16 & 84.32 $\pm$ 0.13 & 
0.9900 & 4722 & 94.22 $\pm$ 0.07 & 83.72 $\pm$ 0.12 \\
0.9925 & 3506 & 93.97 $\pm$ 0.42 & 82.92 $\pm$ 0.48 & 
0.9925 & 3549 & 93.26 $\pm$ 0.22 & 82.95 $\pm$ 0.17 \\
0.9950 & 2340 & 88.41 $\pm$ 4.99 & 79.09 $\pm$ 1.33 & 
0.9950 & 2376 & 90.17 $\pm$ 0.27 & 79.98 $\pm$ 0.22 \\
\bottomrule
\end{tabular}
\end{subtable}

\begin{subtable}{\linewidth}
\centering
\caption{Performance of SNIP pruned ResNet-20 and MPM-ResNet-20 for various pruning ratios on Fashion-MNIST and CIFAR-10. Note that F-MNIST is grayscale while CIFAR-10 is RGB, leading to a slight discrepancy in the initial parameters for the two datasets; both are reported separated by a slash "/".}
\label{table:pruning-snip-res}
\begin{tabular}{cccccccc}
\toprule
\multicolumn{4}{c}{ResNet-20 (initial params: 271994 / 272282)} &
\multicolumn{4}{c}{MPM-ResNet-20 (initial params: 276810 / 277098)} \\
\cmidrule(lr){1-4} \cmidrule(lr){5-8}
\multirow{2}{*}{\shortstack[l]{Pruning\\ ratio}} & \multirow{2}{*}{\shortstack[l]{Params\\ (kept)}} & F-MNIST & CIFAR-10 &
\multirow{2}{*}{\shortstack[l]{Pruning\\ ratio}} & \multirow{2}{*}{\shortstack[l]{Params\\ (kept)}} & F-MNIST & CIFAR-10 \\
& & & & & & & \\
\midrule
0.9500 & 14217 & 91.34 $\pm$ 0.23 & 72.27 $\pm$ 0.57 & 
0.9500 & 14472 & 89.27 $\pm$ 0.11 & 61.49 $\pm$ 1.16 \\
0.9750 & 7433 & 90.02 $\pm$ 0.38 & 66.78 $\pm$ 0.32 & 
0.9750 & 7561 & 87.76 $\pm$ 0.22 & 59.52 $\pm$ 1.02 \\
\bottomrule
\end{tabular}
\end{subtable}
\end{table}

\subsection{Hybrid-MLP: Gradient descent with large batches}

In this section, we train the proposed Hybrid-MLP with \(5\) linear and \(5\) morphological hidden layers of size \(256\) each, using different batch sizes and report the results in Table \ref{table:hybridmlp}. Our findings indicate that the network requires a large batch size to be trainable, suggesting that the inclusion of morphological layers introduces significant noise in the gradient estimation of stochastic optimization methods like Adam. However, for sufficiently large batch sizes, Adam exhibits rapid convergence. This is evident in Figures \ref{fig:convergence-train-mnist}, \ref{fig:convergence-val-mnist}, \ref{fig:convergence-train-fmnist}, and \ref{fig:convergence-val-fmnist}, which compare the training and validation accuracy of a standard MLP, a Maxout network \citep{goodfellow2013maxout}, and a Hybrid-MLP, each with five layers of size \(256\), trained with Adam using a batch size of \(6400\). Under sufficiently large batch sizes, the Hybrid-MLP exhibits stable and rapid optimization dynamics, suggesting that reducing gradient noise may play an important role in training deep morphological architectures. A final interesting observation is that the Hybrid-MLP of setting 3 was the most demanding at training, while the MPM-SVD of setting 2, while easier at training, lost out with respect to generalization to the RMPM with weight dropout of setting 1. 

{
\setlength{\tabcolsep}{5pt}
\begin{table*}
  \caption{Performance of Hybrid-MLP on MNIST and Fashion-MNIST for different batch sizes}
  \label{table:hybridmlp}
  \centering
  \begin{tabular}{lcccccc}
    \toprule
    & \multicolumn{3}{c}{\textbf{MNIST}} & \multicolumn{3}{c}{\textbf{Fashion-MNIST}} \\
    \cmidrule(lr){2-4} \cmidrule(lr){5-7}
    & \textbf{64} & \textbf{640} & \textbf{6400} & \textbf{64} & \textbf{640} & \textbf{6400} \\
    \midrule
    Train Acc. (peak) (\%) & 33.54 $\pm$ 2.89 & 98.52 $\pm$ 0.18 & 99.96 $\pm$ 0.06 & 35.13 $\pm$ 2.79 & 92.68 $\pm$ 0.40 & 98.90 $\pm$ 0.19 \\
    Train Acc. & 18.66 $\pm$ 5.28 & 16.76 $\pm$ 1.31 & 99.81 $\pm$ 0.27 & 19.76 $\pm$ 6.20 & 32.32 $\pm$ 15.92 & 98.36 $\pm$ 0.27 \\ 
    \quad (last epoch) (\%) \\
    Validation Acc. (\%) & 33.37 $\pm$ 2.60 & 96.76 $\pm$ 0.21 & 97.63 $\pm$ 0.15 & 34.61 $\pm$ 2.69 & 88.62 $\pm$ 0.02 & 88.96 $\pm$ 0.19 \\
    Test Acc. (\%) & 33.59 $\pm$ 3.20 & 96.74 $\pm$ 0.23 & 97.42 $\pm$ 0.15 & 34.64 $\pm$ 2.40 & 87.89 $\pm$ 0.24 & 88.15 $\pm$ 0.24 \\
    \bottomrule
  \end{tabular}
\end{table*}
}

\begin{figure}
    \centering
    \begin{subfigure}{0.4\linewidth}
        \centering
        \includegraphics[width=\linewidth]{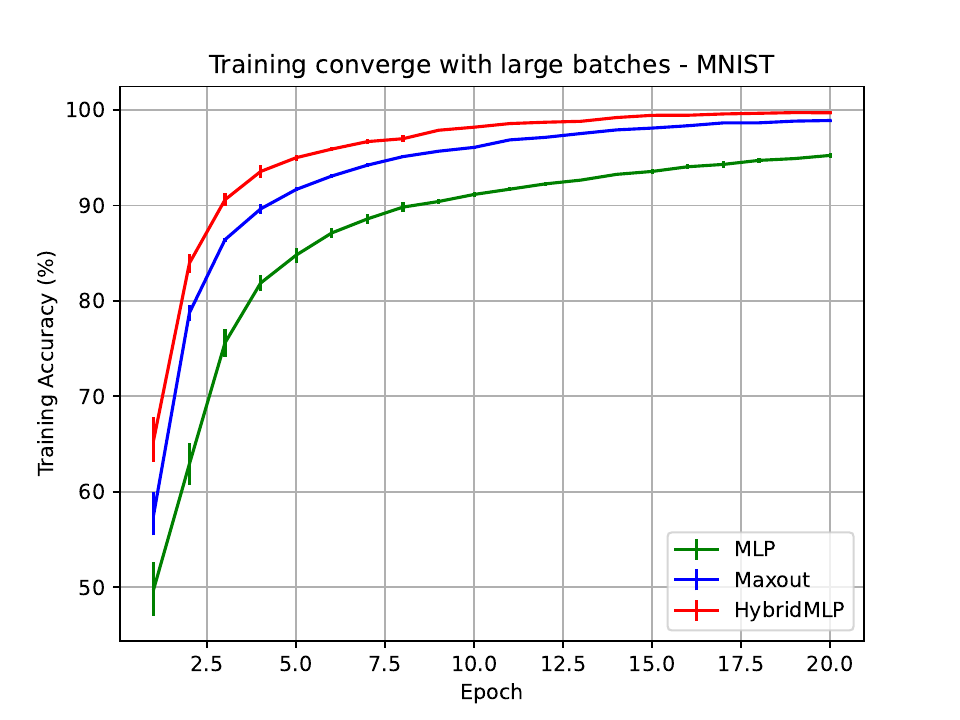}
        \caption{MNIST, Train acc.}
        \label{fig:convergence-train-mnist}
    \end{subfigure}
    \begin{subfigure}{0.4\linewidth}
        \centering
        \includegraphics[width=\linewidth]{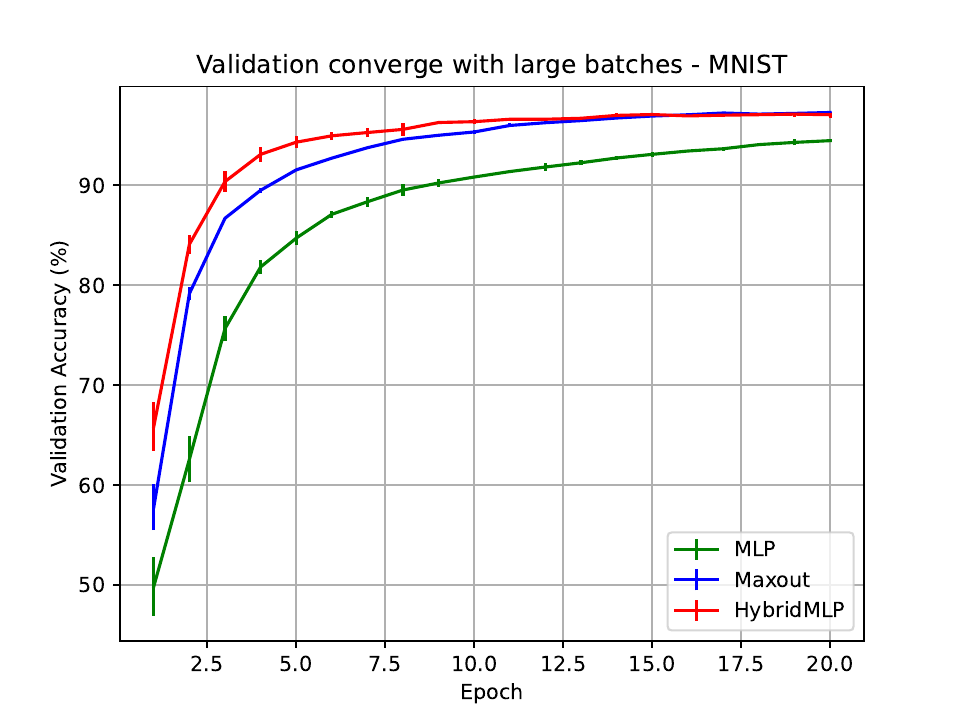}
        \caption{MNIST, Val. acc.}
        \label{fig:convergence-val-mnist}
    \end{subfigure}
    \begin{subfigure}{0.4\linewidth}
        \centering
        \includegraphics[width=\linewidth]{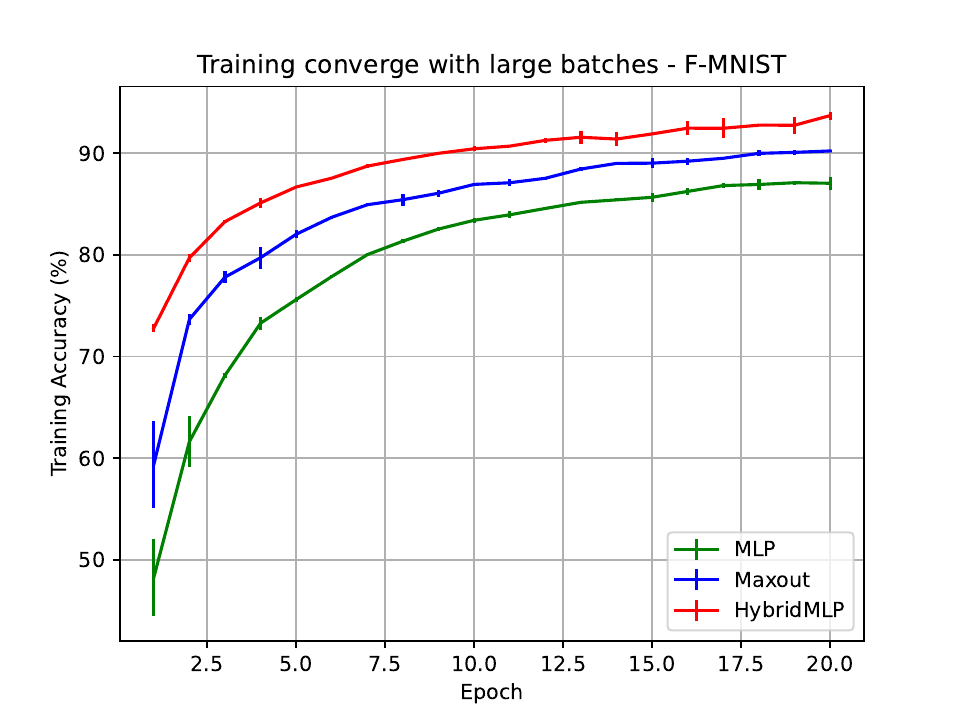}
        \caption{F-MNIST, Train acc.}
        \label{fig:convergence-train-fmnist}
    \end{subfigure}
    \begin{subfigure}{0.4\linewidth}
        \centering
        \includegraphics[width=\linewidth]{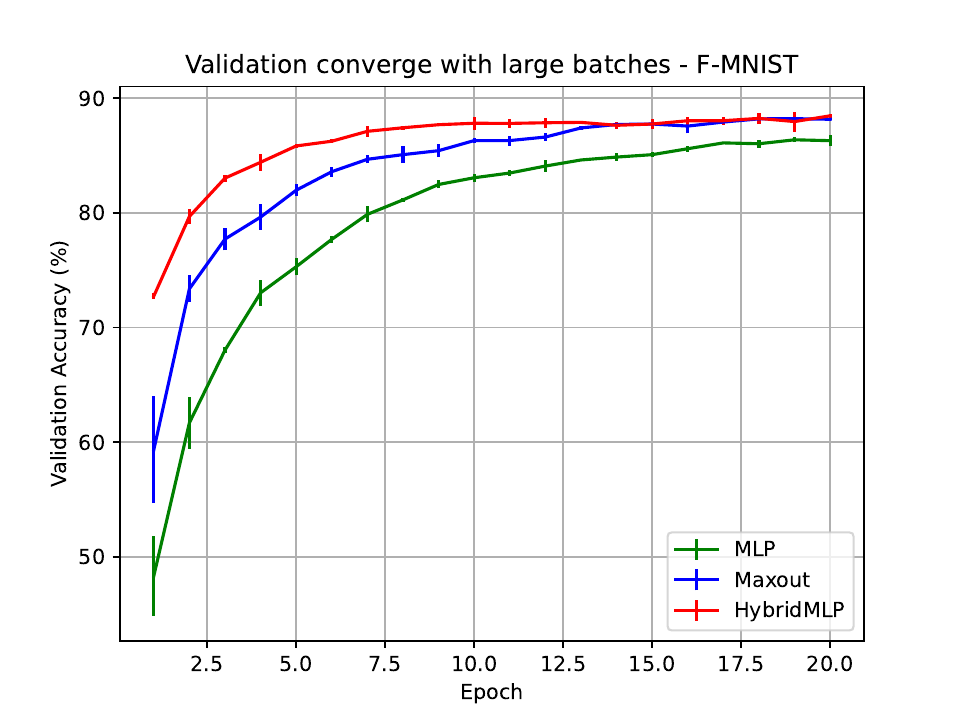}
        \caption{F-MNIST, Val. acc. }
        \label{fig:convergence-val-fmnist}
    \end{subfigure}
    \caption{Convergence rate of different models on MNIST and Fashion-MNIST for \(6400\) batch size}
    \label{fig:convergence}
\end{figure}

\section{Conclusion}

We studied deep morphological neural networks from the perspective of how alternative algebraic structures affect the expressiveness and trainability of deep architectures. Our analysis revealed fundamental algebraic and optimization limitations of existing deep morphological networks, motivating the introduction of constrained linear activations between morphological layers and averaging of max and min components. We showed that these minimal augmentations restore universal approximation, achieve \(1\)-st order universal approximation, and can successfully be trained on standard image classification tasks. Residual connections and weight dropout further improved generalization, and our pruning experiments demonstrated that the resulting architectures remain competitive without requiring prohibitively large active parameter counts. More broadly, our results suggest that deep learning architectures built on non-standard algebraic operations can achieve expressive and trainable representations. We hope this work motivates further study of neural architectures beyond the standard linear setting and improved optimization methods for regimes with sparse, uninformative gradients. 

% ------------------------------------------------------------
% Statements
% ------------------------------------------------------------
\section*{Acknowledgments}
The research project was supported by the Hellenic Foundation for Research and Innovation (H.F.R.I.) under the “2nd Call for H.F.R.I. Research Projects to support Faculty Members \& Researchers” (Project Number:2656, Acronym: TROGEMAL).

The authors thank Panagiotis Papanikolaou and Assistant Professor G. Tzimpragos for identifying an implementation issue in the SNIP pruning experiments, which led to a correction incorporated in this version.

We thank all people who build and maintain the tools that enabled this work, including all open-source software used for the experiments. We also thank all people posting answers on fora such as Stack Exchange, which were valuable for the development of the CUDA module. 

% ------------------------------------------------------------
% Bibliography
% ------------------------------------------------------------
\bibliography{refs}

% For biblatex, replace the line above with:
% \printbibliography

% ------------------------------------------------------------
% Appendices
% ------------------------------------------------------------
\appendix

% \vspace{-10cm}

\section*{Organization of the appendix}

The Appendix is organized as follows:
\begin{itemize}[label={}, left=5pt]
    \item Appendix \ref{appendix:b} establishes results on previous architectures. Specifically, we prove Theorems \ref{theorem:1} and \ref{theorem:3} and show that these results are sufficient to conclude non-universality of i) networks using only max-plus and min-plus MPs with Theorem~\ref{theorem:b5}, and ii) existing DEP networks with Theorem~\ref{theorem:b7}. We prove Theorem \ref{theorem:2}, which highlights the sparse gradient signal in MP-based networks. 
    \item Appendix \ref{appendix:c} proves Theorems \ref{theorem:4} and \ref{theorem:5} and Corollary \ref{theorem:6}, demonstrating that our proposed networks are universal approximators. We also derive Example~\ref{example:1} more thoroughly and prove Proposition~\ref{theorem:8}. 
    \item Appendix \ref{appendix:d} investigates the initialization of morphological neural networks. In particular, apart from the activation of their layers, we showcase a main problem that makes training and initializing morphological networks difficult. We show how initializing our proposed networks is easier than initializing MP-based networks or DEP-based networks with \(\lambda\neq 1/2\). Finally, we explain how the networks in our experiments were initialized, and the reasoning behind it. 
    \item Appendix \ref{appendix:f} presents the relationship between our results and the Representation Theorem 
    % \citep{MaSc87a}
    \citep{maragos1987morphological}. Specifically, it showcases the differences in assumptions made, where these matter in the proof of our theorems, and shows also that our theoretical results are actually complementary to the Representation Theorem. 
    \item Appendix \ref{appendix:e} presents our experimental setup details, compute resources required, and our declaration of LLM usage. In addition, it includes additional experimental results, including experiments on simple regression tasks, and the full training history of networks that we claimed were trainable but for which only peak training accuracy was reported in the main body.
\end{itemize}

% \clearpage

\section{Results on previous work}
\label{appendix:b}

In this appendix we establish our results regarding previous architectures. For our proofs, we are going to need the following lemmas.

\begin{lemma}
\label{theorem:b1}
    Let \(\mathcal{B}\subseteq \mathbb{R}\) be a non-trivial compact interval, and let \(\mathcal{F}\) be the class of Lipschitz continuous, uni-variate functions \(f:\mathcal{B} \to \mathbb{R}\) with \(|f'(x)| \leq 1\) a.e.. The class \(\mathcal{F}\) is not a universal approximator on \(\mathcal{B}\) (i.e. it is not a dense subclass of the  class of continuous functions on \(\mathcal{B}\)).
\end{lemma}

\begin{proof}
    Since \(\mathcal{B}\) has non-empty interior, there exists \(x_0\) and \(0<r<1\) such that \([x_0-r, x_0+r]\subseteq \mathcal{B}\). Let \(g(x)=x/r\), a continuous function on \(\mathcal{B}\). For the sake of contradiction, suppose that \(\mathcal{F}\) is a universal approximator. Then, for every \(\varepsilon>0\) we can find \(f\in \mathcal{F}\) such that \(|f(x)-g(x)|<\varepsilon, \forall x\in \mathcal{B}\). In particular, \(f(x_0)<g(x_0)+\varepsilon=x_0/r+\varepsilon\) and \(f(x_0+r) > g(x_0+r)-\varepsilon=x_0/r+1-\varepsilon\). It follows that \(f(x_0+r)-f(x_0)>1-2\varepsilon\). However, f is Lipschitz continuous, and hence also absolutely continuous. Thus, it is differentiable a.e., and it holds that \(f(x_0+r)-f(x_0)=\int_{x_0}^{x_0+r} f'(x)dx\). But, due to the bound on the derivates of \(f\), we have that \(\int_{x_0}^{x_0+r} f'(x)dx \leq \int_{x_0}^{x_0+r}1dx=r \Rightarrow r > 1-2\varepsilon, \forall \varepsilon>0 \Rightarrow r \geq 1\), a contradiction. 
\end{proof}

The previous lemma can be generalized to functions of multiple variables. 

\begin{lemma}
\label{theorem:b2}
    Let \(\mathcal{B}\subseteq \mathbb{R}^d\) be a compact, convex domain with non-empty interior, and let \(\mathcal{F}\) be the class of Lipschitz continuous functions \(f:\mathcal{B} \to \mathbb{R}\) with \(\|\nabla f(\mathbf{x})\|_1 \leq 1\) a.e.. The class \(\mathcal{F}\) is not a universal approximator on \(\mathcal{B}\).
\end{lemma}

\begin{proof}
    First, we prove that if \(\mathcal{H}\) is a universal approximator on \(\mathcal{B}\), \(\mathbf{x}_0\in \mathrm{int}(\mathcal{B})\), \(\hat{\mathbf{v}}\) is a directional vector, and \(\widetilde{\mathcal{B}}=\{t\in \mathbb{R} : \mathbf{x}_0 + t\hat{\mathbf{v}}\in \mathcal{B}\}\), then \(\widetilde{\mathcal{H}}=\{f(\mathbf{x}_0+t\hat{\mathbf{v}}) | f\in \mathcal{H}\}\) is a universal approximator on \(\widetilde{\mathcal{B}}\). Let \(\tilde{g}:\widetilde{\mathcal{B}} \to \mathbb{R}\) be a continuous function. Since \(\tilde{\mathcal{B}}\) is a compact interval, we extend the function continuously on \(\mathbb{R}\). For every \(\mathbf{x}\in \mathcal{B}\) if \(\mathbf{x}_0+t\mathbf{v}\) is its projection on the line, we define \(g(\mathbf{x})=\tilde{g}(t)\).  In the end, we obtain a continuous \(g\) on \(\mathcal{B}\) for which \(\tilde{g}(t)=g(\mathbf{x}_0+t\hat{\mathbf{v}})\). Since \(\mathcal{H}\) is a universal approximator of \(\mathcal{B}\), we can find a sequence \(f_1, f_2, \ldots\) of functions in \(\mathcal{H}\) such that \(f_n\to g\) uniformly on \(\mathcal{B}\), i.e. \(\sup_{\mathbf{x}\in \mathcal{B}}(f_n(\mathbf{x})-g(\mathbf{x}))\to 0 \Rightarrow \sup_{t\in \widetilde{\mathcal{B}}}(f_n(\mathbf{x}_0+t\hat{\mathbf{v}})-\tilde{g}(t)) \to 0\). Thus, \(f_1(\mathbf{x}_0+t\hat{\mathbf{x}}), f_2(\mathbf{x}_0+t\hat{\mathbf{x}}), \ldots\) approximates the continuous function \(\tilde{g}\), and \(\widetilde{\mathcal{H}}\) is a universal approximator on \(\widetilde{\mathcal{B}}\). 

    Next, we take some \(\mathbf{x}_0\in \mathrm{int}(\mathcal{B})\) and \(\hat{\mathbf{v}}=(1,0,\ldots,0)\). To prove that \(\mathcal{F}\) is not a universal approximator on \(\mathcal{B}\), it suffices to prove that \(\widetilde{\mathcal{F}}=\{f(\mathbf{x}_0+t\hat{\mathbf{x}}) | f\in \mathcal{F}\}\) is not a universal approximator on \(\widetilde{\mathcal{B}}\). Notice that \(\widetilde{\mathcal{F}}\) and \(\widetilde{\mathcal{B}}\) satisfy all requirements of Lemma \ref{theorem:b1}. Indeed, \(\mathcal{B}\) is compact, convex, with non-empty interior, and thus \(\widetilde{\mathcal{B}}\) is a non-trivial compact interval. In addition, let \(\tilde{f}\in \widetilde{\mathcal{F}}\). Then, we have that
    \begin{enumerate}
        \item \(\tilde{f}\) is Lipschitz: \(|\tilde{f}(t_2)-\tilde{f}(t_1)|=|f(\mathbf{x}_0+t_2\hat{\mathbf{v}})-f(\mathbf{x}_0+t_1\hat{\mathbf{v}})| \leq L\|\mathbf{x}_0+t_2\hat{\mathbf{v}}-\mathbf{x}_0+t_1\hat{\mathbf{v}}\|= L|t_2-t_1|\|\hat{\mathbf{v}}\|=L|t_2-t_1|\)
        \item \(|\tilde{f}'(t)|\leq 1\) a.e.: Since \(f\) is Lipschitz, by Rademacher's theorem, it is differentiable a.e. on \(\mathrm{int}(\mathcal{B})\). Since \(\mathcal{B}\) is convex, its boundary has zero measure, and hence its interior \(\mathrm{int}(\mathcal{B})\) has same measure as \(\mathcal{B}\). This means that \(f\) is differentiable a.e. on \(\mathcal{B}\). This in turn means that the directional derivative of \(f\) in the direction \(\hat{\mathbf{v}}\) (i.e. the derivative of \(\tilde{f}\)) equals \(\left<\nabla f, \hat{\mathbf{v}}\right> = (\nabla f)_1\) a.e.. It also holds that \(|(\nabla f)_1|\leq \|\nabla f\|_1 \leq 1\) a.e..
    \end{enumerate}
    Hence, by Lemma \ref{theorem:b1}, \(\widetilde{\mathcal{F}}\) is not a universal approximator on \(\widetilde{\mathcal{B}}\), and \(\mathcal{F}\) is not a universal approximator on \(\mathcal{B}\). 
\end{proof}

Before we proceed with the proof of Theorem \ref{theorem:1}, we will need the following two auxiliary lemmas. 

\begin{lemma}
\label{theorem:b3}
    Let \(\mathcal{B}\subseteq \mathbb{R}^d\) and \(f_1, f_2, \ldots, f_n:\mathcal{B} \to \mathbb{R}\) be Lipschitz continuous functions with Lipschitz constants \(L_1, \ldots\). Then, their pointwise maximum and point-wise minimum are Lipschitz continuous functions.
\end{lemma}

\begin{proof}
    We can prove the statement for two functions \(f_1\) and \(f_2\), and for multiple functions it will follow from induction on \(n\). 

    Notice that \(\max(f_1, f_2)=\frac{f_1+f_2+|f_1-f_2|}{2}\), and \(\min(f_1, f_2)=\frac{f_1+f_2-|f_1-f_2|}{2}\). Since the addition and scaling of Lipschitz functions is Lipschitz, it suffices to show that \(|f_1-f_2|\) is Lipschitz. By triangle inequality, for every \(\mathbf{x}_1, \mathbf{x}_2\) we have that
    \begin{align*}
    &\big| |f_1(\mathbf{x}_1)-f_2(\mathbf{x}_1)|-|f_1(\mathbf{x}_2)-f_2(\mathbf{x}_2)| \big| \\
    \leq & |f_1(\mathbf{x}_1)-f_2(\mathbf{x}_1)-f_1(\mathbf{x}_2)+f_2(\mathbf{x}_2)|\\ 
    \leq & |f_1(\mathbf{x}_1)-f_1(\mathbf{x}_2)|+|f_2(\mathbf{x}_1)-f_2(\mathbf{x}_2)| \\
    \leq & (L_1+L_2)\|\mathbf{x}_1-\mathbf{x}_2\|
    \end{align*}
    Thus, \(\max(f_1, f_2), \min(f_1, f_2)\) are Lipschitz. By induction, \(\max_{i\in [n]}f_i, \min_{i\in [n]}f_i\) are Lipschitz.
\end{proof}

\begin{lemma}
\label{theorem:b4}
    Let open \(\mathcal{B}\subseteq \mathbb{R}^d\), \(f_1, f_2, \ldots, f_n:\mathcal{B}\to \mathbb{R}\), and \(g=\max_{i\in [n]}(f_i),\ h=\min_{i\in [n]}(f_i)\). Let \(\mathbf{x}_0\in \mathcal{B}\) and \(I=\{i\in [n]: f_i(\mathbf{x}_0)=g(\mathbf{x}_0)\}, J=\{i\in [n]:f_i(\mathbf{x}_0)=h(\mathbf{x}_0)\}\). Suppose that \(f_i\):differentiable on \(\mathbf{x}_0, \forall i\in [n]\). Then, we have that 
    \[
    g:\text{diff. on }\mathbf{x}_0\Rightarrow \nabla g(\mathbf{x}_0)=\nabla f_i(\mathbf{x}_0), \forall i\in I.
    \]
    \[
    h:\text{diff. on }\mathbf{x}_0\Rightarrow \nabla h(\mathbf{x}_0)=\nabla f_j(\mathbf{x}_0), \forall j\in J.
    \]
\end{lemma}

\begin{proof}
    We will prove only the part of the statement involving the maximum, with the proof of the statement for the minimum being similar. Suppose that \(g\) is differentiable on \(\mathbf{x}_0\) and that there exists \(i\in I\) such that \(\nabla g(\mathbf{x}_0) \neq \nabla f_i(\mathbf{x}_0)\). Then, we have that
    \begin{align*}
    & \nabla (f_i-g)(\mathbf{x}_0)\neq \mathbf{0}\\
    \Rightarrow \quad & \exists \hat{\mathbf{v}}:\nabla (f_i-g)(\mathbf{x}_0)\cdot \hat{\mathbf{v}}>0 \\
    \Rightarrow \quad & \exists \hat{\mathbf{v}}: \frac{\partial (f_i-g)}{\partial \hat{\mathbf{v}}}(\mathbf{x}_0) > 0
    \end{align*}
    But it also holds that \((f_i-g)(\mathbf{x}_0)=0\). Hence, for small enough \(t>0\), we have that
    \[
    (f_i-g)(\mathbf{x}_0+t\hat{\mathbf{v}}) > 0
    \]
    \[
    \Rightarrow f_i(\mathbf{x}_0+t\hat{\mathbf{v}})> g(\mathbf{x}_0+t\hat{\mathbf{v}})=\max_{j\in [n]}(f_j(\mathbf{x}_0 +t\hat{\mathbf{v}})),
    \]
    which is a contradiction. 
\end{proof}

\begin{rmk}
    In Lemma \ref{theorem:b4} it is implicitly proven that if \(g\) is differentiable at \(\mathbf{x}_0\), then the gradients of \(f_i\) at \(\mathbf{x}_0\) for \(i\in I\) will be equal. 
\end{rmk}

Next, we restate and prove Theorem \ref{theorem:1}. 

% \begin{theorem1}
%     For any network that only uses max-plus and min-plus MPs with input \(\mathbf{x}\in \mathbb{R}^d\) and a single output \(y(\mathbf{x})\), we have that \(y(\mathbf{x})\) is Lipschitz continuous on \(\mathbb{R}^d\) and a.e. it holds that either \(\nabla y(\mathbf{x})=0\) or \( \nabla y(\mathbf{x}) = \mathbf{e}_i = [0,\ldots,1,\ldots, 0]^\top \) for some \(i=i(\mathbf{x})\).
% \end{theorem1}

\first*

\begin{proof}
    We first prove the theorem for the case of no negative (anti-erosion or anti-dilation) terms and then extend it to cover that case as well. 
    
    We first prove that \(y\) is Lipschitz continuous. We use induction on the layers of the network to prove that the output of each unit is Lipschitz continuous. For the base case \(n=0\), notice that each input \(i\) can be thought of as a projection \(\mathbf{x}\to x_i\), which is Lipschitz continuous. Suppose that the outputs \(x^{(n)}_j\) of all units of the \(n\)-th layer are Lipschitz continuous. Without loss of generality, suppose the \(i\)-th unit of the (\(n+1\))-th layer is a max-plus MP. We have that \(x^{(n)}_j + w^{(n+1)}_{ij}\) is Lipschitz for all \(j\in [N^{(n)}]\), \(w^{(n+1)}_{i0}\) is Lipschitz, and hence by Lemma \ref{theorem:b3} \(x^{(n+1)}_i=w^{(n+1)}_{i0}\vee \max_{j\in [n]}(x^{(n)}_j+w^{(n+1)}_{ij})\) is Lipschitz continuous. Thus, by induction, the outputs of all units of the networks, including output \(y\) of the last unit, are Lipschitz continuous. 

    Next, we prove the stated form of the gradient. To simplify notation, we denote \(\mathbf{e}_0=\mathbf{0}=[0,\ldots,0]^\top\). 
    
    To prove our theorem, we use induction on the number of layers. For the base case, i.e. the input, we have that \(\nabla x_i=\mathbf{e}_i\) everywhere. Suppose that the claim holds true for the \(n\)-th layer. Take any MP unit of the \((n+1)\)-th layer, say the \(i\)-th unit, and suppose without loss of generality that it is a max-plus unit. We have that
    \[
    x^{(n+1)}_i(\mathbf{x})=w^{(n+1)}_{i0}\vee \max_{j}(x^{(n)}_j(\mathbf{x})+w^{(n+1)}_{ij}). 
    \]
    We set \(f^{(n+1)}_{i0}=w^{(n+1)}_{i0}, f^{(n+1)}_{ij}=x^{(n)}_j+w^{(n+1)}_{ij}\). Notice that by Rademacher's theorem, we have that \(x^{(n+1)}_i(\mathbf{x})\) is differentiable a.e. on \(\mathbb{R}^d\). In addition, by the inductive hypothesis, we have that \(x^{(n)}_j(\mathbf{x})\) is differentiable a.e. with (\(\nabla x^{(n)}_j(\mathbf{x}) = \mathbf{e}_{k_j(\mathbf{x})}\) for some \(k_j(\mathbf{x})\)) a.e., and hence \(f^{(n+1)}_{ij}\) is differentiable a.e. with (\(\nabla f^{(n+1)}_{ij} = \nabla x^{(n)}_j=\mathbf{e}_{k_j(\mathbf{x})}\) for some \(k_j(\mathbf{x})\)) a.e.. Moreover, \(w^{(n+1)}_{i0}\) is differentiable everywhere with \(\nabla f^{(n+1)}_{i0}=\nabla w^{(n+1)}_{i0}=\mathbf{e}_0\). This means that Lemma \ref{theorem:b4} holds a.e., and we have that a.e. \(\nabla x^{(n+1)}_i(\mathbf{x})=\nabla f^{(n+1)}_{ij}(\mathbf{x}), \forall j \in J\), where \(J=\{j:x^{(n+1)}_{i}(\mathbf{x})=f^{(n+1)}_{ij}(\mathbf{x})\}\). This means that a.e.: 1) \(k_j(\mathbf{x})=k(\mathbf{x}), \forall j\in J\), and 2) \(\nabla x^{(n+1)}_i(\mathbf{x})=\mathbf{e}_{k(\mathbf{x})}\). This concludes the induction. 

    Finally, we extend the proof to cover the case where there are also negative terms. If a function \(f\) is Lipschitz, then so is \(-f\), with \(\nabla(-f) = -\nabla f\) wherever it exists. Hence, both inductions go though once we account for any possible negation of intermediate outputs. 
\end{proof}

Notice that the above proof also gives us a way to find the derivative (i.e. the \(i=i(\mathbf{x})\) of the statement): It suffices to work backwards, following the paths where the maximum or the minimun is attained. If we end up with two or more "leafs" as inputs/biases, then the function is not differentiable. Otherwise, if we end up at a unique leaf, which can either be an input or a bias, then the index of the input (0 for the bias) is the \(i=i(\mathbf{x})\) we are looking for. Refer to Figure~\ref{fig:network_input_diff} for a differentiable example, and Figure~\ref{fig:network_input_nondiff} for a non-differentiable example. 

Combining Theorem \ref{theorem:1} and Lemma \ref{theorem:b2}, we obtain the following result.

\begin{theorem}
\label{theorem:b5}
    Networks that only use max-plus and min-plus MPs are not universal approximators.
\end{theorem}

\begin{proof}
    We prove that they are not universal approximators on any compact, convex domain \(\mathcal{B}\subseteq \mathbb{R}^d\) with non-empty interior, which is stronger that not being universal approximators on \(\mathbb{R}^d\). Notice that by Theorem \ref{theorem:1} we have that a.e. either \(\nabla y(\mathbf{x})=0\) or \(\nabla y(\mathbf{x})=\mathbf{e}_{i(\mathbf{x})} \Rightarrow \|\nabla y(\mathbf{x})\|_1 \leq 1\). Hence, \(\|\nabla y(\mathbf{x})\|_1 \leq 1\) a.e. on \(\mathbb{R}^n\). If we restrict \(y\) on \(\mathcal{B}\), we have a compact, convex, domain \(\mathcal{B}\) with non-empty interior and a Lipschitz continuous \(y(\mathbf{x})\) with \(\|\nabla y(\mathbf{x})\|_1 \leq 1\) a.e. on \(\mathcal{B}\). From Lemma \ref{theorem:b2} it follows that networks that only use max-plus and min-plus MPs are not universal approximators on \(\mathcal{B}\). 
\end{proof}

Before we proceed with the proof of Theorem \ref{theorem:2}, we prove the following lemma. 

\begin{lemma}
\label{theorem:b6}
Let open \(\mathcal{B}\subseteq \mathbb{R}^d\), continuous \(f_1, \ldots, f_n:\mathcal{B} \to \mathbb{R}\), and \(g=\max_{i\in [n]}(f_i), h=\min_{i\in [n]}(f_i)\). If for some \(\mathbf{x}_0\in \mathcal{B}\) we have that the maximum is attained only at some \(j\in [n]\) (i.e. \(f_j(\mathbf{x}_0)>f_i(\mathbf{x}_0),\forall i\neq j\)) and \(f_j:\) differentiable on \(\mathbf{x}_0\), then \(g:\) differentinable on \(\mathbf{x}_0\) with \(\nabla g(\mathbf{x}_0)=\nabla f_j(\mathbf{x}_0)\). Similarly for minimum. 
\end{lemma}

\begin{proof}
    We only prove the statement for the maximum, with the minimum being similar. Since \(f_i\): continuous on \(\mathbf{x}_0\) and \(f_j(\mathbf{x}_0) > f_i(\mathbf{x}_0), \forall i\neq j\), we have that in a small enough neighborhood \(N(\mathbf{x}_0, \varepsilon)\) it holds that \(f_j(\mathbf{x})>f_i(\mathbf{x}), \forall i\neq j, \mathbf{x}\in N(\mathbf{x}_0, \varepsilon)\). Then \(g(\mathbf{x})=f_j(\mathbf{x}), \forall \mathbf{x}\in N(\mathbf{x}_0, \varepsilon) \Rightarrow g:\) differentiable on \(\mathbf{x}_0\) with \(\nabla g(\mathbf{x}_0)=\nabla f_j(\mathbf{x}_0)\). 
\end{proof}

In a similar fashion to Theorem~\ref{theorem:1}, we can prove Theorem~\ref{theorem:2}.

\second*
% \begin{theorem2}
% Consider a network that only uses max-plus and min-plus MPs with output \(\mathbf{y}\in \mathbb{R}^m\). For any given input \(\mathbf{x}\), if \(\mathbf{y}\) is differentiable with respect to the network parameters, then in each layer \(n\), there exists at most \(m\) parameters \(w^{(n)}_{ij}\) for which the derivative of \(\mathbf{y}\) is nonzero. 
% \end{theorem2}

\begin{proof}
    We work backwards, following the paths where the maximum and the minimum is attained. If the maximum or the minimum is attained only for one argument, we keep track of the argument's "slack" and continue the path from this argument. We continue this process until either 1) we reach a "leaf" (i.e. an input or a bias), or 2) we find a maximum or minimum that is attained for multiple arguments and we get a "split". 
    
    In the second case, we have that the output is not differentiable with respect to the weights: If the split happens for the path of the \(k\)-th output, WLOG at a max-plus MP, and \(w^{(n)}_{ij}\) is a weight such that the maximum is attained at \(j\), then by letting a small \(dw>0\), smaller than the minimum of the recorded slacks, and \((w^{(n)}_{ij})'=w^{(n)}_{ij}+dw\), we have that all the units along the path are incremented by \(dw\), and \(y_k'=y_k+dw\Rightarrow dy_k=dw\). If, on the other hand, we let \(dw<0\), then the MP unit where the split happened remains unchanged and \(y_k'=y_k\Rightarrow dy_k=0\). Hence, \(\mathbf{y}\) is not differentiable with respect to \(w^{(n)}_{ij}\). 

    In the first case, each output \(y_k\) has a path from \(y_k\) to a "leaf". We have that the output is differentiable with respect to the weights, and the statement holds. We will prove this using induction on the nodes of the path. We will prove the following stronger statement: If \(i_k(n)\geq 0\) is the unique node of layer \(n\) that belongs to the path of \(y_k\), and \(\text{path}_k(n)\) is the set of weights that belong to the path of \(y_k\) and are before the \(n\)-th layer, then for every \(n\) it holds that \(x^{(n)}_{i_k(n)}\) is differentiable with respect to the weights and \(\nabla_{\mathbf{w}} x^{(n)}_{i_k(n)}=\sum_{w\in \text{path}_k(n)} \mathbf{e}_w\) (here we denote \(x^{(n)}_0=0\)). Notice that since we have no "splits", paths can only merge, and hence the statement is indeed stronger. 

    Take the path of \(y_k\). The statement obviously holds true for the first node of the path, i.e. the leaf. Suppose it holds true for the \(n\)-th node. Then, we have that \(f_{i_k(n+1),i_k(n)}=x^{(n)}_{i_k(n)}+w^{(n+1)}_{i_k(n+1),i_k(n)}\) is differentiable in terms of the weights with \(\nabla_\mathbf{w} f_{i_k(n+1),i_k(n)}=\sum_{w\in \text{path}_k(n)} \mathbf{e}_w + \mathbf{e}_{w^{(n+1)}_{i_k(n+1),i_k(n)}}=\sum_{w\in \text{path}_k(n+1)} \mathbf{e}_w\). Also since the maximum or minimum is attained only for \(i_k(n)\), by Lemma \ref{theorem:b6} we have that \(x^{(n+1)}_{i_k(n+1)}\) is differentiable with respect to the weights, and \(\nabla_\mathbf{w} x^{(n+1)}_{i_k(n+1)} = \nabla_\mathbf{w} f_{i_k(n+1),i_k(n)} = \sum_{w\in \text{path}_k(n+1)}\mathbf{e}_w\), which concludes the induction. 
\end{proof}

The above proof also gives us a way to calculate the derivative with respect to any weight. Refer to Figure~\ref{fig:network_weights_diff} for a differentiable example, and to Figure~\ref{fig:network_weights_nondiff} for a non-differentiable example. 

Continuing, we restate and prove Theorem \ref{theorem:3}. We consider two cases: the non-activated DEP-based network (i.e. \(f^{(n)}=1,\forall n\)), and the activated DEP-based network with a common activation function (i.e. sigmoid as in \citep{araujo2017morphological}, tanh, ReLU, ELU, or Leaky ReLU). For both cases, we prove a stronger variant of Theorem \ref{theorem:3}. 

\setcounter{repeatedtheorem}{4}
\begin{repeatedtheorem}[non-activated]
For any non-activated DEP-based network with input \(\mathbf{x}\in \mathbb{R}^d\) and a single output \(y(\mathbf{x})\), we have that \(y(\mathbf{x})\) is Lipschitz continuous on \(\mathbb{R}^d\) and a.e. it holds that \( \nabla y(\mathbf{x}) \succeq 0, \|\nabla y(\mathbf{x})\|_{1} = 1 \).
\end{repeatedtheorem}

\begin{proof}
    First, we prove that \(y\) is Lipschitz continuous. We again use induction on the layers of the network to prove that the output of each unit is Lipschitz continuous. For the base case \(n=0\), notice that each input \(i\) can be thought of as a projection \(\mathbf{x}\to x_i\), which is Lipschitz continuous. Suppose that the outputs \(x^{(n)}_j\) of all units of the \(n\)-th layer are Lipschitz continuous. We have that \(x^{(n)}_j + w^{(n+1)}_{ij}\) is Lipschitz for all \(j\in [N^{(n)}]\), and hence by Lemma \ref{theorem:b3} \(f^{(n+1)}_i=\max_{j\in [n]}(x^{(n)}_j+w^{(n+1)}_{ij})\) is Lipschitz continuous for all \(i\in [N^{(n+1)}]\). Similarly, \(g^{(n+1)}_i=\min_{j\in [n]}(x^{(n)}_j+m^{(n+1)}_{ij})\) is Lipschitz continuous. Since scaling and addtion preserve Lipschitz continuity, we conclude that \(x^{(n+1)}_i=\lambda^{(n+1)}_i f^{(n+1)}_i + (1-\lambda^{(n+1)}_i) g^{(n+1)}_i\) is also Lipschitz continuous. Thus, by induction, the outputs of all units of the networks, including output \(y\) of the last unit, are Lipschitz continuous. 

    Next, we prove the remaining part of the theorem. To prove our theorem, we use induction on the number of layer. For the base case, i.e. the input, we have that \(\nabla x_i=\mathbf{e}_i\Rightarrow \nabla x_i \succeq 0, \|\nabla x_i\|_1=1\) everywhere. Take any DEP unit of the \((n+1)\)-th layer, say the \(i\)-th unit. We have that
    \[
    f^{(n+1)}_i(\mathbf{x})=\max_{j}(x^{(n)}_j(\mathbf{x})+w^{(n+1)}_{ij}),
    \]
    \[
    g^{(n+1)}_i(\mathbf{x})=\min_{j}(x^{(n)}_j(\mathbf{x})+m^{(n+1)}_{ij}),
    \]
    \[
    x^{(n+1)}_i(\mathbf{x})=\lambda^{(n+1)}_i f^{(n+1)}_i(\mathbf{x}) + (1-\lambda^{(n+1)}_i)g^{(n+1)}_i(\mathbf{x}).
    \]
    Notice that by Rademacher's theorem, we have that \(f^{(n+1)}_i\) and \(g^{(n+1)}_i\) are differentiable a.e. on \(\mathbb{R}^d\) (since they are Lipschitz, see previous paragraph). In addition, by the inductive hypothesis, we have that \(x^{(n)}_j\) is differentiable a.e., and hence \(x^{(n)}_j+w^{(n+1)}_{ij}\) and \(x^{(n)}_j+m^{(n+1)}_{ij}\) are differentiable a.e.. This means that Lemma \ref{theorem:b4} holds a.e., and we have that a.e. \(\nabla f^{(n+1)}_i(\mathbf{x})=\nabla (x^{(n)}_j(\mathbf{x})+w^{(n+1)}_{ij})= \nabla x^{(n)}_j(\mathbf{x}), \forall j\in J_1\), \(\nabla g^{(n+1)}_i(\mathbf{x})=\nabla (x^{(n)}_j(\mathbf{x})+w^{(n+1)}_{ij})= \nabla x^{(n)}_j(\mathbf{x}), \forall j\in J_2\), where \(J_1=\{j:f^{(n+1)}_i(\mathbf{x})=x^{(n)}_j(\mathbf{x})+w^{(n+1)}_{ij}\}, J_2=\{j:g^{(n+1)}_i(\mathbf{x})=x^{(n)}_j(\mathbf{x})+w^{(n+1)}_{ij}\}\). This means that a.e. \(x^{(n+1)}_i\) is differentiable, with \(\nabla x^{(n+1)}_i=\lambda^{(n+1)}_i\nabla x^{(n)}_{j_1} + (1-\lambda^{(n+1)}_i)\nabla x^{(n)}_{j_2}\), for some \(j_1\in J_1, j_2\in J_2\). Since \(\lambda^{(n+1)}_i\in [0, 1]\) and a.e. \(\nabla x^{(n)}_{j_1}\succeq 0, \nabla x^{(n)}_{j_2}\succeq 0\), we have also that a.e. \(\nabla x^{(n+1)}_i\succeq 0\). Therefore, we also have a.e. that \(\|\nabla x^{(n+1)}_i\|_1=\sum_{k}(\nabla x^{(n+1)}_i)_k=\sum_{k}(\lambda^{(n+1)}_i\nabla x^{(n)}_{j_1}+(1-\lambda^{(n+1)}_i)\nabla x^{(n)}_{j_2})_k=\lambda^{(n+1)}_i\sum_{k}(\nabla x^{(n)}_{j_1})_k+(1-\lambda^{(n+1)}_i)\sum_{k}(\nabla x^{(n)}_{j_2})_k=\lambda^{(n+1)}_i + 1-\lambda^{(n+1)}_i=1\). This concludes the induction. 
\end{proof}

At this point, we should note that using standard non-linear activations, as is done by some existing work such as \citep{araujo2017morphological}, which uses a sigmoid function, does not solve the problem (and if anything, it makes it worse). 

\setcounter{repeatedtheorem}{4}
\begin{repeatedtheorem}[activated]
    For activated DEP-based networks with \(L\) layers in total, of which \(\widetilde{L}\) are activated by a common activation function \(f\) with \(0 \leq f' \leq s \leq 1\) a.e., input \(\mathbf{x}\in \mathbb{R}^d\) and a single output \(y(\mathbf{x})\), we have that \(y(\mathbf{x})\) is Lipschitz continuous on \(\mathbb{R}^d\) and a.e. it holds that \( \nabla y(\mathbf{x}) \succeq 0, \|\nabla y(\mathbf{x})\|_{1} \leq s^{\widetilde{L}} \leq 1\). 
\end{repeatedtheorem}

\begin{proof}
    We denote with \(y^{(n)}_i\) the non-activated outputs, and \(x^{(n)}_i\) the outputs after activation (where as activation we can either have a common activation function or an identity). Again, we denote with \(f^{(n)}_i\) the output of the maximum operations, and by \(g^{(n)}_i\) the output of the minimum operation. 

    Because the composition of Lipschitz continuous functions is Lipschitz continuous, and the common activation functions are Lipschitz, with a proof similar to that of the previous theorem it follows immediately that both \(y^{(n)}_i\) and \(x^{(n)}_i\) are Lipschitz continuous for all \(n, i\). 

    We will prove that after \(l\leq \widetilde{L}\) activated layers, we have that a.e. it holds that \(\nabla x^{(n)}_i \succeq 0, \|\nabla x^{(n)}_i\|_1 \leq s^l\). We will prove this with induction. The base case is similar to that of the previous theorem. Suppose that the statement holds for the \(n\)-th layer, i.e. if \(l\) activated layers precede the \(n\)-th output, then a.e. it holds that \(\nabla x^{(n)}_i \succeq 0, \|\nabla x^{(n)}_i\|_1 \leq s^l\). With an argument similar to that of the previous theorem, we have that a.e. \(\nabla y^{(n+1)}_i = \lambda^{(n+1)}_i \nabla x^{(n)}_{j_1} + (1-\lambda^{(n+1)}_i) \nabla x^{(n)}_{j_2}\), for some \(j_1\in J_1, j_2\in J_2\), where \(J_1, J_2\) are defined as in the above proof. Since \(\lambda^{(n+1)}_i\in [0,1]\) and a.e. \(\nabla x^{(n)}_{j_1} \succeq 0, \nabla x^{(n)}_{j_2} \succeq 0\), we also have that a.e. \(\nabla y^{(n+1)}_i\succeq 0\). Also, a.e. \(\|\nabla y^{(n+1)}_i\|_1 \leq \lambda^{(n+1)}_i \|\nabla x^{(n)}_{j_1}\|_1 + (1-\lambda^{(n+1)}_i)\|\nabla x^{(n)}_{j_2}\|_1 \leq \lambda^{(n+1)}_i s^l + (1-\lambda^{(n+1)}_i)s^l=s^l\). If the \((n+1)\)-th layer is not activated, then we have that \(x^{(n+1)}_i=y^{(n+1)}_i\), and then, \(l\) activated layers precede the \((n+1)\)-th output, and we have that a.e. \(\nabla x^{(n+1)}_i = \nabla y^{(n+1)}_i\succeq 0\), \(\|\nabla x^{(n+1)}_i\|_1 = \|\nabla y^{(n+1)}_i\|_1 \leq s^l\). If the \((n+1)\)-th layer is activated, then by the chain rule, because the derivative of all common activation functions exists a.e. and is bounded by \(0\) and \(s\), and because \(\nabla y^{(n)}_i\succeq 0\), we have that a.e. \(\nabla x^{(n+1)}_i \succeq 0\), \(\|\nabla x^{(n+1)}_i\|_1 \leq s \|\nabla y^{(n+1)}_i\|_1 \leq s^{l+1} \). This concludes the induction. 
\end{proof}

Combining Theorem \ref{theorem:3} and Lemma \ref{theorem:b2}, we obtain the following result. 

\begin{theorem}
\label{theorem:b7}
    Existing DEP-based networks are not universal approximators. 
\end{theorem}

\begin{proof}
    We prove that they are not universal approximators on any compact, convex domain \(\mathcal{B}\subseteq \mathbb{R}^d\) with non-empty interior, which is stronger that not being universal approximators on \(\mathbb{R}^d\). Notice that by Theorem \ref{theorem:3} we have that a.e. \(\nabla y(\mathbf{x})\succeq 0, \|\nabla y(\mathbf{x})\|_1 \leq 1\). Hence, \(\|\nabla y(\mathbf{x})\|_1 \leq 1\) a.e. on \(\mathbb{R}^n\). If we restrict \(y\) on \(\mathcal{B}\), we have a compact, convex, domain \(\mathcal{B}\) with non-empty interior and a Lipschitz continuous \(y(\mathbf{x})\) with \(\|\nabla y(\mathbf{x})\|_1 \leq 1\) a.e. on \(\mathcal{B}\). From Lemma \ref{theorem:b2} it follows that DEP-based networks are not universal approximators on \(\mathcal{B}\). 
\end{proof}

At this point, it is important to note that Theorem \ref{theorem:b7} can be easily generalized to hold for any activation function with a bounded derivative, which also includes more exotic activation functions. 

\section{Proposed networks are universal approximators}
\label{appendix:c}

In this appendix we provide the proofs of Theorems \ref{theorem:4} and \ref{theorem:5} regarding the universality of our proposed networks. 

\fourth*
% \begin{theorem4}
% If the domain of the input is compact, the Max-Plus-Min (MPM) network is a universal approximator. 
% \end{theorem4}

First, we give an overview of the ideas of the proof. The main idea is that, if \(C\) is a large enough constant, then for some MPM unit, by letting some input weights hover around \(+C\), some other input weights hover around \(-C\) and the rest to hover around \(0\), we effectively untangle the MPM unit: The third type of weights are rendered idle, the first type of weights contribute only to the maximum, the second type of weights contribute only to the minimum, and after summing the result the \(+C\) and \(-C\) cancel out. This way, we can control which inputs contribute to the maximum, which to the nimimum, and which to neither. This way, we can do the following: 1) We can build functions of the form \(\mathbf{a}^\top \mathbf{x} + b\), 2) We can split the network into parallel networks each of which build their own function, 3) we can feed some of them into a maximum, and the rest into a minimum, to build a function of the form \(\max_{i\in [n]}(\mathbf{a}^\top_i \mathbf{x} + b_i)+\min_{i\in [n+m]-[n]}(\mathbf{a}^\top_i \mathbf{x} + b_i)\), which can be shown to be a universal approximator based on the proof of universality of the Maxout networks \citep{goodfellow2013maxout}. 

\begin{proof}
    A byproduct of the proof of universality of Maxout networks, is that the class of functions of the form \(\max_{k\in [K]}(\mathbf{a}^\top_k\mathbf{x} + b_k) - \max_{m\in [M]}(\mathbf{c}^\top_m\mathbf{x}+d_m)=\max_{k\in [K]}(\mathbf{a}^\top_k\mathbf{x} + b_k) + \min_{m\in [M]}((-\mathbf{c})^\top_m\mathbf{x}-d_m)=\max_{k\in [K]}(\mathbf{a}^\top_k\mathbf{x} + b_k) + \min_{m\in [M]}(\tilde{\mathbf{c}}^\top_m\mathbf{x}+\tilde{d}_m)\) is a universal approximator on \(\mathbb{R}^d\). Thus, it suffices to prove that over any compact domain, we can build any function of the above form using an MPM. We restrict ourselves further by assuming that \(\mathbf{w}^{(n)}_{0}=\mathbf{m}^{(n)}_0\), for all layers \(n\).

    First, we show how we can build linear functions of the form \(\mathbf{a}^\top \mathbf{x}+b\). Let a bounded domain \(\mathcal{B}\) and let some \(R>0\) such that \(\mathcal{B} \subseteq B_1(\mathbf{0}, R)\), where \(B_1(\mathbf{0}, R)\) denotes the \(\ell_1\) ball centered at the origin with \(\ell_1\) radius \(R\). Let \(C>0\) be a large constant to be determined. Write \(\mathbf{a}^\top \mathbf{x} + b=a_1x_1+\ldots+a_dx_d+b\). First, we build each \(a_ix_i\) term. Take the first MPM layer to have \(d\) outputs and \(d\) inputs, with weights
    \[
    w^{(1)}_{ii}=+C,
    \]
    \[
    w^{(1)}_{ij}=0, \quad i\neq j, j\neq 0,
    \]
    \[
    w^{(1)}_{i0} = -C.
    \]
    Then, for the \(i\)-th output we have that
    \begin{gather*}
    x_i+w^{(1)}_{ii} = x_i + C > x_j + 0 = x_j+w^{(1)}_{ij}, \quad j\neq i,\\
    x_i+w^{(1)}_{ii} = x_i + C > -C = w^{(1)}_{i0}, \\
    x_j+w^{(1)}_{ij} = x_j > -C = w^{(n)}_{i0} \\
    \Rightarrow w^{(1)}_{i0}\vee \max_{j\in [d]} (x_j+w^{(1)}_{ij})=x_i+C, \\
    w^{(1)}_{i0}\wedge \min_{j\in [d]} (x_j+w^{(1)}_{ij})=-C \\
    \Rightarrow y^{(1)}_i = (w^{(1)}_{i0}\vee \max_{j\in [d]} (x_j+w^{(1)}_{ij})) + (w^{(1)}_{i0}\wedge \min_{j\in [d]} (x_j+w^{(1)}_{ij})) = x_i+C-C=x_i
    \end{gather*}
    For the inequalities to hold, it suffices that \(C>\|\mathbf{x}\|_1\). Take the activation of the first layer to be
    \[
    \alpha^{(1)}_i=a_i
    \]
    Then, for the \(i\)-th activated output we have that
    \begin{equation}
    \label{eq:uni_1}
    x^{(1)}_i=a_iy^{(1)}_i=a_ix_i
    \end{equation}
    Next, we will start summing up the terms. First, we sum up \(a_{d-1}x_{d-1}\) and \(a_dx_d\). Take a MPM layer with \(d-1\) outputs and \(d\) inputs. The weights of the second layer are as follows:
    \begin{gather*}
    w^{(2)}_{(d-1)(d-1)} = +C, \quad  w^{(2)}_{(d-1)d}=-C,\\
    w^{(2)}_{(d-1)j}=0, \quad \forall j<d-1, \\
    w^{(2)}_{ii}=+C, \quad i<d-1, \\
    w^{(2)}_{ij}=0, \quad i\neq j, i<d-1, j\neq 0, \\
    w^{(2)}_{i0} = -C, \quad i<d-1.
    \end{gather*}
    Similarly with before, for \(i<d-1\) we have that
    \begin{gather*}
    x^{(1)}_i+w^{(2)}_{ii} = x^{(1)}_i + C > x^{(1)}_j + 0 = x^{(1)}_j+w^{(2)}_{ij}, \  j\neq i, \\
    x^{(1)}_i+w^{(2)}_{ii} = x^{(1)}_i + C > -C = w^{(2)}_{i0}, \\
    x^{(1)}_j+w^{(2)}_{ij} = x^{(1)}_j > -C = w^{(2)}_{i0} \\
    \Rightarrow w^{(2)}_{i0}\vee \max_{j\in [n]} (x^{(1)}_j+w^{(2)}_{ij})=x^{(1)}_i+C, \\
    w^{(2)}_{i0}\wedge \min_{j\in [n]} (x^{(1)}_j+w^{(2)}_{ij})=-C \\
    \Rightarrow y^{(2)}_i = (w^{(2)}_{i0}\vee \max_{j\in [d]} (x^{(1)}_j+w^{(2)}_{ij})) + (w^{(2)}_{i0}\wedge \min_{j\in [d]} (x^{(1)}_j+w^{(2)}_{ij})) = x^{(1)}_i+C-C=x^{(1)}_i=a_ix_i
    \end{gather*}
    For \(i=d-1\) we get the sum of the terms \(a_{d-1}x_{d-1}, a_dx_d\). We have that
    \begin{gather*}
    x^{(1)}_{d-1} + w^{(2)}_{(d-1)(d-1)} = x^{(1)}_{d-1}+C > x^{(1)}_j + 0 = x^{(1)}_j+w^{(2)}_{(d-1)j}, j < d-1, \\
    x^{(1)}_{d-1} + w^{(2)}_{(d-1)(d-1)} = x^{(1)}_{d-1}+C > x^{(1)}_d - C = x^{(1)}_d+w^{(2)}_{(d-1)d}, \\
    x^{(1)}_{d-1} + w^{(2)}_{(d-1)(d-1)} = x^{(1)}_{d-1}+C > 0 = w^{(2)}_{(d-1)0} \\
    \Rightarrow w^{(2)}_{(d-1)0}\vee \max_{j\in [d]} (x^{(1)}_{d-1}+w^{(2)}_{(d-1)j})=x^{(1)}_{d-1}+C, \\
    x^{(1)}_{d} + w^{(2)}_{(d-1)d} = x^{(1)}_{d}-C < x^{(1)}_j + 0 = x^{(1)}_j+w^{(2)}_{(d-1)j}, j < d-1, \\
    x^{(1)}_{d} + w^{(2)}_{(d-1)d} = x^{(1)}_{d}-C < x^{(1)}_{d-1} + C = x^{(1)}_{d-1}+w^{(2)}_{(d-1)(d-1)}, \\
    x^{(1)}_{d} + w^{(2)}_{(d-1)d} = x^{(1)}_{d}-C < 0 = w^{(2)}_{(d-1)0} \\
    \Rightarrow w^{(2)}_{(d-1)0}\wedge \min_{j\in [d]} (x^{(1)}_{d-1}+w^{(2)}_{(d-1)j})=x^{(1)}_{d}-C \\
    \Rightarrow y^{(2)}_{d-1} = (w^{(2)}_{(d-1)0}\vee \max_{j\in [d]} (x^{(1)}_{d-1}+w^{(2)}_{(d-1)j})) + (w^{(2)}_{(d-1)0}\wedge \min_{j\in [d]} (x^{(1)}_{d-1}+w^{(2)}_{(d-1)j}))= x^{(1)}_{d-1}+C + x^{(1)}_{d}-C \\
    =x^{(1)}_{d-1}+x^{(1)}_d=a_{d-1}x_{d-1}+a_dx_d
    \end{gather*}
    For the inequalities to hold, it suffices that \(C>\|\mathbf{x}^{(1)}\|_1=|a_1x_1|+\ldots+|a_dx_d|\). Take the activation of the second layer to be
    \[
    \alpha^{(2)}_i=1
    \]
    Then, for the \(i\)-th activated output we have that
    \begin{equation}
    \label{eq:uni_2}
    x^{(2)}_i=y^{(2)}_i=a_ix_i, \quad i<d-1, \quad x^{(2)}_{d-1}=y^{(2)}_{d-1}=a_{d-1}x_{d-1}+a_dx_d
    \end{equation}

    We repeat this process for a total of \(d\) layers (\(1^{st}\) layer for multiplication with \(a_i\), the rest for adding the terms) until all the terms \(a_ix_i\) have been summed up and we end up with a single output. To add the bias term \(b\), we add one final layer with \(1\) output and \(1\) input, with weights defined as follows:
    \[
    w^{(d+1)}_{11} = b, \quad w^{(d+1)}_{10}=0
    \]
    Then, we have that
    \begin{gather}
    x^{(d+1)}_1=y^{(d+1)}_1=(w^{(d+1)}_{10} \vee (x^{(d)}_1+w^{(d+1)}_{11})) + (w^{(d+1)}_{10} \wedge (x^{(d)}_1+w^{(d+1)}_{11})) \notag \\ \label{eq:uni_3}
    = w^{(d+1)}_{10} + (x^{(d)}_1+w^{(d+1)}_{11}) = 0+x^{(d)}_1+b=a_1x_1+\ldots+a_dx_d+b
    \end{gather}
    For all the inequalities to hold, it suffices that the following hold:
    \[
    C > \|\mathbf{x}\|_1,
    \]
    \[
    C>\|\mathbf{x}^{(1)}\|_1=|a_1x_1|+\ldots+|a_dx_d|, 
    \]
    \[
    C>\|\mathbf{x}^{(2)}\|_1=|a_1x_1|+\ldots +|a_{d-2}x_{d-2}| + |a_{d-1}x_{d-1}+a_dx_d|, 
    \]
    \[
    \vdots
    \]
    \[
    C> \|\mathbf{x}^{(d)}\|_1=|a_1x_1+\ldots+a_dx_d|
    \]
    For the above to hold, it suffices that the following holds
    \[
    C > (1+\max_{i\in [d]}|a_i|)\|\mathbf{x}\|_1
    \]
    We simply choose \(C= (1+\max_{i\in [d]}|a_i|)R\) and we are done. 
    
    So far we have shown how to built the single output function \(\mathbf{a}^\top \mathbf{x} + b\). We can also build networks with multiple outputs each of which is an affine function of the input. 
    
    Say we want to build a network which outputs the functions \(\mathbf{a}^\top_1 \mathbf{x}+b_1, \ldots, \mathbf{a}^\top_K \mathbf{x}+b_K\). We can use the exact same method as the one used for the single function \(\mathbf{a}^\top\mathbf{x} + b\). Specifically, we build "parallel" networks, each calculating a single function independent of each other. 
    
    Again, let a bounded domain \(\mathcal{B}\) and let some \(R>0\) such that \(\mathcal{B} \subseteq B_1(\mathbf{0}, R)\), where \(B_1(\mathbf{0}, R)\) denotes the \(\ell_1\) ball centered at the origin with \(\ell_1\) radius \(R\). Let \(C>0\) be a large constant to be determined. Write \(\mathbf{a}^\top_k \mathbf{x} + b_k = a_{k,1}x_1+\ldots+a_{k,d}x_d+b_k\). First, we build each \(a_{k,i}x_i\) term for all \(k\in [K], i\in [d]\). Take the first MPM layer to have \(Kd\) outputs and \(d\) inputs, with weights
    \[
    w^{(1)}_{(i+(k-1)d),i}=+C,
    \]
    \[
    w^{(1)}_{(i+(k-1)d),j}=0, \quad i\neq j, j\neq 0,
    \]
    \[
    w^{(1)}_{(i+(k-1)d),0} = -C,
    \]
    where \(k\in [K]\). Then, similarly to before, if \(C>\|\mathbf{x}\|_1\) we have that
    \[
    y^{(1)}_{i+(k-1)d}=x_i+C-C=x_i,
    \]
    where \(k\in [K]\). For the activations, take 
    \[
    \alpha^{(1)}_{i+(k-1)d}=a_{k,i}, \quad k\in [K], i\in [d].
    \]
    Then, for the (\(i+(k-1)d\))-th activated output we have that
    \[
    x^{(1)}_{i+(k-1)d}=a_{ki}y^{(1)}_{i+(k-1)d}=a_{k,i}x_i, \quad k\in [K], i\in [d]
    \]
    Next, we start summing up the terms. First, we take the sums \(a_{1,(d-1)}x_{d-1}+a_{1,d}x_d\), \ldots, \(a_{K,(d-1)}x_{d-1}+a_{K,d}x_d\). We use the same method as the one used previously. Take a MPM layer with \(K(d-1)\) outputs and \(Kd\) inputs. The weights of the second layer are as follows: 
    \[
    w^{(2)}_{k(d-1),k(d-1)} = +C, \quad  w^{(2)}_{k(d-1), kd}=-C,\quad w^{(2)}_{k(d-1),j}=0, \forall k\in [K], j\neq kd, k(d-1),
    \]
    \[
    w^{(2)}_{i+(k-1)(d-1),i+(k-1)(d-1)}=+C, \quad k\in [K], i<d-1,
    \]
    \[
    w^{(2)}_{i+(k-1)(d-1),j}=0, \quad k\in [K], i<d-1, j\neq i+(k-1)(d-1), j\neq 0,
    \]
    \[
    w^{(2)}_{i+(k-1)(d-1),0} = -C, \quad k\in [K], i<d-1,
    \]
    Similarly to before, if \(C>\|\mathbf{x}^{(1)}\|_1=|a_{11}x_1|+\ldots+|a_{1d}x_d|+\ldots+|a_{Kd}x_d|\), then we have
    \[
    y^{(2)}_{i+(k-1)(d-1)}=a_{ki}x_i, \quad k\in [K], i<d-1
    \]
    \[
    y^{(2)}_{k(d-1)}=a_{k,(d-1)}x_{d-1}+a_{k,d}x_d, \quad k\in [K]
    \]
    Take the activation of the second layer to be \(
    a^{(2)}_{i+(k-1)(d-1)}=1, \quad i\in [d-1], k\in[K].
    \)
    Then, for the \(i\)-th activated output we have that
    \[
    x^{(2)}_{i+(k-1)(d-1)}=y^{(2)}_{i+(k-1)(d-1)}=a_{k,i}x_i, \quad i<d-1, k\in [K], 
    \]
    \[
    \quad x^{(2)}_{k(d-1)}=y^{(2)}_{k(d-1)}=a_{k,(d-1)}x_{d-1}+a_{k,d}x_d, \quad k\in [K]
    \]
    We repeat this process for a total of \(d\) layers until all the terms have been summed up and we end up with \(K\) outputs. To add the bias terms \(b_k\) we add one final layer with \(K\) outputs and \(K\) inputs, with weights defined as follows:
    \[
    w^{(d+1)}_{kk}=b_k+C, \quad w^{(d+1)}_{k0}=-C, \quad w^{(d+1)}_{kj}=0, j\neq k
    \]
    Notice the slight deviation from the previous method. If \(C > \|\mathbf{x}^{(d)}\|_1+\max_{k\in [K]} |b_k|\), then we have that
    \[
    x^{(d+1)}_k=y^{(d+1)}_k=a_{k1}x_1+\ldots+a_{kd}x_d + b_k
    \]
    For the inequalities to hold, it suffices that the following hold:
    \[
    C > \|\mathbf{x}\|_1, 
    \]
    \[
    C > \|\mathbf{x}^{(1)}\|_1=|a_{11}x_1|+\ldots+|a_{1d}x_d|+\ldots+|a_{Kd}x_d|, 
    \]
    \[
    C>\|\mathbf{x}^{(2)}\|_1=|a_{11}x_1|+\ldots+|a_{1(d-1)}x_{d-1}+a_{1d}x_d|+\ldots+|a_{K(d-1}x_{d-1}+a_{Kd}x_d|,
    \]
    \[
    \vdots
    \]
    \[
    C>\|\mathbf{x}^{(d)}\|_1 + \max_{k\in [K]} |b_k|=|a_{11}x_1+\ldots+a_{1d}x_d|+\ldots+|a_{K1}x_1+\ldots+a_{Kd}x_d| + \max_{k\in [K]} |b_k|
    \]
    For the above to hold, it suffices that the following holds
    \[
    C>(1+K \max_{i\in [d], k\in [K]}|a_{ki}|)\|\mathbf{x}\|_1 + \max_{k\in [K]} |b_k|
    \]
    We simply take \(C>(1+K \max_{i\in [d], k\in [K]}|a_{ki}|)R + \max_{k\in [K]} |b_k|\) and we are done. 

    To finish the proof, we will prove that we can build functions of the form \(\max_{k\in [K]}(\mathbf{a}^\top_k \mathbf{x}+b_k)+\min_{m\in [M]}(\mathbf{c}^\top_m \mathbf{x}+d_m)\). First, using the previous method, build a network with \(K+M\) outputs, of which the first \(K\) outputs are the functions \(\mathbf{a}^\top_k \mathbf{x} + b_k, k\in [K]\), and the last \(M\) outputs are the functions \(\mathbf{c}^\top_m \mathbf{x} + d_m, m\in [M]\). Activate the final layer with an identity activation. Let \(C'>0\) be a large constant to be determined. Place another final layer with \(1\) output and \(K+M\) inputs. The weights of the new final layer are as follows
    \[
    w^{(d+2)}_{10}=0, \quad w^{(d+2)}_{1i}=+C', i\in [K], \quad w^{(d+2)}_{1(K+i)}=-C', i\in [M]
    \]
    Then, we have that
    \[
    \mathbf{a}^\top_k \mathbf{x} + b_k + C' > 0 > \mathbf{c}^\top_m \mathbf{x} + d_m,-C', k\in [K], m\in [M],
    \]
    and hence, the output will be given by
    \[
    x^{(d+2)}_1=y^{(d+2)}_1=\max_{k\in [K]}(\mathbf{a}^\top_k \mathbf{x} + b_k + C) + \min_{m\in [M]}(\mathbf{c}^\top \mathbf{x} + d_k - C) = \max_{k\in [K]}(\mathbf{a}^\top_k \mathbf{x} + b_k) + \min_{m\in [M]}(\mathbf{c}^\top \mathbf{x} + d_k).
    \]
    For the above inequalities to hold, it suffices that
    \[
    C'>(1+(K+M) \max(\max_{i\in [d], k\in [K]}|a_{ki}|, \max_{i\in [d], m\in [M]}|c_{mi}|))R + \max(\max_{k\in [K]} |b_k|, \max_{m\in [M]} |d_m|).
    \]
    This finishes the proof. 
\end{proof}

We should point out that for the purpose of simplifying the proof, we summed up the terms one-by-one, resulting in \(\Theta(d)\) layers. However, one can easily sum up the terms in a hierchical fashion, resulting in \(\Theta(\log(d))\) layers. In addition, in each layer, each unit has \(O(1)\) active inputs (inputs that can even win the maximum or minimum). In addition, we required width \(O(d)\) for each affine term, and \(O(dK)\) in total. Hence, we have \(O(d \log d \cdot K)\) total active parameters, a factor of \(O(\log d)\) then those required by the corresponding maxout construction of \citet{goodfellow2013maxout}.

The above proof can also be extended to the RMPM networks. It suffices to adjust steps \eqref{eq:uni_1}, \eqref{eq:uni_2}, and \eqref{eq:uni_3} to work with residual connections. We avoid another lengthy derivation, and instead show how these steps can be adjusted. 

\sixth*
% \begin{theorem6}
%     If the domain of the input is compact (i.e., bounded and closed), and assuming no residual skip connection for the output layer, the Residual-Max-Plus-Min (RMPM) network is a universal approximator.
% \end{theorem6}
\begin{proof}
    To avoid repetition with Theorem \ref{theorem:4}, in this proof wherever dots (\(\ldots\)) are used in equations, they indicate arguments that are dominated and never chosen by \(\max\) or \(\min\). Also, in contrast to Theorem \ref{theorem:4}, here the biases of the \(\max\) and the \(\min\) are not taken to be the same. 
    
    \eqref{eq:uni_1} sets the weights of the MPM appropriately so that the first layer performs scaling for each input. For the RMPM, if the first layer contains a residual skip connection, then \eqref{eq:uni_1} becomes
    \[
    x^{(1)}_i = x_i + \tilde{a}_i x_i = (1+\tilde{\alpha}_i) x_i. 
    \]
    For the RMPM we select \(\tilde{a}_i = a_i - 1\), and then the outputs of the first layer are matching by using the same large constants. 

    \eqref{eq:uni_2} performs either a pass-through for a variable, or pair-wise addition of others. Since we now have skip connections, we can instead set outputs to 0 or allow a pass-through respectively to get the same behavior. In particular, for \(i < d-1\), instead of setting \(w^{(2)}_{ii} = +C\), we set \(w^{(2)}_{i0} = C \neq m^{(2)}_{i0} = -C\), and after including the residual skip we get
    \[
    y^{(2)}_i = x^{(1)}_i + (C \vee \max_{j \in [d]}(\ldots)) + (-C\wedge \min_{j\in [d]}(\ldots)) = x^{(1)}_i + C - C = x^{(1)}_i, 
    \]
    and we recover the same output as the MPM. For \(i=d-1\), instead of setting \(w^{(2)}_{(d-1)(d-1)} = +C\), we set \(w^{(2)}_{(d-1)0} = +C \neq m^{(2)}_{(d-1)0} = 0\), and after including the residual skip we get
    \[
    y^{(2)}_{d-1} = x^{(1)}_{d-1} + (+C \vee \max_{j \in [d]}(\ldots)) + (0\wedge \min(x^{(1)}_d-C, \ldots)) = x^{(1)}_{d-1} + C + x^{(1)}_d - C,
    \]
    and we again recover the same output as the MPM. Choosing the same large constants suffices. 

    Finally, \eqref{eq:uni_3} adds a bias to the input. Instead of setting \(w^{(d+1)}_{11}=b\), we set \(w^{(d+1)}_{10} = C+b \neq m^{(d+1)}_{10} = -C\), and after including the residual skip we get
    \[
    x^{(d+1)}_1 = x^{(d)}_i + (b+C \vee \max(\ldots))+(-C \wedge \min(\ldots)) = x^{(d)}_i+b,
    \]
    and we recover the same output as the MPM. Choosing the same large constants suffices. 

    With \eqref{eq:uni_1}, \eqref{eq:uni_2}, and \eqref{eq:uni_3} properly adjusted, the remainder of the proof of Theorem \ref{theorem:4} can also be adjusted since it uses the same exact decompositions (with the exception of the final merging layer, which we took by assumption to not have a residual skip connection). 
\end{proof}

Next, we proceed with our theorem on the universality of Hybrid-MLP. We are going to need the following two lemmas regarding ReLU and Maxout networks. 

\begin{lemma}
\label{theorem:c1}
    Suppose we are given a ReLU activated fully connected linear network with \(L\) linear layers \((A^{(n)}, b^{(n)}),n\in [L]\), which can be recursively defined as follows
    \[
    \mathbf{y}^{(n)}=\mathbf{A}^{(n)}\mathbf{x}^{(n-1)}+\mathbf{b}^{(n)}, \quad n\in [L],
    \]
    \[
    \mathbf{x}^{(n)}=\max(\mathbf{0},\mathbf{y}^{(n)}), \quad n\in [L-1], \quad \mathbf{x}^{(L)}=\mathbf{y}^{(L)}
    \]
    where max is taken element-wise. Then, for every \(n\in [L]\) it holds that
    \[
    \|\mathbf{y}^{(n)}\|_1 \leq \left( \prod_{i=1}^n \|\mathbf{A}^{(i)}\|_1 \right) \|\mathbf{x}\|_1 + \sum_{i=1}^n\left(\prod_{j=i+1}^n \|\mathbf{A}^{(j)}\|_1\right) \|\mathbf{b}^{(i)}\|_1. 
    \]
\end{lemma}

\begin{proof}
    We will prove that for every \(n\in \{0,\ldots,L-1\}\) it holds that
    \[
    \|\mathbf{x}^{(n)}\|_1 \leq \left( \prod_{i=1}^n \|\mathbf{A}^{(i)}\|_1 \right) \|\mathbf{x}\|_1 + \sum_{i=1}^n\left(\prod_{j=i+1}^n \|\mathbf{A}^{(j)}\|_1\right) \|\mathbf{b}^{(i)}\|_1. 
    \]
    For \(n=0\) the inequality obviously holds. For \(n\in [L-1]\) we have that
    \[
    \|\mathbf{x}^{(n)}\|_1 = \|\max(\mathbf{0}, \mathbf{y}^{(n)})\|_1 = \sum_{j} |\max(0, y^{(n)}_j)| \leq \sum_{j} |y^{(n)}_j| = \|\mathbf{y}^{(n)}\|_1 = \|\mathbf{A}^{(n)} \mathbf{x}^{(n-1)}+\mathbf{b}^{(n)}\|_1.
    \]
    By the sub-additivity of vector norm \(\|\cdot\|_1\) and sub-multiplicativity of matrix operator norm \(\|\cdot\|_1\) we have that
    \[
    \|\mathbf{x}^{(n)}\|_1 \leq \|\mathbf{A}^{(n)}\|_1 \|\mathbf{x}^{(n-1)}\|_1 + \|\mathbf{b}^{(n)}\|_1
    \]
    By expanding the recurrence, we have that
    \[
    \|\mathbf{x}^{(n)}\|_1 \leq \|\mathbf{A}^{(n)}\|_1 \|\mathbf{x}^{(n-1)}\|_1 + \|\mathbf{b}^{(n)}\|_1 
    \]
    \[
    \leq \|\mathbf{A}^{(n)}\|_1 \|\mathbf{A}^{(n-1)}\|_1 \|\mathbf{x}^{(n-2)}\|_1 + \|\mathbf{b}^{(n)}\|_1 + \|\mathbf{A}^{(n)}\|_1 \|\mathbf{b}^{(n-1)}\|_1 
    \]
    \[
    \leq \ldots \leq \left( \prod_{i=1}^n \|\mathbf{A}^{(i)}\|_1 \right) \|\mathbf{x}\|_1 + \sum_{i=1}^n\left(\prod_{j=i+1}^n \|\mathbf{A}^{(j)}\|_1\right) \|\mathbf{b}^{(i)}\|_1. 
    \]
    Then, for every \(n\in [L]\) we have that
    \[
    \|\mathbf{y}^{(n)}\|_1 \leq \|\mathbf{A}^{(n)}\|_1 \|\mathbf{x}^{(n-1)}\|_1 + \|\mathbf{b}^{(n)}\|_1, 
    \]
    and the result follows immediately. 
\end{proof}

\begin{lemma}
\label{theorem:c2}
    Suppose we are given a Maxout fully connected linear network with \(L\) layers and pooling of \(P\) which can be recursively defined as follows
    \[
    \mathbf{y}^{(n)}_p=\mathbf{A}^{(n)}_p\mathbf{x}^{(n-1)}+\mathbf{b}^{(n)}_p, \quad n\in [L-1], p\in [P],
    \]
    \[
    \mathbf{x}^{(n)}=\max_{p\in [P]}(\mathbf{y}^{(n)}_p), \quad n\in [L-1], \quad \mathbf{x}^{(L)}=\mathbf{y}^{(L)} = \mathbf{A}^{(L)}\mathbf{x}^{(L-1)} + \mathbf{b}^{(L)}
    \]
    where max is taken element-wise. Then, for every \(n\in [L-1]\) it holds that
    \[
    \sum_{p=1}^P \|\mathbf{y}^{(n)}_p\|_1 \leq \left( \prod_{i=1}^n \sum_{p\in [P]}\|\mathbf{A}^{(i)}_p\|_1 \right) \|\mathbf{x}\|_1 + \sum_{i=1}^n\left(\prod_{j=i+1}^n \sum_{p\in [P]}\|\mathbf{A}^{(j)}_p\|_1\right) \left(\sum_{p\in [P]}\|\mathbf{b}^{(i)}_p\|_1\right)
    \]
\end{lemma}

\begin{proof}
    First, we will prove that for every \(n\in \{0, \ldots, L-1\}\) it holds that
    \[
    \|\mathbf{x}^{(n)}\|_1 \leq \left( \prod_{i=1}^n \sum_{p\in [P]}\|\mathbf{A}^{(i)}_p\|_1 \right) \|\mathbf{x}\|_1 + \sum_{i=1}^n\left(\prod_{j=i+1}^n \sum_{p\in [P]}\|\mathbf{A}^{(j)}_p\|_1\right) \left(\sum_{p\in [P]}\|\mathbf{b}^{(i)}_p\|_1\right). 
    \]
    For \(n=0\) the inequality obviously holds. For \(n\in [L-1]\) we have that
    \[
    \|\mathbf{x}^{(n)}\|_1=\|\max_{p\in [P]}(\mathbf{y}^{(n)}_p)\|_1=\sum_{j}|\max_{p\in [P]}(y^{(n)}_{pj})| \leq \sum_{j} \sum_{p\in [P]} |y^{(n)}_{pj}| = \sum_{p\in [P]} \|\mathbf{y}^{(n)}_p\|_1
    \]
    \[
    = \sum_{p\in [P]}\|\mathbf{A}^{(n)}_p\mathbf{x}^{(n-1)}+\mathbf{b}^{(n)}_p\|_1
    \]
    By the sub-additivity of vector norm \(\|\cdot\|_1\) and sub-multiplicativity of matrix operator norm \(\|\cdot\|_1\) we have that
    \[
    \|\mathbf{x}^{(n)}\|_1 \leq \left( \sum_{p\in [P]} \|\mathbf{A}^{(n)}_p\|_1 \right) \|\mathbf{x}^{(n)}\|_1 + \left( \sum_{p\in [P]}\|\mathbf{b}^{(n)}_p\|_1 \right)
    \]
    By expanding the recurrence, we have that
    \[
    \|\mathbf{x}^{(n)}\|_1 \leq \left( \sum_{p\in [P]} \|\mathbf{A}^{(n)}_p\|_1 \right) \|\mathbf{x}^{(n)}\|_1 + \left( \sum_{p\in [P]}\|\mathbf{b}^{(n)}_p\|_1 \right)
    \]
    \[
    \leq \left( \sum_{p\in [P]} \|\mathbf{A}^{(n)}_p\|_1 \right) \left( \sum_{p\in [P]} \|\mathbf{A}^{(n-1)}_p\|_1 \right) \|\mathbf{x}^{(n-2)}\|_1 + \left( \sum_{p\in [P]}\|\mathbf{b}^{(n)}_p\|_1 \right)
    \]
    \[
    + \left( \sum_{p\in [P]} \|\mathbf{A}^{(n)}_p\|_1 \right) \left( \sum_{p\in [P]}\|\mathbf{b}^{(n-1)}_p\|_1 \right)
    \]
    \[
    \leq \ldots \leq \left( \prod_{i=1}^n \sum_{p\in [P]}\|\mathbf{A}^{(i)}_p\|_1 \right) \|\mathbf{x}\|_1 + \sum_{i=1}^n\left(\prod_{j=i+1}^n \sum_{p\in [P]}\|\mathbf{A}^{(j)}_p\|_1\right) \left(\sum_{p\in [P]}\|\mathbf{b}^{(i)}_p\|_1\right). 
    \]
    Then, for every \(n\in [L-1]\) we have that
    \[
    \sum_{p\in[P]}\|\mathbf{y}^{(n)}_p\|_1 \leq \left( \sum_{p\in [P]} \|\mathbf{A}^{(n)}_p\|_1 \right) \|\mathbf{x}^{(n-1)}\|_1 + \left( \sum_{p\in [P]}\|\mathbf{b}^{(n)}_p\|_1 \right),
    \]
    and the result follows immediately. 
\end{proof}

\fifth*
% \begin{theorem5}
% If the domain of the input is compact, the Hybrid-MLP is a universal approximator. In fact, any fully connected ReLU or maxout network is a special case of the Hybrid-MLP. 
% \end{theorem5}

\begin{proof}
    We will show that for any ReLU and Maxout networks, if the domain \(\mathcal{B}\) is bounded (i.e. there exists \(R\) such that \(\mathcal{B}\subseteq B_1(\mathbf{0}, R)\)), then there exists a Hybrid-MLP network that gives the same output. By our construction it will immediately follow that for bounded domains, ReLU and Maxout networks are special cases of the Hybrid-MLP. 

    First, we focus on ReLU activated networks. Take a network as defined in Lemma \ref{theorem:c1}. Layer-by-layer, we will replace its ReLU activations with our proposed morphological layer. Pick an activated layer \(n\in [L-1]\) and some large constant \(C>0\) to be determined. Remove the ReLU activations and replace them by our morphological layer:
    \[
    \mathbf{y}^{(n)}=\mathbf{A}^{(n)}\mathbf{x}^{(n-1)}+\mathbf{b}^{(n)},
    \]
    \[
    \mathbf{x}^{(n)}=\left( \mathbf{w}^{(n)}_i \vee \mathbf{W}^{(n)} \boxplus \mathbf{y}^{(n)} \right) + \left( \mathbf{m}^{(n)}_i \wedge \mathbf{W}^{(n)} \boxplus' \mathbf{y}^{(n)} \right).
    \]
    The weights of the morphological layer are defined as follows:
    \[
    w^{(n)}_{i0}=+C, \quad m^{(n)}_{i0}=-C,
    \]
    \[
    w^{(n)}_{ii} = +C,
    \]
    \[
    w^{(n)}_{ij} = 0, j\neq i.
    \]
    Then, for the \(i\)-th output, we have that
    \[
    x^{(n)}_i=(C\vee (y^{(n)}_i+C)\vee\max_{j\neq i}(y^{(n)}_j))+ ((-C)\wedge (y^{(n)}_i+C)\wedge\min_{j\neq i}(y^{(n)}))=(C\vee (y^{(n)}_i+C))+(-C) 
    \]
    \[
    = (C+\max(0, y^{(n)}_i))+(-C) = \max(0, y^{(n)}_i), 
    \]
    and the output is the same as the ReLU network. For the above to hold, it suffices that \(C>\|\mathbf{y}^{(n)}\|_1\). According to Lemma \ref{theorem:c1}, it suffices to pick
    \[
    C > \left( \prod_{i=1}^n \|\mathbf{A}^{(n)}\|_1 \right) R+ \sum_{i=1}^n\left(\prod_{j=i+1}^n \|\mathbf{A}^{(i)}\|_1\right) \|\mathbf{b}^{(n)}\|_1. 
    \]

    We continue with Maxout networks. Take a network as defined in Lemma \ref{theorem:c2}. If we concatenate the vectors \(\mathbf{y}^{(n)}_p\), then maxout networks are effectively a maxpooling layer. We will replace this maxpooling layer with our morphological layer. 

    Pick a maxout-activated layer \(n\in [L-1]\) and some large constant \(C>0\) to be determined. Concatenate the vectors \(\mathbf{y}^{(n)}_p\)  to form the vector \(\mathbf{g}^{(n)}\). Suppose that the \(n\)-th layer has output size \(N\). Then, the vector \(\mathbf{g}^{(n)}\) has dimension \(NP\). The \(i\)-th output \(x^{(n)}_i\) will be given by
    \[
    x^{(n)}_i=\max_{j=i, i+N,\ldots, i+(P-1)N} (g^{(n)}_j)
    \]
    Remove the maxpooling layer, and replace it by our morphological layer of output size \(N\) and input size \(NP\):
    \[
    \mathbf{y}^{(n)}_p=\mathbf{A}^{(n)}_p\mathbf{x}^{(n-1)}+\mathbf{b}^{(n)},
    \]
    \[
    \mathbf{g}^{(n)}=\text{concat}_{p\in [P]}(\mathbf{y}^{(n)}_p),
    \]
    \[
    \mathbf{x}^{(n)}=\left( \mathbf{w}^{(n)}_i \vee \mathbf{W}^{(n)} \boxplus \mathbf{g}^{(n)} \right) + \left( \mathbf{m}^{(n)}_i \wedge \mathbf{W}^{(n)} \boxplus' \mathbf{g}^{(n)} \right).
    \]
    The weights of the morphological layer are defined as follows:
    \[
    w^{(n)}_{i,0} = +C, m^{(n)}_{i,0}=-C, \quad i\in [N],
    \]
    \[
    w^{(n)}_{i, i+(k-1)N} = +C, \quad i\in [N], k\in [P],
    \]
    \[
    w^{(n)}_{i, j}=0, \quad j\neq i+(k-1)N, \text{for some } k\in [P]
    \]
    Then, for the \(i\)-th output, we have that
    \[
    x^{(n)}_i=\left(-C\vee \max_{j=i,\ldots,i+(P-1)N}(g^{(n)}_j+C)\vee \max_{j\notin \{i, \ldots, i+(P-1)N\}}(g^{(n)})\right)
    \]
    \[
    + \left(-C \wedge \min_{j=i,\ldots,i+(P-1)N}(g^{(n)}_j+C)\wedge \min_{j\notin \{i, \ldots, i+(P-1)N\}}(g^{(n)})\right)
    \]
    \[
    = \left(\max_{j=i,\ldots,i+(P-1)N}(g^{(n)}_j+C)\right) + (-C) = \max_{j=i,\ldots,i+(P-1)N}(g^{(n)}_j),
    \]
    and the output is the same as the Maxout network. For the above to hold, it suffices that \(C>\|\mathbf{g}^{(n)}\|_1 = \sum_{p\in [P]}\|\mathbf{y}^{(n)}_p\|_1\). According to Lemma \ref{theorem:c2} it suffices to pick
    \[
    C > \left( \prod_{i=1}^n \sum_{p\in [P]}\|\mathbf{A}^{(i)}_p\|_1 \right) R + \sum_{i=1}^n\left(\prod_{j=i+1}^n \sum_{p\in [P]}\|\mathbf{A}^{(j)}_p\|_1\right) \left(\sum_{p\in [P]}\|\mathbf{b}^{(i)}_p\|_1\right). 
    \]

    So far we have proven that for bounded domains, we can build any ReLU or Maxout network with a Hyrbid-MLP. And in fact by our construction, ReLU and Maxout networks are special cases of the Hybrid-MLP. Since these networks are universal approximators, it follows that Hybrid-MLP will be a universal approximator on compact domains. 
\end{proof}

We continue with a thorough exposition of Example~\ref{example:1}. 
\begin{repeatedexample}
    Consider the uni-variate function 
    \[
    f^{(\epsilon)}(x)=\max(-|x|+\epsilon, -|x|/\epsilon+1)
    \] on \([-1, 1]\) for \(\epsilon > 0\). The optimal min-max-plus network approximating this function uses a multiplier \(1/\epsilon\) on the input, and \(2\operatorname{round}(\frac{1-\epsilon}{2\epsilon d})+1\) pyramids to achieve an approximation error of \(\delta(1-\epsilon^2)>0\). 
    
    Indeed, the function is \(1/\epsilon\)-Lipschitz, so an \(1/\epsilon\) multiplier will be used. Without sampling a point from the middle part, the error will always be at least \(f(0)=1\). Hence, we are forced to sample from the middle part. Then, there is no reason not to sample the point \(x = 0\), since this covers the whole middle part. After that, we have to cover the flatter \(-|x|\) regions. Since we are working in 1D, we can work greedily, by placing a new pyramid exactly when the error is about to surpass our upper bound \(\delta(1-\epsilon^2)\) (if we choose another construction, then we can push the pyramids to the right or left, so long as it does not exceed our error bound, and then we obtain a construction equal to the greedy construction, plus some unused overlapping pyramids). 

    We start from \(x=\epsilon\) being covered by the middle pyramid. By making steps of size \(\delta x = 2d\epsilon\) we place new pyramids of slope \(1/\epsilon\). The tip of the new pyramid is \(2d\epsilon\) lower than the previously placed pyramid. Hence, they meet at the point \(x' \in [x, x+2d\epsilon]\) for which \((x'-(x+2d\epsilon))/\epsilon = 2d\epsilon - (x'-x)/\epsilon \Rightarrow x'-x = d\epsilon(1+\epsilon)\). The error at this maximal point will be \(|d\epsilon(1+\epsilon)(1/\epsilon-1){}| = d(1-\epsilon^2)\)
    
    For \(\epsilon \to 0^+\) and \(\delta(1-\epsilon^2) = \mathrm{const} \Rightarrow \delta \approx \mathrm{const}\), this requires an unbounded number of min-plus MPs. The MPM network, on the other hand, can represent the function exactly with 3 MPM units. We assume access to negations of inputs, in the same way the min-max-plus network assumes it. We build the function as follows:
    \[
    f_0(x) = 1\cdot\min(-x, x), \quad f_1(x) = 1/\epsilon \cdot \min(-x, x), \quad f(x) = 1\cdot \max(f_0(x)+\epsilon, f_1(x)+1). 
    \]
    As used in the proof of universality, the MPM unit can be decomposed into either a max-plus or a min-plus MPM. In the above, we used two min-plus and one max-plus decompositions. The increase in expressivity is due to the fact that the scalar multipliers worked as an activation, after the output of the first min-plus layer, allowing more diversity in the slopes of the pyramids that are used. 
\end{repeatedexample}

Finally, we prove the 1D approximation theorem for the shallow networks.

\eighth*

\begin{proof}
    Let \(x_1<x_2<\ldots < x_{N_s}\) be the samples of the optimal linear interpolator for achieving an error of \(\epsilon > 0\). By the mean value theorem on the secants, the inflection points of the linear interpolation will be at most as many as the inflection points of the underlying function. Let \(i_1, \ldots, i_{N_i}\) be their indices, which partition it into concave and convex parts. 

    The first MPM layer constructs slopes. For every adjacent pair of points \((x_i, x_{i+1})\), we use an MPM unit acting as a pass-through and the linear activation constructs the slope:
    \begin{gather*}
    w^{(0)}_{i, 1} = +C,\ w^{(0)}_{i, 0} = -C, \ a_{i} = \frac{f(x_{i+1})-f(x_i)}{x_{i+1} - x_i} \\
    \Rightarrow \max(x+w^{(0)}_{i, 1}, w^{(0)}_{i, 0})+\min(x+w^{(0)}_{i, 1}+w^{(0)}_{i, 0}) = x + C - C = x \\
    \Rightarrow x^{(1)}_i(x) = \frac{f(x_{i+1})-f(x_i)}{x_{i+1} - x_i} x. 
    \end{gather*}
    We also use two additional slopes \(+1\) and \(-1\):
    \[
    a_{N_s} = 1, \ a_{N_s+1} =-1 \quad 
    \Rightarrow \quad x^{(1)}_{N_s}(x) = x, \ x^{(1)}_{N_s+1}(x) = -x.
    \]
    The second layer will be a mix of max-plus and min-plus MPs constructed from the MPM. Consider a convex part, say between \(i_k\) and \(i_{k'}\). Since the linear interpolation is continuous PWL, we can write it as the maximum of its secants. We have already constructed the slopes, so now we need translations to build the secants, and a maximum. In other words, we need a selective max-plus MP. We use a MPM unit with the following weights:
    \[
    w^{(2)}_{j, i} = f(x_{i})+C, \ i\in \{i_k, \ldots i_{k'}\}, \quad w^{(2)}_0 = -C, \quad w^{(2)}_{j, i} =0, \ \text{otherwise},
    \]
    and no activation. In addition, we include the \(\pm 1\) slope terms on the endpoints:
    \[
    w^{(2)}_{j, N_s} = f(i_k), \quad w^{(2)}_{j, N_s+1} = f(i_{k'}). 
    \]
    For sufficiently large \(C\), this constructs the following function:
    \[
    x^{(2)}_j = \max(x^{(1)}_{i_k}-\frac{f(x_{i_k+1})-f(x_{i_k})}{x_{i_k+1} - x_{i_k}} x_{i_k}+f(x_{i_k}), \ldots)
    \]
    which matches the linear interpolation on the convex part. For the concave parts, the construction is similar, but we need the minimum of the secants instead of the maximum. Hence, we place \(-C\) instead of \(+C\) so that the MPM behaves like a min-plus MP. We also use two MPMs as pass-throughs for the \(\pm 1\) slopes.

    The third layer secures the parts from interference. It places sharp slopes at the endpoints of each part via the constructed \(\pm 1\) slopes:
    \[
    x^{(3)}_j = \min\left(x-\frac{f(x_{i_k+1})-f(x_{i_k})}{x_{i_k+1} - x_{i_k}} x_{i_k} + f(x_{i_k}), -x+\frac{f(x_{i_{k'}-1})-f(x_{i_{k'}})}{x_{i_{k'}-1} - x_{i_{k'}}} x_{i_{k'}} + f(x_{i_{k'}}), x^{(2)}_j\right)
    \]
    The final layer combines all parts with a max-plus behaving MPM. The parts do not interfere due to the \(1\)-Lipschitz continuity of \(f\): by the mean value theorem, the linear interpolation is itself an \(1\)-Lipschitz continuous PWL function, so the \(\pm 1\) slopes at the endpoints of the parts ensure that outside of the interval of a single part, its values never win the final maximum.

    We used a total of at most \((N_s+2)+(N_i+1)+(N_i+1)+1 = N_s + 2N_i + 5\) MPM units, with a total of \(O((N_s+2)+(N_s+2))=O(N_s)\) active parameters and \(O(N_s+N_i)\) active connections. The last part of the statement follows from Example~\ref{example:1}. 

    Finding a sufficiently large constant \(C\) follows the same methodology of repeated application of the triangle inequality, and is omitted. The construction is illustrated in Figure~\ref{fig:1d_interpolation}
\end{proof}

\begin{figure}
    \centering
    \begin{subfigure}{0.24\linewidth}
        \includegraphics[width=\linewidth]{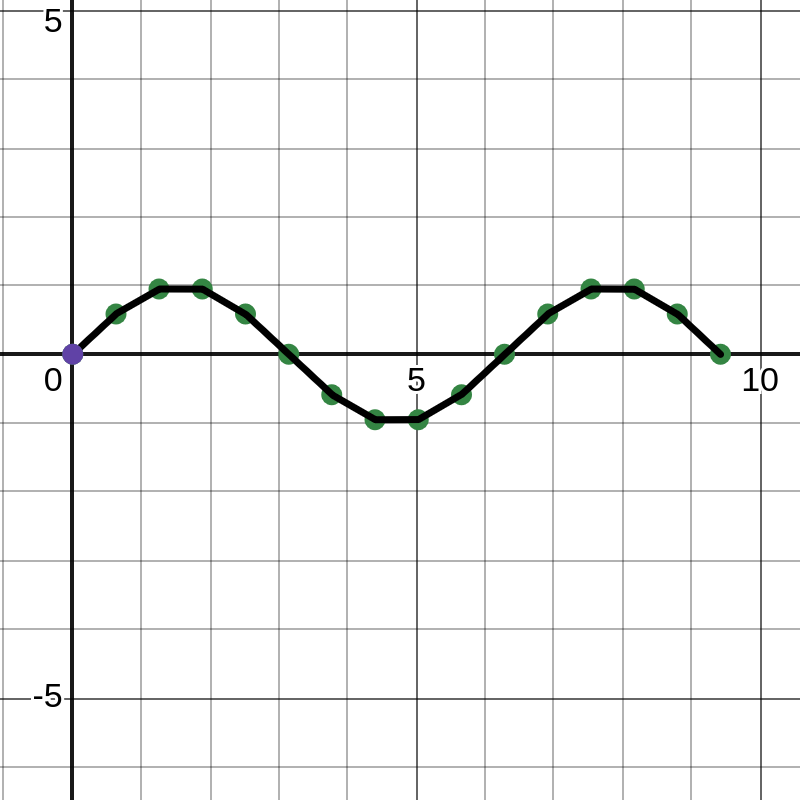}
        \caption{Linearly interpolated sin wave}
        \label{fig:sin_inter}
    \end{subfigure}
    \begin{subfigure}{0.24\linewidth}
        \includegraphics[width=\linewidth]{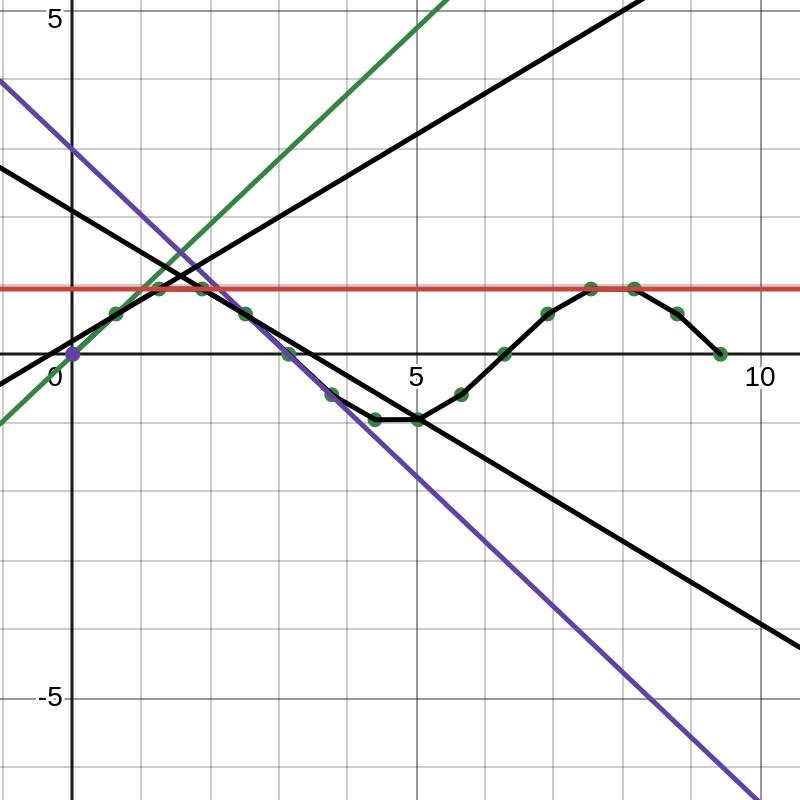}
        \caption{Slope construction of concave part}
        \label{fig:slopes_inter}
    \end{subfigure}
    \begin{subfigure}{0.24\linewidth}
        \includegraphics[width=\linewidth]{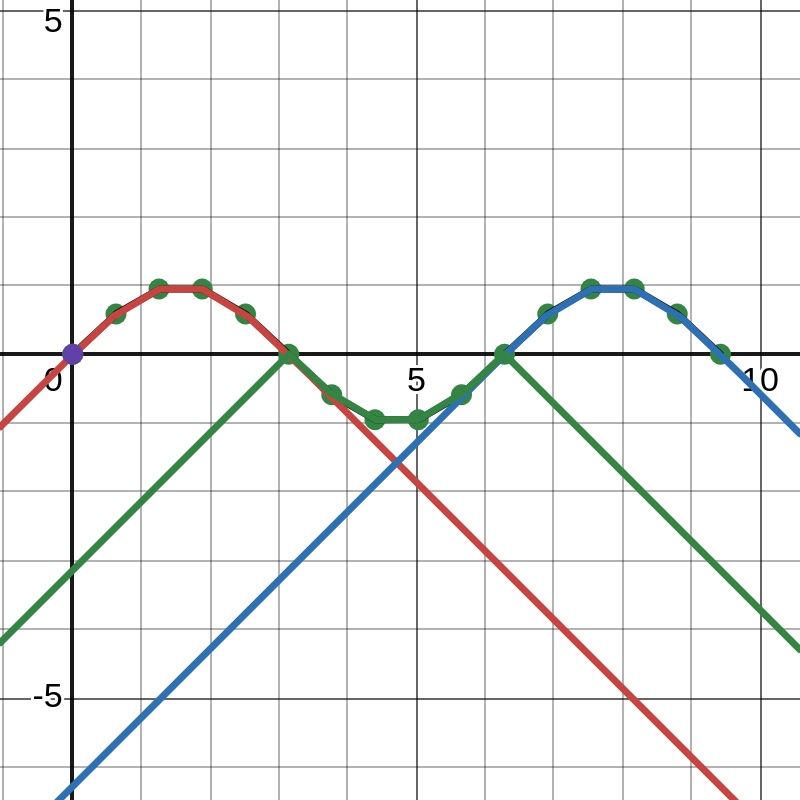}
        \caption{Approximations of convex/concave parts}
        \label{fig:parts_inter}
    \end{subfigure}
    \begin{subfigure}{0.24\linewidth}
        \includegraphics[width=\linewidth]{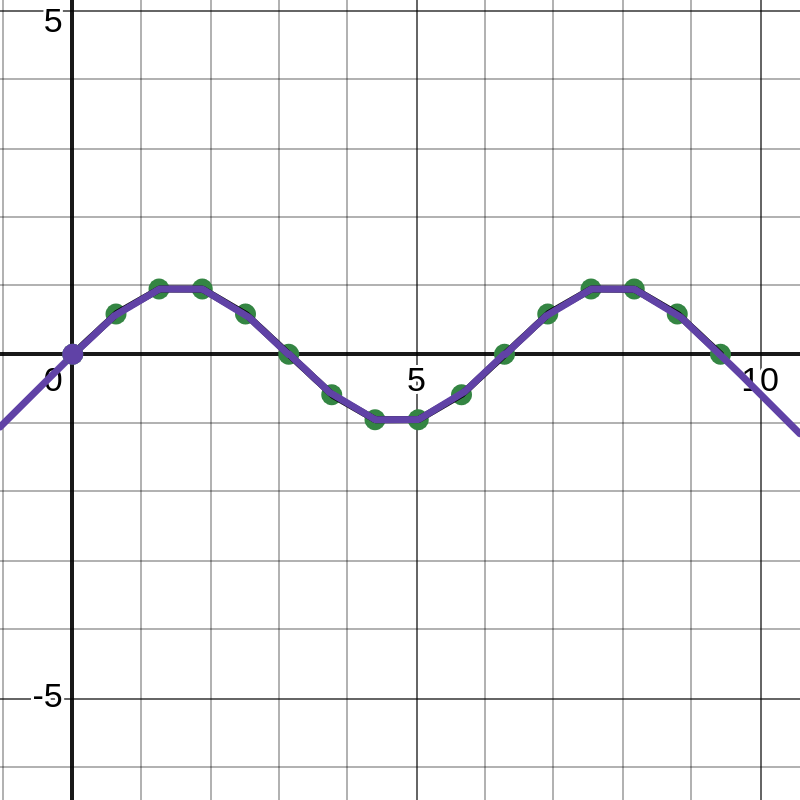}
        \caption{Max interpolation of all parts}
        \label{fig:mpm_inter}
    \end{subfigure}
    \caption{1D approximation of MPM compared to linear interpolation}
    \label{fig:1d_interpolation}
\end{figure}

\section{On the initialization of morphological neural networks}
\label{appendix:d}

In this appendix, we study what makes training morphological networks difficult, apart from the lack of activations. We showcase that a good initialization is crucial, and that this is easier to achieve with our networks as opposed with previous architectures. Finally, we explain how we initialized the networks in our experiments. Most of the claims in this appendix are qualitative and lack formal proofs. However, we still believe that they serve an important purpose in the understanding of morphological networks. 

\paragraph{Optimization landscape.}
First, we start with some examples that will allow us to gain an idea of the optimization landscapes we are faced against. 
\begin{enumerate}
    \item Consider an unbiased max-plus MP with 2 inputs and a single output, and an unbiased linear perceptron with 2 inputs and a single output. Consider 3 samples \((-1.7, 1; 2.3), (5, -2.2; 3.7), (1,1; 4.7)\). The loss functions of the two percetrons are illustrated in Figure~\ref{fig:landscape1}. As we can see, for this particular dataset, the loss function of the MP (blue) has a smaller global minimum than the loss function of the linear perceptron (green). However, the loss function of the linear perceptron is a simple paraboloid, which is easy to optimize using gradient methods. On the other hand, the loss function of the MP is more complicated, comprised of pieces with their own local minima. For most initializations, solving with a gradient method would result in reaching these local minima, which are worse than the global minimum of the paraboloid loss of the linear perceptron.
    \item Now consider a hybrid network, with a first layer of 2 unbiased max-plus MPs of 2 inputs, and a single output unbiased linear perceptron. We fix the weights of the output perceptron to 1, and the weights \(w_{12}=w_{21}=0\), study the loss function for variable weights \(w_{11}, w_{22}\). Consider again 3 samples \((1.2, -2.4; 1.4), (-3.36, 2.34; 2.16), (-2.1, -1.5; 2.4)\). The loss function of the network is illustrated in Figure~\ref{fig:landscape2}. As we can see, for this dataset, the loss function has 9 local minima. Depending on the initialization, all local minima are likely to be reached, with most of them being a lot worse than the global minimum. 
\end{enumerate}

\begin{figure}
    \centering
    \begin{subfigure}{0.45\textwidth}
        \centering
        \includegraphics[width=\linewidth]{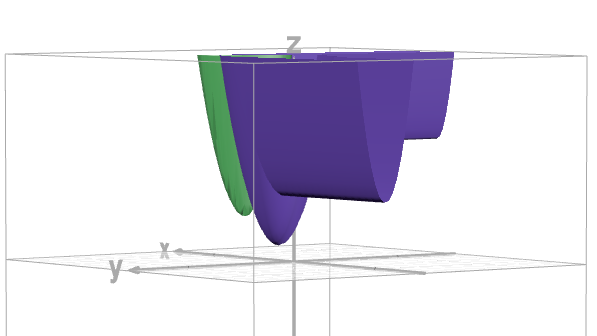}
        \caption{Landscape 1.}
        \label{fig:landscape1}
    \end{subfigure}
    \hspace{10pt}
    \begin{subfigure}{0.45\textwidth}
        \centering
        \includegraphics[width=\linewidth]{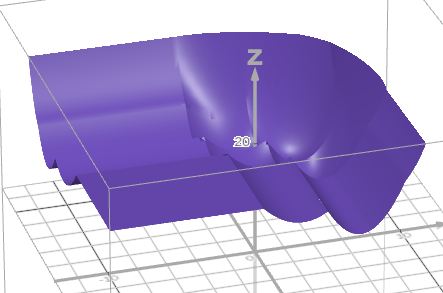}
        \caption{Landscape 2.}
        \label{fig:landscape2}
    \end{subfigure}
    \caption{Different landscapes of morphological networks.}
    \label{fig:landspaces}
\end{figure}

From the above, we understand that a good initialization is crucial for training morphological networks. 

\paragraph{Shift of the mean.}
Let us now study the optimization of morphological networks from a different viewpoint. Let \(d=10\). Suppose we want to approximate the function \(f(\mathbf{x})=\mathbf{10}^\top \mathbf{x}=10x_1+\ldots+10x_d\) using a linear perceptron \(y_1(\mathbf{x}; \mathbf{w})=\mathbf{w}^\top \mathbf{x}\). Also, suppose we want to approximate the function \(g(\mathbf{x})=\max_{i}(x_i+10)\) using a max-plus MP \(y_2(\mathbf{x}; \mathbf{m})=\max_{i}(x_i+m_i)\). Obviously, both approximations are possible with zero error by letting \(\mathbf{w}=\mathbf{m}=\mathbf{10}=(10,\ldots, 10)\). 

For both problems, a good initialization of the weights \(\mathbf{w}, \mathbf{m}\) would be with a mean of \(10\). However, such initializations are not common, because we can not know a-priori the distribution of the weights of the function we want to approximate. As such, let us initialize the weights to 0, and let us test how efficient both networks are at changing the mean of the distribution of their weights. 

First, we use Adam with a batch size of 1 for 1000 epochs. In each epoch, we generate a training sample at random according to a standard distribution and our target functions and introduce no noise. Ideally, the networks should learn the weights \(\mathbf{w}=\mathbf{m}=\mathbf{10}\), i.e. they should be able to shift the mean of the distribution of their weights, maintaining a standard deviation close to 0. The results are reported in Tables \ref{table:shift1_linear} and \ref{table:shift1_mp}. We used a learning rate of \(0.1\) as this optimized the results of the MP. As we can see, the linear perceptron worked efficiently, meaning that it eventually managed to shift the mean to \(10\). The MP, on the other hand, did not manage to change the mean of the distribution of its weights to \(10\), and it also introduced a lot of variance in the values of the weights. 

\begin{table*}
  \caption{Linear perceptron trained on shifting its mean with batch size of 1.}
  \label{table:shift1_linear}
  \centering
  \begin{tabular}{lcccccccccccc}
    \toprule
    Epoch & \(w_1\) & \(w_2\) & \(w_3\) & \(w_4\) & \(w_5\) & \(w_6\) & \(w_7\) & \(w_8\) & \(w_9\) & \(w_{10}\) & mean & std \\
    \midrule
    1 & -0.1 & 0.1 & -0.1 & -0.1 &  0.1 & 0.1 & 0.1 & 0.1 & 0.1 & 0.1 & 0.04 & 0.10 \\
    2 & -0.04 & 0.17 & -0.16 & -0.08 & 0.18 & 0.12 & 0.16 & 0.12 & 0.14 & 0.17 & 0.08 & 0.12 \\
    3 &  0 & 0.23 & -0.21 & -0.05 &  0.25 & 0.12 & 0.21 & 0.14 & 0.17 & 0.23 & 0.11 & 0.15 \\
    \(\vdots\) & \(\vdots\) & \(\vdots\) & \(\vdots\) & \(\vdots\) & \(\vdots\) & \(\vdots\) & \(\vdots\) & \(\vdots\) & \(\vdots\) & \(\vdots\) & \(\vdots\) & \(\vdots\) \\
    998 & 9.98 & 9.99 & 9.99 & 9.98 & 9.97 & 9.97 & 10 & 9.95 & 9.97 & 9.97 & 9.98 & 0.01 \\
    999 & 9.98 & 9.99 & 9.99 & 9.99 & 9.97 & 9.97 & 10 & 9.95 & 9.97 & 9.97 & 9.98 & 0.01 \\
    1000 & 9.98 & 9.99 & 9.99 & 9.99 & 9.97 & 9.97 & 10 & 9.95 & 9.98 & 9.97 & 9.98 & 0.01 \\
    \bottomrule
  \end{tabular}
\end{table*}

\begin{table*}
  \caption{MP trained on shifting its mean with batch size of 1.}
  \label{table:shift1_mp}
  \centering
  \begin{tabular}{lcccccccccccc}
    \toprule
    Epoch & \(w_1\) & \(w_2\) & \(w_3\) & \(w_4\) & \(w_5\) & \(w_6\) & \(w_7\) & \(w_8\) & \(w_9\) & \(w_{10}\) & mean & std \\
    \midrule
    1 & 0. & 0.1 & 0. & -0. & -0. & 0. &  -0. & -0. & -0. &  0. & 0.01 & 0.03 \\
    2 & 0.07 & 0.17 & 0. &  -0.  & -0. &   0.  & -0.  & -0. &  -0. &   0. & 0.02 & 0.06 \\
    3 & 0.13 & 0.22 & 0. &  -0.  & -0. &   0.06 & -0. &  -0. &  -0. & 0. & 0.04 & 0.08 \\
    \(\vdots\) & \(\vdots\) & \(\vdots\) & \(\vdots\) & \(\vdots\) & \(\vdots\) & \(\vdots\) & \(\vdots\) & \(\vdots\) & \(\vdots\) & \(\vdots\) & \(\vdots\) & \(\vdots\) \\
    998 & 2.67 & 0.86 & 0.51 & 11.67 & 0.52 & 4.27 & 1.82 & 1.7 &  1.96 & 1.06 & 2.70 & 3.35 \\
    999 & 2.67 & 0.86 & 0.51 & 11.67 & 0.52 & 4.27 & 1.82 & 1.7  & 1.96 & 1.06 & 2.70 & 3.35 \\
    1000 & 2.67 & 0.86 & 0.51 & 11.66 &  0.52 & 4.27 & 1.82 & 1.7 &  1.96 & 1.06 & 2.70 & 3.34 \\
    \bottomrule
  \end{tabular}
\end{table*}

We repeat the experiment with a batch size of \(100\). The results are reported in Tables \ref{table:shift100_linear} and \ref{table:shift100_mp}. Again, the linear perceptron excelled, learning the new mean successfully. The MP, while it performed better than the previous experiment, still failed to shift the mean of its weights to \(10\). 

\begin{table*}
  \caption{Linear perceptron trained on shifting its mean with batch size of 100.}
  \label{table:shift100_linear}
  \centering
  \begin{tabular}{lcccccccccccc}
    \toprule
    Epoch & \(w_1\) & \(w_2\) & \(w_3\) & \(w_4\) & \(w_5\) & \(w_6\) & \(w_7\) & \(w_8\) & \(w_9\) & \(w_{10}\) & mean & std \\
    \midrule
    1 & 0.1 & 0.1 & 0.1 & 0.1 & 0.1 & 0.1 & 0.1 & 0.1 & 0.1 & 0.1 & 0.10 & 0.00 \\
    2 & 0.2 & 0.2 & 0.2 & 0.19 & 0.2 & 0.2 & 0.18 & 0.2 & 0.2 & 0.2 & 0.20 & 0.01 \\
    3 & 0.3 & 0.29 & 0.3 & 0.29 & 0.3 & 0.3 & 0.26 & 0.3 & 0.3 & 0.29 & 0.29 & 0.01 \\
    \(\vdots\) & \(\vdots\) & \(\vdots\) & \(\vdots\) & \(\vdots\) & \(\vdots\) & \(\vdots\) & \(\vdots\) & \(\vdots\) & \(\vdots\) & \(\vdots\) & \(\vdots\) & \(\vdots\) \\
    998 & 10. & 10. & 10. & 10. & 10. & 10. & 10. & 10. & 10. & 10. & 10.00 & 0.00 \\
    999 & 10. & 10. & 10. & 10. & 10. & 10. & 10. & 10. & 10. & 10. & 10.00 & 0.00 \\
    1000 & 10. & 10. & 10. & 10. & 10. & 10. & 10. & 10. & 10. & 10. & 10.00 & 0.00 \\
    \bottomrule
  \end{tabular}
\end{table*}

\begin{table*}
  \caption{MP trained on shifting its mean with batch size of 100.}
  \label{table:shift100_mp}
  \centering
  \begin{tabular}{lcccccccccccc}
    \toprule
    Epoch & \(w_1\) & \(w_2\) & \(w_3\) & \(w_4\) & \(w_5\) & \(w_6\) & \(w_7\) & \(w_8\) & \(w_9\) & \(w_{10}\) & mean & std \\
    \midrule
    1 & 0.1 & 0.1 & 0.1 & 0.1 & 0.1 & 0.1 & 0.1 & 0.1 & 0.1 & 0.1 & 0.10 & 0.00 \\
    2 & 0.2 & 0.2 & 0.2 & 0.2 & 0.2 & 0.2 & 0.2 & 0.2 & 0.2 & 0.2 & 0.20 & 0.00 \\
    3 & 0.29 & 0.3 & 0.3 & 0.3 & 0.3 & 0.3 & 0.3 & 0.3 & 0.29 & 0.3 & 0.30 & 0.00 \\
    \(\vdots\) & \(\vdots\) & \(\vdots\) & \(\vdots\) & \(\vdots\) & \(\vdots\) & \(\vdots\) & \(\vdots\) & \(\vdots\) & \(\vdots\) & \(\vdots\) & \(\vdots\) & \(\vdots\) \\
    998 & 6.03 & 5.03 & 10.75 & 6.52 & 7.28 & 10.67 & 6.7 & 10.61 & 6.67 & 6.59 & 7.69 & 2.14 \\
    999 & 6.03 & 5.03 & 10.75 & 6.52 & 7.28 & 10.66 & 6.7 & 10.62 & 6.67 & 6.59 & 7.69 & 2.14 \\
    1000 & 6.03 & 5.03 & 10.76 & 6.52 & 7.28 & 10.65 & 6.7 & 10.63 & 6.67 & 6.59 & 7.69 & 2.14 \\
    \bottomrule
  \end{tabular}
\end{table*}

From the above, we understand that MPs have trouble at changing the mean of the distribution of their weights, especially when using smaller batch sizes. This is due to the sparse gradient signals that they produce. As such, we are in a unprecedented situation. Standard linear networks struggle with the initialization of the standard deviation of their weights, with a wrong initialization causing exploding and vanishing gradients \citep{glorot2010understanding, he2015delving}. In contrast,  morphological networks struggle with the wrong initialization of the mean of their weights. 

Based on the above, we conclude that attention has to be paid when initializing morphological networks. We have to ensure that the weights of each MP of the network are initialized so that their mean will not have to change much throughout training. This remark is not limited only to initialization, but also throughout training: During training we should not make changes that require an MP to drastically change the mean of its weights. 

\paragraph{Why \(\lambda=1/2\)?}  
In the main body, we argued that choosing \(\lambda \neq 1/2\) (or allowing \(\lambda\) to be learnable) hinders trainability, leading us to propose the sum of the maximum and minimum as a stable alternative. Here, we analyze this issue in detail.  

Consider the extreme case where \(\lambda=1\), reducing the DEP operation to an unbiased max-plus MP layer. Suppose we have an MP-based morphological layer with \(n=256\) inputs, where the weights \(w_{ij}\) follow a normal distribution, either at initialization or during training. As assumed in prior works \citep{glorot2010understanding, he2015delving}, let the input \(\mathbf{x}\) also follow a normal distribution, independent of the weights. The output of the layer is then given by  \(y_i = \max_j (x_j + w_{ij})\). Since \(x_j\) and \(w_{ij}\) are independent zero-mean Gaussian variables with variance \(1\), their sum follows a Gaussian distribution with zero mean and variance \(2\). Consequently, the output distribution follows the "maximum of Gaussians" distribution. This results in a shift towards positive values (nonzero mean) and a reduction in variance for large \(n\). Specifically, a bound on the mean of the output is given by \(\sqrt{4\log(n)}\) (the proof of this fact can be found in an answer by user "Sivaraman" on Mathematics Stack Exchange.\footnote{Sivaraman, "Expectation of the Maximum of Gaussian Random Variables," Mathematics Stack Exchange, 2011. URL: \url{https://math.stackexchange.com/q/89147}}.) The variance can be estimated using the Fisher–Tippett–Gnedenko theorem \citep{fisher1928limiting}, though we omit the detailed derivation.  

For proper initialization, following the principles of Glorot and Kaiming initialization \citep{glorot2010understanding}, the output should ideally follow a zero-mean normal distribution. However, correcting the output distribution by adjusting the weights is a highly nontrivial inverse problem. At best, one can attempt to correct the mean and variance heuristically. For instance, in our case with \(n=256\), setting the weights to have a mean of \(-7.5\) and variance \(2.56\) approximates a normal distribution, as shown in Figure~\ref{fig:mp_initialization}. However, this process is cumbersome, fragile, and highly sensitive to small deviations of \(n\), making it impractical for initializing morphological networks. As a result, the max-plus MP-based layer is effectively unusable in practice, even if activated.  

\begin{figure}
    \centering
    \begin{subfigure}{0.45\textwidth}
        \centering
        \includegraphics[width=\linewidth]{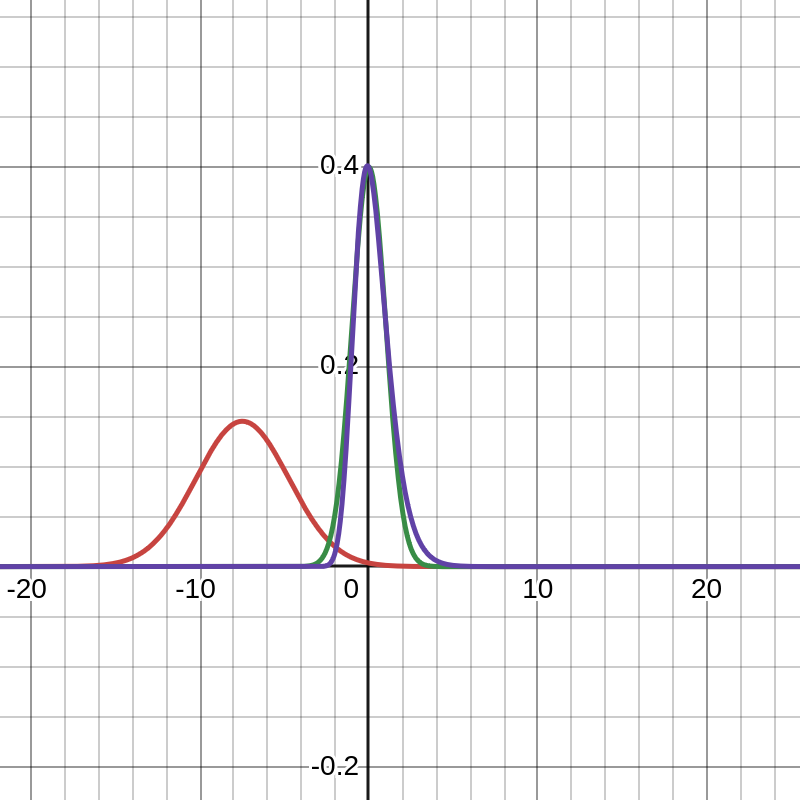}
        \caption{Distribution of max of shifted Gaussians.}
        \label{fig:mp_initialization}
    \end{subfigure}
    \hspace{10pt}
    \begin{subfigure}{0.45\textwidth}
        \centering
        \includegraphics[width=\linewidth]{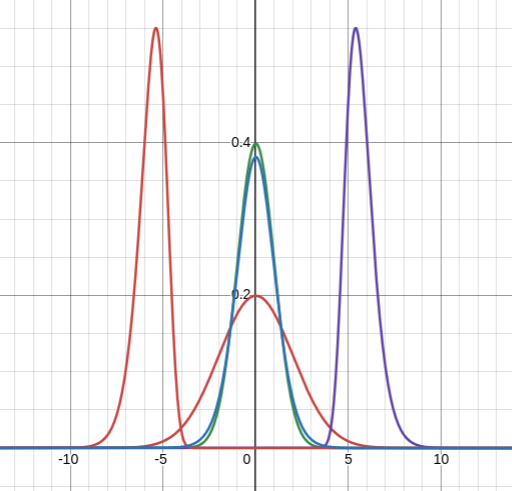}
        \caption{Distribution of sum of max and min of Gaussians.}
        \label{fig:mpm_initialization}
    \end{subfigure}
    \caption{Distributions of MP and MPM networks.}
    \label{fig:initialization}
\end{figure}

A similar issue arises with DEP-based layers when \(\lambda\) is either learnable or fixed at \(\lambda \neq 1/2\). When \(\lambda\) is fixed but not \(1/2\), the weight distributions for the maximum and minimum terms must be initialized differently to ensure a zero-mean output. When \(\lambda\) is learnable, even if initialized near \(1/2\), any deviation during training forces the layer to adjust the mean of its weights dynamically. Morphological networks, however, are not well suited for such adaptive mean corrections, further compromising their trainability.

On the other hand, setting fixed \(\lambda=1/2\) ensures that the output has zero mean and maintains a relatively stable variance, even for small variations in \(n\). In Figure~\ref{fig:mpm_initialization}, we illustrate the distribution of the sum of the maximum and minimum, assuming that for sufficiently large \(n\), the two terms are approximately uncorrelated. Additionally, the variance can be further adjusted to \(1\) by appropriately scaling the activations. This makes networks with fixed \(\lambda=1/2\) trainable. 

\paragraph{Initialization in experiments.}
In our experiments, we did not ablate the effect of initialization. Rather, we tried to initialize each network properly according to the above remarks and give it as much of an advantage as possible. Specifically:

\begin{enumerate}
    \item \textbf{Max-Plus MP-based Networks:} As discussed earlier, properly initializing this network required setting the weight variance to a value greater than \(1\) and the mean to a negative value. After extensive trial and error, we found that initializing the weights with a mean of \(-5/3\) and a standard deviation of \(3\) yielded the best results. Therefore, this initialization was used for all MP-based networks in our experiments. 
    
    \item \textbf{DEP (\(\lambda=3/4\)):} Properly initializing this network proved to be impractical. As a result, we used the same initialization as for DEP networks with \(\lambda=1/2\). 
    
    \item \textbf{DEP (learnable \(\lambda\)):} The parameter \(\lambda\) was initialized from a uniform distribution \(U([0,1])\), which has a mean of \(1/2\). Given this, the ideal weight initialization follows the same approach as DEP networks with \(\lambda=1/2\). 
    
    \item \textbf{DEP (\(\lambda=1/2\)) and MPM:} The initialization of networks with our proposed fixed \(\lambda=1/2\) and MPM networks is significantly simpler than that of MP-based networks. All morphological layers are initialized to follow a standard distribution.  
\end{enumerate}

\begin{figure}
    \centering
    \includegraphics[width=0.7\linewidth]{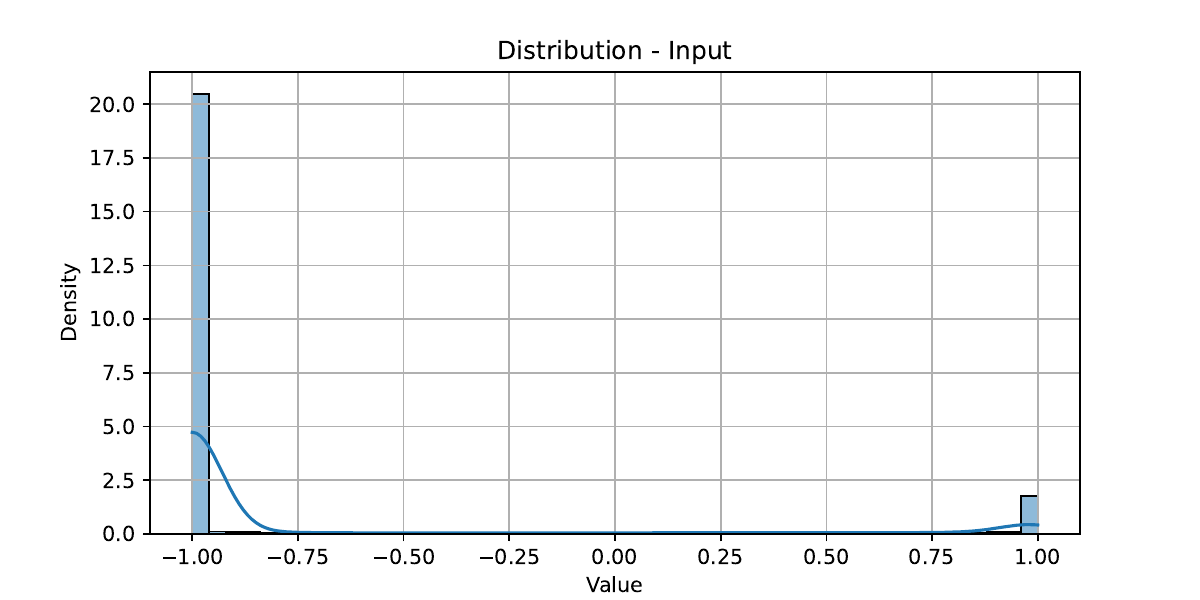}
    \caption{Distribution of MNIST values.}
    \label{fig:mnist_bernoulli}
\end{figure}

We also optimized the initialization of the linear activations (which correct the variance of the output). For fully connected networks of Setting 1, we found that initializing according to a zero-mean Gaussian distribution with standard deviation \(1/3.46\) yielded the best results (in general, initialization of activations for Setting 1 does not make significant differences. One could choose to initialize the activations to \(1/2\) and obtain good results). For convolutional networks  of Setting 1, each convolution kernel was simply initialized according to a standard distribution. For Settings 2 and 3, the linear layers were initialized according to Glorot. 

Some exceptions: i) For MNIST, the input images do not follow a zero-mean distribution (see Figure~\ref{fig:mnist_bernoulli}). For this reason, and to keep the output zero-mean, the first morphological layer of the networks was initialized to zero. ii) For the morphological ResNet-20 networks, all morphological layers were initialized following the standard distribution for all datasets, and the linear activations were initialized to give the average of the max and the min (making the initialization of the morphological ResNet-20 models simple).  

\section{Connection to the Representation Theorem}
\label{appendix:f}

In this appendix, we highlight the connection between our theoretical results and the \textit{Representation Theorem} of \citet{maragos1987morphological, maragos2013representations}. For completeness, we restate the theorem below.

\begin{theorem}[Representation Theorem]
    Let \(h(\mathbf{x}), \mathbf{x}\in \mathbb{Z}^d\) be the finite-extent impulse response of an \(m\)-dimensional linear shift-invariant filter \(\Gamma(f)=f*h\), which is defined on a class \(\mathcal{S}\) of real-valued discrete-domain signals \(f:\mathbb{Z}^d \to \mathbb{R}\) closed under translation. If \(h\) satisfies the following conditions:
    \begin{enumerate}
        \item Increasing: \(h(\mathbf{x}) \geq 0, \forall \mathbf{x}\),
        \item Translation-invariant: \(\sum_{\mathbf{x}} h(\mathbf{x})=1\), 
    \end{enumerate}
    then we can represent the linear operator as a supremum of weighted erosions:
    \[
    \Gamma(f)(\mathbf{x})=(h*f)(\mathbf{x})=\bigvee_{g\in \mathrm{Bas}(\Gamma)} \bigwedge_{\mathbf{y}\in \mathbb{Z}^d} f(\mathbf{y}) - g(\mathbf{y}-\mathbf{x}), 
    \]
    where the basis \(\mathrm{Bas}(\Gamma)\) is given by
    \[\mathrm{Bas}(\Gamma)=\{g\in \mathcal{S} : \sum_{\mathbf{y}\in spt(h)} h(\mathbf{y})g(-\mathbf{y}) = 0 \ \wedge \ g(-\mathbf{x}) = -\infty \Leftrightarrow h(\mathbf{x}) = 0\}\]
\end{theorem}

\begin{example}
    If we take \(h(n) = \alpha \delta(n) + (1-\alpha) \delta(n-1)\) with \(0 < \alpha < 1\), then a corresponding basis function \( g \in \mathrm{Bas}(\Gamma) \) takes the form:
    \[
    g(n) = 
    \begin{cases}
    r, & n = 0, \\
    -\frac{\alpha r}{1 - \alpha}, & n = -1, \\
    -\infty, & \text{otherwise},
    \end{cases}
    \quad \text{for } r \in \mathbb{R}.
    \]
    Substituting into the Representation Theorem, we obtain:
    \[
    \alpha x_n + (1 - \alpha)x_{n-1} = \sup_{r \in \mathbb{R}} \left[ \min\left\{ x_n - r,\ x_{n-1} + \frac{\alpha r}{1 - \alpha} \right\} \right].
    \]
    For example, setting \( n = 1 \) yields the following identity:
    \begin{equation}
    \label{equation:representation_theorem_example_1}
    \alpha x_1 + (1 - \alpha)x_0 = \sup_{r \in \mathbb{R}} \left[ \min\left\{ x_1 - r,\ x_0 + \frac{\alpha r}{1 - \alpha} \right\} \right].
    \end{equation}
    The LHS of this identity is a linear perceptron with weights \(\alpha_1 = \alpha, a_0 = 1-\alpha\) with \(\alpha > 0\), i.e. a linear perceptron with positive weights that sum up to 1. Notice that this condition follows from the "Increasing" and "Translation-invariant" hypotheses of the Representation Theorem.

    More generally, consider a linear perceptron with inputs \( x_0, x_1, \ldots, x_n \) and weights \( \alpha_0, \ldots, \alpha_n \), where \( \alpha_i > 0 \) for all \( i \) and \( \sum_i \alpha_i = 1 \). Then, the following identity holds:
    \begin{equation}
    \label{equation:representation_theorem_example_2}
    \sum_{i = 0}^n \alpha_i x_i = \sup_{r_0, \ldots, r_{n-1} \in \mathbb{R}} \left[ \min \left\{ x_0 - r_0, \ldots, x_{n-1} - r_{n-1},\ x_n + \frac{\sum_{i = 0}^{n-1} \alpha_i r_i}{\alpha_n} \right\} \right].
    \end{equation}
    Furthermore, if we relax the normalization constraint and only require the weights to be positive and sum to less than 1, the result still holds if we allow biased erosions. For instance, taking \( \alpha = 1/2 \) and evaluating equation \eqref{equation:representation_theorem_example_1} with \( x_0 = 0 \) and \( x_1 = x \), we find:
    \[
    \frac{x}{2} = \sup_{r \in \mathbb{R}} \left[ \min\left\{ x - r,\ r \right\} \right].
    \]
    Similarly, setting \(x_n=0\) in \eqref{equation:representation_theorem_example_2} yields the identity:
    \begin{equation}
    \label{equation:representation_theorem_example_3}
    \sum_{i = 0}^{n-1} \alpha_i x_i = \sup_{r_0, \ldots, r_{n-1} \in \mathbb{R}} \left[ \min \left\{ x_0 - r_0, \ldots, x_{n-1} - r_{n-1},\ \frac{\sum_{i = 0}^{n-1} \alpha_i r_i}{1-\sum_{i=0}^{n-1}\alpha_i} \right\} \right].
    \end{equation}
    In the above identities we can even allow the weights to be non-negative and they will still hold true: if a weight \(\alpha_i=0\), then we can always take \(r_i\to-\infty\) and the \(i\)-th term of the erosion vanishes. 
\end{example}

\textbf{Conclusion:} Based on the above example, we conclude that the Representation Theorem allows us to build any linear perceptron with weights greater than or equal to \(0\) whose sum is less than or equal to \(1\), provided we are allowed to take suprema over infinite domains and biased erosions. If the weights sum up to 1, then we use \eqref{equation:representation_theorem_example_2}; if they sum up to less than 1 then we use \eqref{equation:representation_theorem_example_3}. We can combine the two forms by writing the generalized operator \(y(\mathbf{x}) = \sup_{r\in \mathbb{R}^n}\min_{j\in [n]\cup\{0\}}(x_j+w(j, r))\), where \(w(j, r)\) is continuous in terms of \(r\), \(x_0 = 0\) is fixed, and depending on the case we can either take \(w(j, r) = +\infty\), in which case the bias vanishes, or \(w(j, r)<+\infty\), in which case we have biased erosions. 

It is trivial to see that the inverse of the above also holds true: If we can build any linear perceptron with non-negative weights that sum up to less than or equal to 1, then we can build an increasing, translation-invariant LSI filter. 

Notice that \emph{the Representation Theorem states nothing about linear perceptrons with weights that are less than \(0\) or that sum up to more than \(1\).} These cases are also discussed in \citep{maragos1987morphological}. 

This is where Theorem \ref{theorem:1} comes in. There are three main differences between Theorem \ref{theorem:1} and the Representation Thoerem:
\begin{enumerate}[left=0pt]
    \item In Theorem \ref{theorem:1} suprema and infima are only over finite domains, exactly how they are taken in neural networks. This is a condition for the proof to work.
    \item In Theorem \ref{theorem:1} we assume the use of max-plus and min-plus MPs, i.e. the weights of the supremum is not inside the previous infimum like they are in the idenities obtained from the Representation Theorem. 
    \item If we take the limit as the domain tends to be infinite, then Theorem \ref{theorem:1} proves the impossibility of representing linear perceptrons with negative weights, or weights that sum up to more that 1.
\end{enumerate}

First, we see how having suprema and infima over finite domains is implicitly used in the proof. This is done in the following two ways:
\begin{enumerate}[left=0pt]
    \item When we have a finite family of functions, with each being differentiable a.e., then they are \textbf{simultaneously} differentiable a.e. (i.e. the Lebesgue-measure of the set of points for which at least one of the functions is non-differentiable is zero).
    \item When we have a supremum or infimum over a finite domain, then it is definitely attained. 
\end{enumerate}
Let us see the above in more detail. In the proof of Theorem \ref{theorem:1}, we write \(x_i^{(n+1)}=\max_j f_{ij}^{(n+1)}\). We note that each \(f_{ij}^{(n+1)}\) is differentiable a.e.. Then, because \(j\) runs over a finite (and hence countable) set, \(f_{ij}^{(n+1)}\) for the different \(j\), and \(x_i^{(n+1)}\) are simultaneously differentiable a.e.. This allows us to apply Lemma \ref{theorem:b4} a.e., etc. 

Notice that the same argument can not be used when dealing with suprema and infima over uncountably infinite domains. For example, it is \(\frac{x}{2}=\sup_{r\in \mathbb{R}} \min (x-r, r) = \sup_{r\in \mathbb{R}} f_r(x)\). For every \(r\) the function \(f_r(x)\) is non-differentiable only for \(x = 2r\), and hence it is a.e. differentiable. However, it does not hold that \(f_r(x)\) are simultaneously differentiable a.e., because for every \(x\), there is the function \(f_{x/2}\) which is not differentiable on \(x\). In fact, for no \(x\) are the functions \(f_r(x)\) simultaneously differentiable. Hence, Lemma \ref{theorem:b4} can never be applied. 

The problem with the above is that \(\mathbb{R}\) has non-zero measure. The next logical question would be, what if we allow the suprema and infima to be taken over a infinite but countable domain. In fact, due to continuity, we can write \(\frac{x}{2} = \sup_{q\in \mathbb{Q}} \min(x - q, q) = \sup_{q\in \mathbb{Q}} f_q(x)\). Then, for every \(x \notin \mathbb{Q}\) it holds that \(f_q(x)\) are simultaneously differentiable, i.e. they are simultaneously differentiable a.e., and Lemma \ref{theorem:b4} can be applied a.e.. However, now the proof stops working for a different reason. Since we are taking a supremum over \(\mathbb{Q}\), it is not necessarily attained. In the proof, for the induction to work we implicitly make use of the fact that, because the domain of the suprema and infima is finite, they are attained, and hence the set \(J\) is non-empty. This allows the induction to work properly. 

Next, we note that despite proving the theorem only in the case of using max-plus and min-plus MPs, it can readily be generalized to also include the case of the weights of the supremum being inside the previous infimum, and vice versa. In other words, if we are given a network defined recursively as \(x^{(n+1)}_i(\mathbf{x}) = \max_{k} \min_{j} (x_j^{(n)}(\mathbf{x}) + w_{kj}), i\in [N(n+1)]\) and \(x_0^{(n)}=0\) fixed, then using the same Lemmas and an induction argument, it can be shown that Theorem \ref{theorem:1} holds also for this new type of network. 

Finally, the two theorems seem contradicting. However, they are actually (almost) complementary to each other. In order to see why, we have to go from the finite domains that Theorem \ref{theorem:1} works with, to the infinite domains that the Representation Theorem works with. To do so, we need the following lemma. 

\begin{lemma}
\label{theorem:f2}
    Let \((f_n)\) be a sequence of Lipschitz functions defined on an interval \(I\), and suppose that \(f_n \to f\) pointwise on \(I\). Assume that there exist constants \(m, M \in \mathbb{R} \cup \{-\infty, +\infty\}\) such that for all \(n \in \mathbb{N}\), it holds a.e. on \(I\) that
    \[
    m \leq f_n'(x) \leq M.
    \]
    Then, wherever the derivative \(f'(x)\) exists on \(I\), it satisfies
    \[
    m \leq f'(x) \leq M.
    \]
\end{lemma}

\begin{proof}
    Since each \(f_n\) is Lipschitz on \(I\), it is absolutely continuous on any compact subinterval \([a, b] \subset I\). Furthermore, since \(f_n'(x)\) exists almost everywhere and is bounded as \(m \leq f_n'(x) \leq M\), we have for any \(a < b\) in \(I\):
    \[
    m(b-a) \leq \int_a^b f_n'(x) \, dx = f_n(b) - f_n(a) \leq M(b-a).
    \]
    Taking the limit as \(n \to \infty\) and using the pointwise convergence \(f_n \to f\), we obtain:
    \[
    m(b-a) \leq f(b) - f(a) \leq M(b-a),
    \]
    which implies:
    \[
    m \leq \frac{f(b) - f(a)}{b - a} \leq M \quad \text{for all } a < b \text{ in } I.
    \]
    Thus, whenever \(f'(x)\) exists, it must lie in the interval \([m, M]\).
\end{proof}

\begin{example}
    Consider the sequence \(f_n(x) = \frac{1}{n} \sin(nx)\), defined on an interval \(I\) containing \(0\). Each \(f_n\) is Lipschitz on \(I\) and converges pointwise to \(f(x) = 0\) for all \(x \in \mathbb{R}\). The derivatives are given by \(f_n'(x) = \cos(nx)\), which do not converge uniformly.
    
    In particular, \(f_n'(0) = 1\) for all \(n\), while \(f'(0) = 0\). Thus, the derivative and limit cannot be interchanged. However, since \(-1 \leq f_n'(x) \leq 1\) for all \(x \in I\), Lemma \ref{theorem:f2} applies, and we conclude that:
    \[
    -1 \leq f'(x) \leq 1 \quad \text{wherever } f'(x) \text{ exists}.
    \]
    Indeed, since \(f(x) \equiv 0\), we have \(f'(x) = 0\) everywhere it exists, and the bound is satisfied. Moreover, the bound \([-1, 1]\) is tight no matter how small an interval \(I\) we choose around 0. Hence, Lemma \ref{theorem:f2} cannot give us a better estimate on \(f'(0)\). 
\end{example}

Take a function of the form \(y(\mathbf{x}) = \sup_{r\in \mathbb{R}^n} \min_{j\in [n] \cup \{0\}} (x_j + w(j, r))\), like in the identities \eqref{equation:representation_theorem_example_2}, \eqref{equation:representation_theorem_example_3} obtained from the Representation theorem, where \(w(j, r)\) is continuous in terms of \(r\). Due to continuity, we have that \(y(\mathbf{x}) = \sup_{q\in \mathbb{Q}^n} \min_{j} (x_j + w(j, q))\), and it also holds that \(y(\mathbf{x}) = \sup_{i\in \mathbb{N}} \min_{j} (x_j + w(j, q_i)) = \lim_{n\to \infty} \max_{i\in [n]} \min_{j} (x_j + w(j, q_i)) = \lim_{n\to \infty} y_n(\mathbf{x})\), where \(q_i\) is a counting of the countable \(\mathbb{Q}^n\). Indeed, the above holds because \(y_n\) is an increasing sequence: 
\[
y_n = \max_{i\in [n]}\min_j(x_j + w(j, q_i)) \leq \max_{i\in [n+1]}\min_j(x_j + w(j, q_i)) = y_{n+1} \Rightarrow
\]
\[
\lim_{n\to \infty} y_n(\mathbf{x}) = \sup_{n\in \mathbb{N}} y_n(\mathbf{x}) = \sup_{n\in \mathbb{N}}\max_{i\in [n]}\min_j(x_j+w(j, q_i)) = \sup_{i\in \mathbb{N}}\min_j(x_j + w(j, q_i)) = y(\mathbf{x}).
\]
For every \(y_n\) Theorem \ref{theorem:1} holds. Hence for every \(n,\ y_n\) is Lipschitz, and a.e. \(\nabla y_n(\mathbf{x}) = 0\) or \(\nabla y_n(\mathbf{x}) = e_{i_n}\) for some \(i_n=i_n(\mathbf{x})\). Hence, a.e. it holds that \(\partial_{x_j} y_n(\mathbf{x}) \geq 0\) and \(\sum_{j} (\nabla y_n(\mathbf{x}))_j\) equals either \(0\) or \(1\). Now we have the following arguments:
\begin{enumerate}
    \item Define \( g_n(t) = y_n(x_j = t, \mathbf{x}_{-j}) \). Since \( y_n \) is Lipschitz, so is \( g_n \). Moreover, we are given that \( \partial_{x_j} y_n(\mathbf{x}) \geq 0 \) a.e., which implies that \( \frac{dg_n}{dt}(t) \geq 0 \) a.e.. Now define \( g(t) = y(x_j = t, \mathbf{x}_{-j}) \). We have that \( g_n \to g \) pointwise, since \(y_n \to y\) pointwise. Therefore, we can apply Lemma \ref{theorem:f2} to obtain \( \frac{dg}{dt}(t) \geq 0 \) wherever the derivative exists. Therefore,
    \[
        \partial_{x_j} y(\mathbf{x}) \geq 0,
    \]
    if the derivative exists. 
    \item We are given that
    \[
        \sum_j (\nabla y_n(\mathbf{x}))_j = \left\langle \mathbf{1}, \nabla y_n(\mathbf{x}) \right\rangle \in \{0, 1\} \quad \Rightarrow \quad \left\langle \mathbf{1}, \nabla y_n(\mathbf{x}) \right\rangle \leq 1, \quad \text{a.e.}
    \]
    Define \( g_n(t) = y_n(\mathbf{x} + t\mathbf{1}) \) and \( g(t) = y(\mathbf{x} + t\mathbf{1}) \). Then
    \[
        \frac{dg_n}{dt}(t) = \left\langle \mathbf{1}, \nabla y_n(\mathbf{x} + t\mathbf{1}) \right\rangle \leq 1, \quad \text{a.e.}
    \]
    Since \( y_n \) is Lipschitz, so is \( g_n \). Since \(y_n\to y\) pointwise, it also holds that \(g_n\to g\) pointwise. Hence, Lemma \ref{theorem:f2} applies, and we have that 
    \[
        \frac{dg}{dt}(0) = \left\langle \mathbf{1}, \nabla y(\mathbf{x}) \right\rangle \leq 1,
    \]
    if the derivative exists. 
\end{enumerate}
Based on the above arguments, we see that \(\nabla y(\mathbf{x}) \succeq 0\) and \(\left<\mathbf{1}, \nabla y(\mathbf{x})\right> \leq 1\), wherever the derivative exists. If we suppose that \(y\) can represent a linear perceptron with weights \(\alpha_1, \ldots, \alpha_n\), then it must hold that \(\alpha_j = \partial_{x_j}y(\mathbf{x}) \geq 0\) and \(\alpha_1 + \ldots + \alpha_n = \left<\mathbf{1}, \nabla y(\mathbf{x})\right> \leq 1\). In other words, \emph{a byproduct of (the slight variation of) Theorem \ref{theorem:1} is that functions of the form of the Representation Theorem can represent linear perceptrons only if the weights of the perceptrons are non-negative, and their sum is less than or equal to 1.}

\section{Additional experiments}
\label{appendix:e}

Before we continue with the additional experiments, we provide the details of our experimental setup and compute resources required: All experiments can be reproduced from the provided python notebooks by sequentially running all of their cells. No seeding was performed, and variation in results was controlled by repeated experiments. Notebooks "\texttt{mainMNNs*.ipynb}" include the majority of experiments. Notebook "\texttt{regressionMNNs.ipynb}" includes additional experiments (and experiments on regression tasks presented in Appendix \ref{appendix:d} and this Appendix). Notebooks "\texttt{snip\_experiments\_corrected*.ipynb}" and "\texttt{snip\_resnet\_experiments\_corrected*.ipynb}" include the SNIP pruning experiments. Notebooks "\texttt{resnet\_20\_experiments*.ipynb}" include the experiments on ResNet-20 models. The GPUs we used were the NVIDIA GeForce RTX 2080 Ti and the RTX 3060, both of which have 12GB of memory. One GPU is sufficient to run each notebook (i.e. 12GB of GPU memory suffice to run each notebook). For notebooks "\texttt{mainMNNs*.ipynb}" and "\texttt{snip\_experiments*.ipynb}", each epoch of training took about 30 seconds, and a whole cell of training took about 30 minutes (with the exception of RMPM-Drop, which was trained for 200 epochs). The other cells took negligible time to complete. Each "\texttt{mainMNNs*.ipynb}" notebook required about a day to run from start to finish. Each "\texttt{snip\_experiments\_corrected*.ipynb}" notebook took about 6 hours to run. For notebooks "\texttt{resnet\_20\_experiments*.ipynb}" and "\texttt{snip\_resnet\_experiments\_corrected*.ipynb}", each epoch of MPM-ResNet-20 training took about 1.5 minutes, longer than the 30 seconds required by the linear ResNet-20. We believe this is due to the inefficiency of our CUDA implementation of max-plus convolution. Each "\texttt{resnet\_20\_experiments*.ipynb}" took about 6 hours to run and each \\ "\texttt{snip\_resnet\_experiments\_corrected*.ipynb}" about a day. 

\textbf{Morphological convolution CUDA module:} Because implementing convolution via unfolding would be prohibitively resource demanding, we developed a CUDA module for PyTorch, implementing max-plus convolution under its strict definition. Unfortunately, we were not able to make it state-of-the-art efficient. Still, it was good enough for the purposes of our experiments. Since the authors are not as well-versed in development of CUDA kernels, and as with all low-level code, we advice a thorough reading of the code before ane use for purposes outside those of the included experiments. 

\textbf{Declaration of LLM Usage:} We made use of LLMs, and specifically ChatGPT, as a programming assistant for tasks like writing boilerplate code, code auto-completion, fixing errors and debugging. Any code generated by ChatGPT was rigorously checked for correctness. In particular, ChatGPT was extensively used to aid with the development of our CUDA module, with the core logic of the module written by the authors and all other aspects checked for correctness. 

\textbf{Parameter count of models:} Table \ref{table:parameter_count} presents the parameter counts of all models used in our experiments. For the ResNet-20 models, two numbers appear: one for gray-scale datasets such as F-MNIST, and one for RGB datasets such as CIFAR-10. 

{
\setlength{\tabcolsep}{5pt}
\begin{table}
  \caption{Parameter count of all models.}
  \label{table:parameter_count}
  \centering
  \begin{tabular}{lcccccc}
    \toprule
    \textbf{Network} & MLP & MP & DEP & DEP (\(1/2\)) & Act-MP & Act-DEP \\
    \textbf{Params.} & 466698 & 466698 & 932106 & 930816 & 467978 & 933386 \\
    \midrule
    \textbf{Network} & Act-DEP (\(3/4\)) & Act-DEP (\(1/2\)) & MinMaxPlus & MPM & RMPM & RMPM-Drop \\
    \textbf{Params.} & 932096 & 932096 & 859408 & 469268 & 469268 & 469268 \\
    \midrule
    \textbf{Network} & MPM-SVD & LeNet-5 & MPM-LeNet-5 & MPM-SVD-LeNet-5 & ResNet-20 & MPM-ResNet-20 \\ 
    \textbf{Params.} & 469268 & 61696 & 63304 & 63304 & 271994/272282 & 278378/278666 \\
    \bottomrule
  \end{tabular}
\end{table}
}

Next, we present our additional experimental results. Specifically, we first provide the results on \(\ell_1\)-based masking that were explained in the main text, in Table \ref{table:pruning}. Then, we provide experiments on simple regression tasks, and the full training history of networks that we claimed were trainable but for which only peak training accuracy was reported in the main body.

\begin{table*}
  \caption{Performance of \(\ell_1\)-based weight masked MLP and MPM for various masking ratios on MNIST and Fashion-MNIST.}
  \label{table:pruning}
  \centering
  \begin{tabular}{lcccc}
    \toprule
    \multirow{2}{*}{\shortstack[l]{\textbf{Masking}\\ \textbf{ratio}}} & \multicolumn{2}{c}{\textbf{MNIST}} & \multicolumn{2}{c}{\textbf{Fashion-MNIST}} \\
    \cmidrule(lr){2-3} \cmidrule(lr){4-5} 
    & MLP & MPM & MLP & MPM \\
    \midrule
    0.85 & 71.93 $\pm$ 5.48 & 93.59 $\pm$ 0.65 & 54.12 $\pm$ 4.38 & 80.83 $\pm$ 1.09 \\
    0.875 & 58.52 $\pm$ 4.63 & 93.02 $\pm$ 0.73 & 34.46 $\pm$ 5.50 & 79.92 $\pm$ 1.17 \\
    0.9 & 38.20 $\pm$ 8.89 & 92.21 $\pm$ 1.01 & 17.17 $\pm$ 2.23 & 79.19 $\pm$ 0.99 \\
    0.925 & 17.19 $\pm$ 3.53 & 90.88 $\pm$ 1.01 & 13.29 $\pm$ 2.02 & 77.53 $\pm$ 1.74 \\
    0.95 & 11.51 $\pm$ 1.80 & 78.65 $\pm$ 13.99 & 10.70 $\pm$ 1.19 & 74.94 $\pm$ 0.16 \\
    0.975 & 9.88 $\pm$ 0.25 & 64.37 $\pm$ 10.75 & 10.27 $\pm$ 0.19 & 57.11 $\pm$ 4.62 \\
    \bottomrule
  \end{tabular}
\end{table*}

\paragraph{Regression experiments.} 
Continuing, we present experiments on simple regression tasks. Notice that on classification tasks (like the ones presented in the main body), it is not necessary for the network to be a universal approximator, since we work with logits. This makes regression a more ideal task for showcasing the need for our "linear" activations. We perform regression of simple single-variate single-output functions sampled with zero-mean i.i.d. gaussian noise. To ablate the effect of our "linear" activations, we use 3 fully connected networks, each with two hidden layers of size \(100\): 1) a simple ReLU-activated MLP, 2) the MPM network, and 3) a non-activated MPM network. The functions which we sample are the following: 1) \(6\sin(x)\), 2) \(x^2\), and 3) \(20x\). The results are presented in Figures \ref{fig:regression1}, \ref{fig:regression2}, \ref{fig:regression3}. As expected, the non-activated MPM fails to approximate the underlying functions. The MLP and MPM however, achieve a successful regression of the samples. 

\begin{figure*}
    \centering
    \begin{subfigure}{0.3\textwidth}
        \centering
        \includegraphics[width=\linewidth]{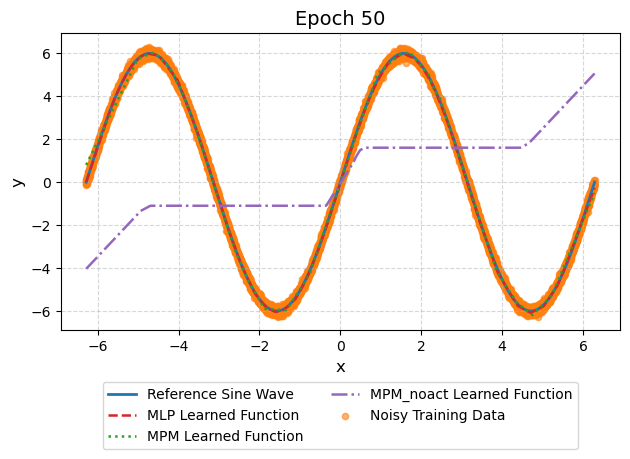}
        \caption{Regression of \(6\sin(x)\).}
        \label{fig:regression1}
    \end{subfigure}
    \begin{subfigure}{0.3\textwidth}
        \centering
        \includegraphics[width=\linewidth]{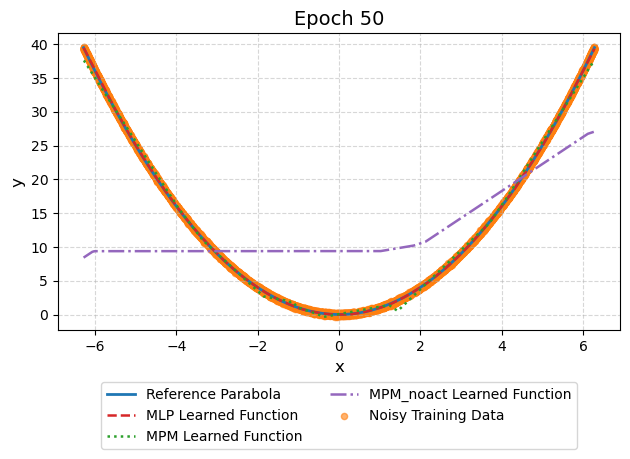}
        \caption{Regression of \(x^2\).}
        \label{fig:regression2}
    \end{subfigure}
    \begin{subfigure}{0.3\textwidth}
        \centering
        \includegraphics[width=\linewidth]{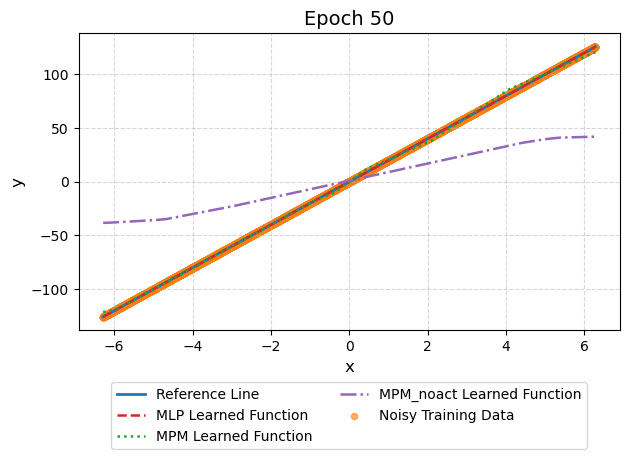}
        \caption{Regression of \(20x\).}
        \label{fig:regression3}
    \end{subfigure}
    \caption{Regression using MLP, MPM, and non-activated MPM on noisy samples for different underlying functions. }
    \label{fig:regression}
\end{figure*}

\paragraph{Full training history.}

Finally, we present the full training history of our trainable networks. These results serve as proof that the networks not only achieve satisfactory peak training accuracy, but also manage to converge. In Table \ref{table:fully-connected-full-mnist} we provide the full training history of the networks Act-DEP(\(\lambda = 1/2\)), MPM, RMPM, and MPM-SVD on MNIST, including training accuracy and validation accuracy. In Table \ref{table:fully-connected-full-fmnist} we provide the full training history of the networks Act-DEP(\(\lambda = 1/2\)), MPM, RMPM, and MPM-SVD on Fashion-MNIST, including training accuracy and validation accuracy.

\begin{table*}
  \caption{Full training history, MNIST.}
  \label{table:fully-connected-full-mnist}
  \centering
  \begin{tabular}{lcccccccc}
    \toprule
    \multicolumn{1}{c}{\textbf{Epoch}} & \multicolumn{4}{c}{\textbf{Train}} & \multicolumn{4}{c}{\textbf{Validation}} \\
    \cmidrule(lr){2-5} \cmidrule(lr){6-9}
    & ActDEP(\(\lambda=1/2\)) & MPM & RMPM & MPM-SVD & ActDEP(\(\lambda=1/2\)) & MPM & RMPM & MPM-SVD \\
    \midrule
    1 & 56.03 & 85.83 & 88.57 & 90.04 & 56.33 & 85.85 & 87.81 & 89.68 \\ 
    2 & 84.13 & 91.19 & 92.71 & 94.06 & 84.05 & 90.55 & 91.75 & 92.99 \\ 
    3 & 88.25 & 93.21 & 94.47 & 95.95 & 87.88 & 92.35 & 93.27 & 94.38 \\ 
    4 & 90.18 & 94.34 & 95.60 & 97.03 & 89.85 & 93.22 & 94.08 & 94.90 \\ 
    5 & 91.72 & 95.01 & 96.40 & 97.73 & 91.18 & 93.62 & 94.62 & 95.16 \\ 
    6 & 92.69 & 95.65 & 96.93 & 98.25 & 91.71 & 94.01 & 94.87 & 95.38 \\ 
    7 & 93.53 & 96.11 & 97.33 & 98.61 & 92.30 & 94.12 & 94.89 & 95.43 \\ 
    8 & 94.18 & 96.58 & 97.83 & 98.98 & 92.73 & 94.36 & 95.15 & 95.71 \\ 
    9 & 94.58 & 96.88 & 98.16 & 98.87 & 93.01 & 94.61 & 95.32 & 95.43 \\ 
    10 & 95.01 & 97.16 & 98.42 & 99.42 & 93.38 & 94.78 & 95.40 & 95.92 \\ 
    11 & 95.38 & 97.44 & 98.74 & 99.47 & 93.50 & 94.84 & 95.41 & 95.78 \\ 
    12 & 95.72 & 97.76 & 98.82 & 99.55 & 93.67 & 94.80 & 95.33 & 95.83 \\ 
    13 & 96.00 & 97.81 & 99.15 & 99.67 & 93.72 & 94.93 & 95.42 & 95.80 \\ 
    14 & 96.26 & 98.10 & 99.25 & 99.67 & 93.97 & 95.09 & 95.48 & 95.77 \\ 
    15 & 96.45 & 98.21 & 99.41 & 99.66 & 94.04 & 94.81 & 95.64 & 95.83 \\ 
    16 & 96.72 & 98.36 & 99.50 & 99.71 & 94.19 & 94.72 & 95.47 & 95.74 \\ 
    17 & 96.78 & 98.44 & 99.55 & 99.84 & 94.27 & 94.99 & 95.55 & 95.83 \\ 
    18 & 97.02 & 98.65 & 99.62 & 99.79 & 94.42 & 94.89 & 95.42 & 95.72 \\ 
    19 & 97.16 & 98.67 & 99.67 & 99.79 & 94.47 & 94.97 & 95.33 & 95.72 \\ 
    20 & 97.33 & 98.83 & 99.73 & 99.84 & 94.52 & 95.03 & 95.38 & 95.78 \\ 
    21 & 97.46 & 98.90 & 99.77 & 99.85 & 94.48 & 95.00 & 95.30 & 95.72 \\ 
    22 & 97.61 & 99.01 & 99.82 & 99.94 & 94.59 & 95.04 & 95.42 & 95.99 \\ 
    23 & 97.73 & 99.02 & 99.85 & 99.74 & 94.56 & 95.09 & 95.37 & 95.92 \\ 
    24 & 97.88 & 99.13 & 99.84 & 99.65 & 94.59 & 95.08 & 95.45 & 95.52 \\ 
    25 & 98.00 & 99.18 & 99.83 & 99.90 & 94.69 & 95.10 & 95.18 & 95.87 \\ 
    26 & 98.08 & 99.14 & 99.90 & 99.86 & 94.67 & 94.93 & 95.40 & 95.74 \\ 
    27 & 98.21 & 99.32 & 99.92 & 99.94 & 94.58 & 95.04 & 95.49 & 96.00 \\ 
    28 & 98.28 & 99.39 & 99.92 & 99.95 & 94.74 & 94.92 & 95.33 & 96.15 \\ 
    29 & 98.34 & 99.36 & 99.92 & 99.74 & 94.72 & 94.94 & 95.38 & 95.86 \\ 
    30 & 98.45 & 99.45 & 99.89 & 99.56 & 94.71 & 94.84 & 95.33 & 95.62 \\ 
    31 & 98.52 & 99.54 & 99.93 & 99.72 & 94.82 & 94.91 & 95.38 & 95.77 \\ 
    32 & 98.64 & 99.50 & 99.96 & 99.96 & 94.78 & 94.69 & 95.59 & 96.04 \\ 
    33 & 98.72 & 99.54 & 99.93 & 99.82 & 94.76 & 94.89 & 95.34 & 95.81 \\ 
    34 & 98.73 & 99.62 & 99.95 & 99.95 & 94.67 & 94.88 & 95.34 & 96.08 \\ 
    35 & 98.86 & 99.60 & 99.91 & 99.94 & 94.75 & 94.69 & 95.41 & 96.12 \\ 
    36 & 98.90 & 99.67 & 99.94 & 99.99 & 94.78 & 94.72 & 95.44 & 96.06 \\ 
    37 & 98.99 & 99.66 & 99.96 & 99.99 & 94.73 & 94.80 & 95.40 & 96.20 \\ 
    38 & 99.04 & 99.73 & 99.98 & 99.67 & 94.70 & 94.72 & 95.52 & 95.64 \\ 
    39 & 99.06 & 99.74 & 99.98 & 99.97 & 94.73 & 94.73 & 95.51 & 96.05 \\ 
    40 & 99.06 & 99.74 & 99.97 & 99.99 & 94.68 & 94.59 & 95.39 & 96.15 \\ 
    41 & 99.13 & 99.78 & 99.85 & 99.99 & 94.67 & 94.61 & 95.35 & 96.17 \\ 
    42 & 99.19 & 99.73 & 99.92 & 99.99 & 94.77 & 94.58 & 95.35 & 96.22 \\ 
    43 & 99.25 & 99.79 & 99.97 & 99.99 & 94.75 & 94.44 & 95.53 & 96.16 \\ 
    44 & 99.26 & 99.76 & 99.98 & 99.12 & 94.81 & 94.81 & 95.76 & 95.34 \\ 
    45 & 99.30 & 99.78 & 99.98 & 99.83 & 94.72 & 94.62 & 95.64 & 96.05 \\ 
    46 & 99.35 & 99.77 & 99.98 & 99.88 & 94.83 & 94.63 & 95.52 & 96.01 \\ 
    47 & 99.33 & 99.81 & 99.98 & 99.99 & 94.68 & 94.58 & 95.52 & 96.33 \\ 
    48 & 99.35 & 99.80 & 99.75 & 99.99 & 94.72 & 94.77 & 95.24 & 96.41 \\ 
    49 & 99.47 & 99.86 & 99.96 & 99.99 & 94.75 & 94.67 & 95.38 & 96.37 \\ 
    50 & 99.47 & 99.81 & 99.99 & 99.99 & 94.68 & 94.75 & 95.51 & 96.38 \\ 

    \bottomrule
  \end{tabular}
\end{table*}

\begin{table*}
  \caption{Full training history, Fashion-MNIST.}
  \label{table:fully-connected-full-fmnist}
  \centering
  \begin{tabular}{lcccccccc}
    \toprule
    \multicolumn{1}{c}{\textbf{Epoch}} & \multicolumn{4}{c}{\textbf{Train}} & \multicolumn{4}{c}{\textbf{Validation}} \\
    \cmidrule(lr){2-5} \cmidrule(lr){6-9}
    & ActDEP(\(\lambda=1/2\)) & MPM & RMPM & MPM-SVD & ActDEP(\(\lambda=1/2\)) & MPM & RMPM & MPM-SVD \\
    \midrule
    1 & 27.90 & 75.07 & 80.46 & 83.01 & 27.97 & 74.61 & 79.55 & 81.33 \\ 
    2 & 71.51 & 82.53 & 85.12 & 87.56 & 70.70 & 80.81 & 82.62 & 84.49 \\ 
    3 & 78.22 & 84.83 & 87.20 & 89.27 & 77.37 & 82.25 & 83.78 & 84.58 \\ 
    4 & 81.36 & 86.30 & 89.29 & 91.91 & 79.85 & 82.89 & 84.55 & 85.55 \\ 
    5 & 82.71 & 87.35 & 90.57 & 93.06 & 81.01 & 83.12 & 84.53 & 85.76 \\ 
    6 & 83.59 & 88.46 & 91.88 & 93.62 & 81.38 & 83.43 & 84.97 & 85.18 \\ 
    7 & 84.54 & 89.22 & 92.74 & 95.35 & 82.20 & 83.11 & 84.72 & 85.66 \\ 
    8 & 85.31 & 90.11 & 93.60 & 95.86 & 82.45 & 83.53 & 84.56 & 85.82 \\ 
    9 & 85.99 & 90.73 & 94.42 & 96.66 & 82.72 & 83.58 & 84.72 & 85.73 \\ 
    10 & 86.43 & 91.53 & 94.92 & 96.89 & 82.98 & 83.50 & 84.61 & 85.16 \\ 
    11 & 86.89 & 91.93 & 95.59 & 97.56 & 82.97 & 83.65 & 84.71 & 85.51 \\ 
    12 & 87.49 & 92.54 & 95.72 & 97.80 & 83.17 & 83.39 & 84.28 & 85.17 \\ 
    13 & 87.88 & 92.63 & 96.39 & 97.66 & 83.32 & 83.27 & 84.24 & 85.32 \\ 
    14 & 88.37 & 93.21 & 96.79 & 97.19 & 83.51 & 83.40 & 84.27 & 84.70 \\ 
    15 & 88.62 & 93.79 & 97.06 & 98.55 & 83.61 & 83.40 & 84.32 & 85.28 \\ 
    16 & 89.09 & 93.75 & 97.34 & 98.24 & 83.42 & 83.41 & 84.29 & 84.88 \\ 
    17 & 89.40 & 94.40 & 97.41 & 98.99 & 83.46 & 83.43 & 83.91 & 85.22 \\ 
    18 & 89.64 & 94.64 & 97.68 & 98.60 & 83.51 & 83.19 & 83.97 & 85.30 \\ 
    19 & 90.02 & 94.88 & 98.09 & 98.77 & 83.74 & 83.20 & 83.95 & 84.52 \\ 
    20 & 90.12 & 94.98 & 98.03 & 98.79 & 83.52 & 83.58 & 83.92 & 85.06 \\ 
    21 & 90.50 & 95.31 & 98.43 & 99.16 & 83.66 & 83.33 & 83.86 & 84.70 \\ 
    22 & 90.75 & 95.71 & 98.45 & 98.66 & 83.74 & 83.42 & 83.51 & 84.69 \\ 
    23 & 90.95 & 95.77 & 98.68 & 99.13 & 83.53 & 83.28 & 83.69 & 84.78 \\ 
    24 & 91.31 & 95.90 & 98.80 & 99.31 & 83.42 & 83.28 & 83.47 & 84.81 \\ 
    25 & 91.49 & 96.13 & 98.67 & 99.40 & 83.57 & 83.11 & 83.49 & 85.22 \\ 
    26 & 91.76 & 96.21 & 98.60 & 98.97 & 83.33 & 82.83 & 83.06 & 84.95 \\ 
    27 & 92.01 & 96.62 & 98.94 & 98.75 & 83.38 & 83.12 & 83.38 & 84.36 \\ 
    28 & 92.26 & 96.75 & 98.69 & 98.89 & 83.40 & 82.92 & 83.07 & 84.53 \\ 
    29 & 92.41 & 96.68 & 98.93 & 99.42 & 83.44 & 82.66 & 83.38 & 84.89 \\ 
    30 & 92.61 & 96.69 & 99.12 & 98.96 & 83.06 & 82.53 & 83.08 & 84.41 \\ 
    31 & 92.81 & 97.06 & 99.15 & 99.43 & 83.08 & 82.49 & 83.03 & 84.97 \\ 
    32 & 93.06 & 97.07 & 99.25 & 99.18 & 83.40 & 82.42 & 82.99 & 84.27 \\ 
    33 & 92.97 & 97.29 & 99.24 & 99.45 & 83.12 & 82.64 & 83.14 & 85.03 \\ 
    34 & 93.25 & 97.43 & 99.36 & 99.59 & 83.25 & 82.59 & 83.21 & 84.64 \\ 
    35 & 93.43 & 97.25 & 99.20 & 99.16 & 83.02 & 82.34 & 82.80 & 84.66 \\ 
    36 & 93.47 & 97.44 & 99.38 & 99.00 & 83.03 & 82.11 & 82.63 & 84.30 \\ 
    37 & 93.65 & 97.65 & 99.29 & 99.64 & 83.22 & 82.35 & 82.99 & 84.74 \\ 
    38 & 93.97 & 97.80 & 99.51 & 99.56 & 83.21 & 82.12 & 82.64 & 84.86 \\ 
    39 & 93.78 & 97.67 & 99.41 & 98.62 & 83.15 & 82.42 & 83.03 & 84.89 \\ 
    40 & 94.23 & 97.78 & 99.24 & 98.30 & 83.22 & 82.21 & 82.59 & 83.72 \\ 
    41 & 94.25 & 97.89 & 99.26 & 99.40 & 83.12 & 82.08 & 82.55 & 84.83 \\ 
    42 & 94.39 & 97.98 & 99.50 & 99.03 & 83.10 & 82.29 & 82.85 & 85.25 \\ 
    43 & 94.55 & 98.08 & 99.56 & 99.42 & 82.93 & 82.50 & 82.63 & 84.66 \\ 
    44 & 94.65 & 98.13 & 99.56 & 99.59 & 83.20 & 82.17 & 82.50 & 85.00 \\ 
    45 & 94.77 & 97.97 & 99.49 & 99.62 & 83.12 & 81.87 & 82.90 & 84.89 \\ 
    46 & 94.86 & 98.29 & 99.51 & 99.51 & 83.11 & 82.17 & 82.88 & 84.63 \\ 
    47 & 94.88 & 98.31 & 99.46 & 99.45 & 82.96 & 81.89 & 82.47 & 84.60 \\ 
    48 & 94.98 & 98.36 & 99.51 & 99.47 & 82.86 & 82.20 & 82.70 & 84.95 \\ 
    49 & 95.08 & 98.41 & 99.53 & 99.65 & 82.92 & 81.99 & 82.85 & 85.00 \\ 
    50 & 95.06 & 98.30 & 99.62 & 99.37 & 82.74 & 81.92 & 82.83 & 84.86 \\ 

    \bottomrule
  \end{tabular}
\end{table*}

\end{document}